\documentclass[twoside,11pt]{article}

\usepackage{url}

\usepackage[mathscr]{eucal}
\usepackage[cmex10]{amsmath}
\usepackage{epsfig,epsf,psfrag}
\usepackage{amssymb,amsmath,amsthm,amsfonts,latexsym}
\usepackage{amsmath,bm,xcolor,url,overpic}
\usepackage{subfigure}
\usepackage{fixltx2e}%ordering of single and double column floats
\usepackage{array}%array and tabular environments
\usepackage{verbatim}
\usepackage{bm}
\usepackage{algorithmic}
\usepackage{algorithm}
\usepackage{verbatim}
\usepackage{textcomp}
\usepackage{mathrsfs}
\usepackage{epstopdf}
\usepackage{booktabs} % for professional tables

\newcommand{\openone}{\leavevmode\hbox{\small1\normalsize\kern-.33em1}}

\catcode`~=11 \def\UrlSpecials{\do\~{\kern -.15em\lower .7ex\hbox{~}\kern .04em}} \catcode`~=13 

\allowdisplaybreaks[1]

\newcommand{\nn}{\nonumber}

\newcommand{\calB}{\mathcal{B}}

\newcommand{\calG}{\mathcal{G}}

\newcommand{\calI}{\mathcal{I}}

\newcommand{\calN}{\mathcal{N}}

\newcommand{\calP}{\mathcal{P}}

\newcommand{\calU}{\mathcal{U}}

\newcommand{\calX}{\mathcal{X}}
\newcommand{\calY}{\mathcal{Y}}
\newcommand{\calZ}{\mathcal{Z}}

\newcommand{\ba}{\mathbf{a}}

\newcommand{\bb}{\mathbf{b}}

\newcommand{\bg}{\mathbf{g}}

\newcommand{\bI}{\mathbf{I}}

\newcommand{\bu}{\mathbf{u}}

\newcommand{\bv}{\mathbf{v}}
\newcommand{\bV}{\mathbf{V}}

\newcommand{\bx}{\mathbf{x}}
\newcommand{\bX}{\mathbf{X}}
\newcommand{\by}{\mathbf{y}}

\newcommand{\bz}{\mathbf{z}}
\newcommand{\bZ}{\mathbf{Z}}

\newcommand{\rmc}{\mathrm{c}}

\newcommand{\rmd}{\mathrm{d}}

\newcommand{\rml}{\mathrm{l}}
\newcommand{\rmL}{\mathrm{L}}

\newcommand{\rmQ}{\mathrm{Q}}

\newcommand{\rmu}{\mathrm{u}}

\newcommand{\bbE}{\mathbb{E}}

\newcommand{\bbN}{\mathbb{N}}

\newcommand{\bbP}{\mathbb{P}}

\newcommand{\bbR}{\mathbb{R}}

\DeclareMathAlphabet{\mathbsf}{OT1}{cmss}{bx}{n}
\DeclareMathAlphabet{\mathssf}{OT1}{cmss}{m}{sl}% slanted sans serif

\DeclareSymbolFont{bsfletters}{OT1}{cmss}{bx}{n}  
\DeclareSymbolFont{ssfletters}{OT1}{cmss}{m}{n}
\DeclareMathSymbol{\bsfGamma}{0}{bsfletters}{'000}
\DeclareMathSymbol{\ssfGamma}{0}{ssfletters}{'000}
\DeclareMathSymbol{\bsfDelta}{0}{bsfletters}{'001}
\DeclareMathSymbol{\ssfDelta}{0}{ssfletters}{'001}
\DeclareMathSymbol{\bsfTheta}{0}{bsfletters}{'002}
\DeclareMathSymbol{\ssfTheta}{0}{ssfletters}{'002}
\DeclareMathSymbol{\bsfLambda}{0}{bsfletters}{'003}
\DeclareMathSymbol{\ssfLambda}{0}{ssfletters}{'003}
\DeclareMathSymbol{\bsfXi}{0}{bsfletters}{'004}
\DeclareMathSymbol{\ssfXi}{0}{ssfletters}{'004}
\DeclareMathSymbol{\bsfPi}{0}{bsfletters}{'005}
\DeclareMathSymbol{\ssfPi}{0}{ssfletters}{'005}
\DeclareMathSymbol{\bsfSigma}{0}{bsfletters}{'006}
\DeclareMathSymbol{\ssfSigma}{0}{ssfletters}{'006}
\DeclareMathSymbol{\bsfUpsilon}{0}{bsfletters}{'007}
\DeclareMathSymbol{\ssfUpsilon}{0}{ssfletters}{'007}
\DeclareMathSymbol{\bsfPhi}{0}{bsfletters}{'010}
\DeclareMathSymbol{\ssfPhi}{0}{ssfletters}{'010}
\DeclareMathSymbol{\bsfPsi}{0}{bsfletters}{'011}
\DeclareMathSymbol{\ssfPsi}{0}{ssfletters}{'011}
\DeclareMathSymbol{\bsfOmega}{0}{bsfletters}{'012}
\DeclareMathSymbol{\ssfOmega}{0}{ssfletters}{'012}

\newcommand{\tilc}{\tilde{c}}

\newcommand{\tilF}{\tilde{F}}

\newcommand{\tilg}{\tilde{g}}

\newcommand{\tilR}{\tilde{R}}

\newcommand{\hatS}{\hat{S}}

\newcommand{\haty}{\hat{y}}
\newcommand{\hatY}{\hat{Y}}

\newcommand{\tilY}{\tilde{Y}}

\newcommand{\tilZ}{\tilde{Z}}

\newcommand{\btheta}{\bm{\theta}}

\newcommand{\bup}{\bm{\upsilon}}
\newcommand{\veps}{\varepsilon}

\newcommand{\bxi}{\bm{\xi}}

\newcommand{\bmu}{\bm{\mu}}

\DeclareMathOperator*{\argmin}{arg\,min}

\DeclareMathOperator{\sgn}{sgn}

\DeclareMathOperator{\cov}{\mathsf{Cov}}

\DeclareMathOperator{\rank}{rank}

\newcommand{\bzero}{\mathbf{0}}

\usepackage{jmlr2e}

\usepackage{float}
\usepackage{extarrows}
\usepackage{enumitem}
\usepackage{makecell}

\newcommand{\mathbbm}[1]{\text{\usefont{U}{bbm}{m}{n}#1}}

\usepackage[level]{datetime}

\newcommand{\reg}{\mathrm{reg}}
\newcommand{\HD}{\mathrm{\Delta h}}

\usepackage{ulem}
\usepackage{xifthen}
\usepackage{pdfpages}
\usepackage{transparent}

\usepackage{mathtools}
\mathtoolsset{showonlyrefs}
\usepackage{color}

\usepackage{wrapfig}

\makeatletter 
\newenvironment{subtheorem}[1]{%
	\def\subtheoremcounter{#1}%
	\refstepcounter{#1}%
	\protected@edef\theparentnumber{\csname the#1\endcsname}%
	\setcounter{parentnumber}{\value{#1}}%
	\setcounter{#1}{0}%
	\expandafter\def\csname the#1\endcsname{\theparentnumber.\Alph{#1}}%
	\ignorespaces
}{%
	\setcounter{\subtheoremcounter}{\value{parentnumber}}%
	\ignorespacesafterend
}
\makeatother
\newcounter{parentnumber}
\usepackage{mathtools}
\mathtoolsset{showonlyrefs}

\usepackage{lastpage}
\jmlrheading{23}{2022}{1-\pageref{LastPage}}{5/22; Revised
	8/22}{8/22}{22-0541}{Haiyun He, Hanshu Yan, and Vincent Y.~F.~Tan}

\ShortHeadings{Information-Theoretic Generalization Error
	for Semi-Supervised Learning}{He, Yan, and Tan}
\firstpageno{1}

\begin{document}

\title{Information-Theoretic Characterization of the Generalization Error
	for Iterative Semi-Supervised Learning}

\author{\name Haiyun He \email haiyun.he@u.nus.edu \\ 
	\addr Department of Electrical and Computer Engineering, \\
	National University of Singapore,\\
	117583 Singapore
	\AND 
	\name Hanshu Yan \email hanshu.yan@u.nus.edu \\
	\addr Department of Electrical and Computer Engineering, \\
	National University of Singapore,\\
	117583 Singapore 
	\AND 
	\name Vincent~Y.~F.~Tan \email vtan@nus.edu.sg \\
	\addr  Department of Mathematics, \\
	Department of Electrical and Computer Engineering, \\
	Institute of Operations Research and Analytics,\\
	National University of Singapore,\\
	119076, Singapore}

\editor{Aarti Singh}

\maketitle

\begin{abstract}
Using information-theoretic principles, we consider the generalization error (gen-error) of iterative semi-supervised learning (SSL) algorithms that iteratively generate pseudo-labels for a large amount of unlabelled data  to progressively refine the model parameters.
In contrast to most previous works that {\em bound} the gen-error, we provide an {\em exact} expression for the gen-error and particularize it to  the binary Gaussian mixture model. 
Our theoretical results suggest that when the class conditional variances are not too large, the gen-error decreases with the number of iterations, but quickly saturates.  On the flip side, if the class conditional variances (and so amount of overlap between the classes) are large, the gen-error  increases with the number of iterations. To mitigate this undesirable effect, we show that regularization can reduce  the gen-error.
The theoretical results are corroborated by extensive experiments on  the MNIST and CIFAR datasets in which we notice that  for easy-to-distinguish classes, the gen-error improves after several pseudo-labelling iterations, but saturates afterwards, and for more difficult-to-distinguish classes, regularization improves the generalization performance. 
\end{abstract}


\section{Introduction}
In real-life machine learning applications,  it is relatively easy and inexpensive to obtain large amounts of unlabelled data, while the number of labelled data examples is usually small due to the high cost of annotating them with true labels. In light of this,  semi-supervised learning (SSL) has come to the fore \citep{books/mit/06/CSZ2006,zhu2008semi,van2020survey}. SSL makes use of the abundant unlabelled data to augment the performance of learning tasks with few labelled data examples. This has been shown to outperform supervised and unsupervised learning under certain conditions. For example, in a classification problem, the correlation between the additional unlabelled data  and the labelled data may help to enhance the accuracy of classifiers. Among the plethora of SSL methods, pseudo-labelling \citep{lee2013pseudo} has been observed to be a simple and efficient way to improve the generalization performance empirically. In this paper, we consider the problem  of pseudo-labelling a subset of the unlabelled data at each iteration based on the previous output parameter and then refining the model progressively, but
we are interested in analysing this procedure theoretically.  Our goal in this paper is to understand the impact of pseudo-labelling on the generalization error. 

A learning algorithm can be viewed as a randomized map from the training dataset to the output model parameter. The output is highly data-dependent and may suffer from overfitting to the given dataset. In statistical learning theory, the \emph{generalization error (gen-error)}, or generalization bias, is defined as the expected gap between the test and training losses, and is used to measure the extent to which the algorithms overfit to the training data \citep{russo2016controlling,xu2017information,kawaguchi2017generalization}. In SSL problems, the unlabelled data are expected to improve the generalization performance in a certain manner and thus, it is a worthy endeavor to investigate the behaviour theoretically.  Although there exist many works studying the gen-error for supervised learning problems, the gen-error of SSL algorithms is yet to be explored.  

\subsection{Related Works}\label{Sec:related works}
The extensive literature review is categorized into three aspects.

\paragraph{Semi-supervised learning:} There have been many existing results discussing about various methods of SSL. The book by \citet{books/mit/06/CSZ2006} presented a comprehensive overview of the SSL methods both theoretically and practically. \citet{chawla2005learning} presented an empirical study of
various SSL techniques on a variety of datasets and investigated sample-selection bias when the labelled and unlabelled data are from different distributions. \citet{zhu2008semi} partitioned SSL methods into six main classes: generative models, low-density separation methods, graph-based methods, self-training and co-training. 
Pseudo-labelling is a technique among the self-training and co-training \citep{zhu2009introduction}. In self-training, the model is initially trained by the limited number of labelled data and generate pseudo-labels to the unlabelled data. Subsequently, the model is retrained with the pseudo-labelled data and repeats the process iteratively. It is a simple and effective SSL method without restrictions on the data samples \citep{triguero2015self}. A variety of works have also shown the benefits of utilizing the unlabelled data.
\citet{singh2008unlabeled} developed a finite sample analysis that characterized how the unlabelled data  improves the excess risk compared to the supervised learning, with respect to the number of unlabelled data and the margin between different classes. \citet{li2019multi} studied multi-class classification with unlabelled data and provided a sharper generalization error bound using the notion of Rademacher complexity that yields a faster convergence rate. \citet{pmlr-v124-zhu20b} considered the general SSL setting by assuming the loss function to be $\beta$-exponentially concave or the $0$-$1$ loss, and used a Bayesian method for prediction instead of empirical risk minimization which we consider. The author presented an upper bound for the excess risk and the learning rate in terms of the number of labelled and unlabelled data examples.
 \citet{carmon2019unlabeled} proved that using unlabelled data can help to achieve high robust accuracy as well as high standard accuracy at the same time. \citet{dupre2019improving} considered iteratively pseudo-labelling the whole unlabelled dataset with a confidence threshold and showed that the accuracy converges relatively quickly.  \citet{oymak2021theoretical}, in which part of our analysis hinges on, studied SSL under  the binary Gaussian mixture model setup and characterized the correlation between the learned and the optimal estimators concerning the margin and the regularization factor.  Recently, \citet{pmlr-v151-aminian22a} considered the scenario where the labelled and unlabelled data are not generated from the same distribution and these distributions may change over time, exhibiting so-called covariate shifts. They provided an upper bound for the gen-error and proposed the Covariate-shift SSL (CSSL) method which outperforms some previous SSL algorithms under this setting.
 However, these works do not investigate how the unlabelled data   affects the generalization error over the iterations.

\paragraph{Generalization error bounds:} The traditional way of analyzing generalization error involves using the Vapnik--Chervonenkis or VC dimension \citep{vapnik2000} and  the Rademacher complexity  \citep{boucheron2005theory}. Recently, \citet{russo2016controlling} proposed using the mutual information between the estimated output of an algorithm and the actual realized value of the estimates to analyze and bound the bias in data analysis, which can be regarded equivalent to the generalization error. This new approach is simpler and can handle a wider range of loss functions compared to the abovementioned methods and other methods such as differential privacy. It also paves a new way to improving generalization capability of learning algorithms from an information-theoretic viewpoint. Following \citet{russo2016controlling}, \citet{xu2017information} derived upper bounds on generalization error  of learning algorithms with mutual information between the input dataset and the output hypothesis, which formalizes the intuition that less information that a learning algorithm can extract from training dataset leads to less overfitting. Later \citet{pensia2018generalization} derived generalization error bounds for noisy and iterative algorithms and the key contribution is to bound the mutual information between input data and output hypothesis.  \citet{negrea2019information} improved mutual information bounds for Stochastic Gradient Langevin Dynamics (SGLD) via data-dependent estimates compared to distribution-dependent bounds.

However, one  major shortcoming of the aformentioned mutual information bounds is that the bounds  go  to infinity for (deterministic) learning algorithms without noise, e.g., Stochastic Gradient Descent (SGD). Some other works have tried to overcome this problem. \citet{lopez2018generalization} derived upper bounds on the generalization error using the Wasserstein distance involving the distributions of input data and output hypothesis, which are shown to be tighter under some natural cases. \citet{esposito2021generalization} derived generalization error bounds via R\'{e}nyi-, $f$-divergences and maximal leakage. \citet{steinke2020reasoning} proposed using the Conditional Mutual Information (CMI) to bound the generalization error; the CMI is useful as it possesses  the chain rule property. \citet{bu2020tightening} provided a tightened upper bound based on the \emph{individual} mutual information (IMI) between the \emph{individual} data sample and the output.   
\citet{wu2020information} extended \citet{bu2020tightening}'s result to   transfer learning problems and characterized the upper bound based on IMI and KL-divergence.
In a similar manner, \citet{jose2020information} provided a tightened bound on transfer generalization error based on the Jensen--Shannon divergence. Moreover, recently, \citet{aminian2021exact} and \citet{bu2021characterizing}   recently derived the exact characterization of gen-error for supervised learning and transfer learning with the Gibbs algorithm.

Regularization is an important technique to reduce the model variance  \citep{anzai2012pattern}, but there are few works that theoretically analyse the relationship between the gen-error and regularization. \citet{moody1992effective} characterized the gen-error as a function of the  regularization parameter in supervised nonlinear learning systems and showed that the gen-error decreases as the parameter increases.  \citet{bousquet2002stability} provided a stability-based gen-error upper bound in terms of the regularization parameter in supervised learning. \citet{mignacco2020role} studied how the regularization affects the expected accuracy in high-dimensional GMM supervised classification problem.

\paragraph{Gaussian mixture models (GMM):} The GMM is a popular, simple but non-trivial model that has been studied by many researchers.   The performance of GMM classification problems depends on the data structure.  The classical work of \citet{castelli1996relative} studied the classification problem in a binary mixture model with known conditional distributions but unknown mixing parameter and characterized the relative value of labelled and unlabelled data in improving the convergence rate of classification error probability.
\citet{akaho2000nonmonotonic} characterized the generalization bias of general GMMs in supervised learning and discussed its dependency on data noise. \citet{watanabe2006stochastic} considered GMM in Bayesian learning and provided bounds for variational stochastic complexity. \citet{wang2021binary} and \citet{muthukumar2021classification} studied the dependence of the bGMM classification performance (using the $0$-$1$ loss) on the structure of data covariance by considering SVM and linear interpolation. 

However, all these aforementioned works do not investigate the generalization performance of SSL algorithms.

\subsection{Contributions}

Our main contributions are as follows.
\begin{enumerate}%[leftmargin= 12 pt, topsep=-3pt,itemsep=0pt]
	\item In Section \ref{Sec:preliminaries}, we leverage results by \citet{bu2020tightening} and \citet{wu2020information} to derive an information-theoretic gen-error bound at each iteration for iterative SSL; see Theorem~\ref{Coro: sub Gau gen}. Moreover,  in contrast to most previous works that bound the gen-error, we derive an \emph{exact} characterization of gen-error at each iteration for \emph{negative log-likelihood (NLL)} loss functions (see Theorem \ref{Thm:exact gen}).
	
	\item In Section \ref{Sec:main results}, we particularize Theorem \ref{Thm:exact gen} to the binary Gaussian mixture model (bGMM) with in-class variance $\sigma^2$. We show that for any fixed number of data samples, there exists a critical value   $\sigma_0$ such that when the data variance (representing the overlap between classes)  $\sigma^2<\sigma_0^2$, the gen-error decreases in the iteration count $t$ and converges quickly with a sufficiently large amount of unlabelled data. When $\sigma^2>\sigma_0^2$, the gen-error increases instead, which means using the unlabelled data does not help to reduce the gen-error across the SSL iterations.  The empirical gen-error corroborates the theoretical results, which suggests that the characterization serves as a useful rule-of-thumb to understand  how the gen-error changes across the SSL iterations and it  can be used to establish conditions under which unlabelled data can help in terms of generalization.

	\item In Section \ref{Sec:regularization}, we theoretically and empirically show that for difficult-to-classify problems with large overlap between classes, regularization can effectively help to mitigate the undesirable increase of the gen-error across the SSL iterations.
	
	\item In Section \ref{Sec:experiments}, we implement the pseudo-labelling procedure on  the MNIST and CIFAR datasets with few labelled data and abundant unlabelled data. The experimental results corroborate the phenomena for the bGMM that  the gen-error decreases quickly in the early pseudo-labelling iterations and saturates  thereafter for easy-to-distinguish classes but increases for hard-to-distinguish classes. By adding $\ell_2$-regularization to the hard-to-distinguish problem, we also observe  improvements to the  gen-error similar to that for the bGMM.
%	
%	\item 	In the appendix, we show that with a sufficiently large number of unlabelled data, reusing the labelled data is not necessarily helpful. We also verify that the upper bound in Theorem \ref{Coro: sub Gau gen} has the similar behaviour as the empirical gen-error.
\end{enumerate}

\section{Problem Setup}

Let the instance space be $\calZ=\calX\times \calY\subset\bbR^{d+1}$, the model parameter space be $\Theta $ and the loss fucntion be $l:\calZ \times \Theta \to \bbR $, where $d\in\bbN$. 
We are given a labelled training dataset $S_{\rml}=\{Z_1,\ldots,Z_n\}=\{(X_i,Y_i)\}_{i=1}^n$ drawn from $\calZ$, where each $Z_i=(X_i,Y_i)$ is independently and identically distributed (i.i.d.) from $P_{Z}=P_{X,Y}\in\calP(\calZ)$. For any $i\in[n]$, $X_i$ is a vector of features and $Y_i$ is a label indicating the class to which $X_i$ belongs. However, in many real-life  machine learning applications, we only have  a limited number of labelled data while we have access to a large amount of unlabelled data, which are expensive to annotate. Then we can incorporate the unlabelled training data together with the labelled data to improve the performance of the model. This procedure is called  
\emph{semi-supervised learning (SSL)}. We are given an independent unlabelled training dataset $S_{\rmu}=\{X'_1,\ldots,X'_{\tau m}\}, \tau\in\bbN$, where each $X'_i$ is  generated i.i.d.\ from $P_X\in\calP(\calX)$. Typically, $m\gg n$.

In the following, we consider the \emph{iterative self-training with pseudo-labelling} in SSL setup, as shown in Figure \ref{Fig:system}. Let $t\in[0:\tau]$ denote the iteration count. In the initial round ($t=0$), the labelled data $S_{\rml}$ are first used to learn an initial model parameter $\theta_0\in\Theta$. Next, we split the unlabelled dataset $S_{\rmu}$ into $\tau$ disjoint equal-size sub-datasets $\{S_{\rmu,k}\}_{k=1}^\tau$, where $S_{\rmu,k}=\{X'_{(k-1)m+1},\ldots,X'_{km}\}$. In each subsequent round $t\in[1:\tau]$, based on $\theta_{t-1}$ trained from the previous round, we use a predictor $f_{\theta_{t-1}}:\calX \mapsto \calY$ to assign a \emph{pseudo-label} $\hatY'_i$ to the unlabelled sample $X'_i$ for all $i\in \calI_t:=\{(t-1)m,(t-1)m+1,\ldots,tm \}$. Let $\hatS_{\rmu,t}=\{(X_i',\hatY_i')\}_{i\in\calI_t}$ denote the $t^{\mathrm{th}}$ pseudo-labelled dataset. After pseudo-labelling, both the labelled data $S_{\rml}$ and the pseudo-labelled data $\hatS_{\rmu,t}$  are used to learn a new model parameter $\theta_t$. The procedure is then repeated iteratively until the maximum number of iterations $\tau$ is reached. 

\begin{figure}[t]
	\centering
	\hspace{-0.4cm}
	\includegraphics[width=.99\linewidth]{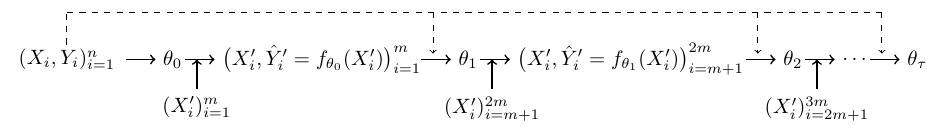}
	\caption{Paradigm of iterative self-training with pseudo-labelling in SSL}
	\label{Fig:system}
	\vspace{-.1in}
\end{figure}
This setup is a classical and widely-used model in the realm of self-training in SSL \citep{books/mit/06/CSZ2006,zhu2008semi,zhu2009introduction,lee2013pseudo}, where in each iteration, only a subset of the unlabelled data are used. Furthermore, as discussed by \citet{arazo2020pseudo}, this method is less likely to overfit to incorrect pseudo-labels, compared to using all the unlabelled data in each iteration (also see Figure~\ref{Fig: horse-ship all vs subset}). Under this setup of iterative SSL, during each iteration $t$, our \emph{goal} is to find a model parameter $\theta_t\in\Theta$ that minimizes the population risk with respect to the underlying data distribution
\begin{align}
	L_{P_Z}(\theta_t):=\bbE_{Z\sim P_Z}[l(\theta_t,Z)].
\end{align}
Since $P_Z$ is unknown, $L_{P_Z}(\theta_t)$ cannot be computed directly. Hence, we instead minimize the empirical risk. The procedure is termed \emph{empirical risk minimization} (ERM). For any model parameter $\theta_t\in\Theta$, the empirical risk of the labelled data is defined as
%\vspace{-0.5em}
\begin{align}
	L_{S_{\rml}}(\theta_t):=\frac{1}{n}\sum_{i=1}^n l(\theta_t,Z_i),
\end{align}
and for $t\geq 1$, the empirical risk of pseudo-labelled data $\hatS_{\rmu,t}$ as
\begin{align}
	L_{\hatS_{\rmu,t}}(\theta_t):=\frac{1}{m}\sum_{i \in \calI_t} l(\theta_t,(X_i',\hatY_i')).
\end{align}
We set $L_{\hatS_{\rmu,t}}(\theta_t)=0$ for $t=0$. For a fixed weight $w\in[0,1]$, the total empirical risk can be defined as the following linear combination of $L_{S_{\rml}}(\theta_t)$ and $L_{\hatS_{\rmu,t}}(\theta_t)$:
\begin{align}
	L_{S_{\rml},\hatS_{\rmu,t}}(\theta_t)
	&:=w L_{S_\rml}(\theta_t)+(1-w)L_{\hatS_{\rmu,t}}(\theta_t). \label{Eq: def of empirical risk}
\end{align} 
%The weight $w$ is commonly chosen as $\frac{n}{n+m}$ and then the empirical risk $L_{S_{\rml},\hatS_{\rmu,t}}(\theta_t)$ can be rewritten as
%\begin{align}
%	L_{S_{\rml},\hatS_{\rmu}}(\theta_t)
%	&=\frac{1}{n+m}\sum_{i=1}^n l(\theta_t,Z_i)+\frac{1}{n+m}\sum_{i=(t-1)m+1}^{tm}l(\theta_t,(X_i',\hatY_i')).
%\end{align}
In the usual case where the algorithm minimizes the average of the empirical training losses, one should set $w=\frac{n}{n+m}$.
An SSL algorithm can be characterized by a randomized map from the labelled and unlabelled training data $S_{\rml}$, $S_{\rmu}$ to a model parameter $\theta$ according to a conditional distribution $P_{\theta|S_{\rml},S_{\rmu}}$. Then at each iteration $t$, we can use the sequence of conditional distributions $\{P_{\theta_k|S_{\rml},S_{\rmu}}\}_{k=0}^t$ with $P_{\theta_0|S_{\rml},S_{\rmu}}=P_{\theta_0|S_{\rml}}$ to represent an iterative SSL algorithm.  The \emph{generalization error} at the $t$-th iteration is defined as the expected \emph{gap} between the population risk of $\theta_t$ and the empirical risk on the training data:
%{\setlength{\belowdisplayskip}{5pt}
\begin{align}
	&\mathrm{gen}_t(P_Z, P_X, \{P_{\theta_k|S_{\rml},S_{\rmu}}\}_{k=0}^t, \{f_{\theta_k}\}_{k=0}^{t-1} ) \nn\\
	&:=\bbE[L_{P_Z}(\theta_t)-	L_{S_{\rml},\hatS_{\rmu,t}}(\theta_t)]\\
	&=w\bigg(\bbE_{\theta_t}\bbE_{Z}[l(\theta_t,Z)]-\frac{1}{n}\sum_{i=1}^n \bbE_{\theta_t,Z_i}[l(\theta_t,Z_i)]    \bigg)   \nn\\
	&\quad +   (1-w  ) \bigg( \bbE_{\theta_t}\bbE_{Z}[l(\theta_t,Z)]  -\frac{1}{m} \sum_{i\in\calI_t} \bbE_{\theta_t,X'_i,\hatY'_i}[l(\theta_t,(X'_i,\hatY'_i))]   \bigg). \label{Def: gent}
\end{align}%}
When $t=0$ and $w=1$, the definition of the generalization error reduces to that of vanilla supervised learning. Based on this definition, the expected population risk can be decomposed as 
\begin{equation}
\bbE[L_{P_Z}(\theta_t) ]=\bbE[L_{S_\rml,\hatS_{\rmu,t}}(\theta_t)]+\mathrm{gen}_t, \label{eqn:decomp}
\end{equation}
 where the first term  on the right-hand side of this equation is what the algorithm minimizes and reflects how well the output hypothesis fits the dataset, and the second term $\mathrm{gen}_t$ is used to measure the extent to which the iterative learning algorithm overfits the training data at the $t$-th iteration. To minimize $\bbE[L_{P_Z}(\theta_t) ]$, we need both terms in \eqref{eqn:decomp} to be small, but there exists a natural trade-off between them. While the algorithm aims to minimize the empirical risk $\bbE[L_{S_\rml,\hatS_{\rmu,t}}(\theta_t)]$,   studying and controlling $\mathrm{gen}_t$ can also help to reduce the population risk $\bbE[L_{P_Z}(\theta_t) ]$, which is the ultimate goal of learning. Instead of focusing on the total generalization error induced during the entire process, we are   interested in the following questions. How does $\mathrm{gen}_t$ evolve as  $t$ increases? Do the unlabelled data examples in $S_\rmu$ help to improve the generalization error?

\section{General Results}\label{Sec:preliminaries}
Inspired by the information-theoretic generalization results in \citet[Theorem~1]{bu2020tightening} and  \citet[Theorem 1]{wu2020information}, we derive an upper bound on the gen-error $\mathrm{gen}_t$ in terms of the mutual information between input data samples (either labelled or pseudo-labelled) and the output model parameter $\theta_t$, as well as the KL-divergence between the   data distribution and the joint distribution of feature vectors and pseudo-labels (cf.\ Theorem~\ref{Coro: sub Gau gen}). Furthermore, by considering the  NLL loss function \citep{mackay2003information,goodfellow2016deep}, we derive the exact characterization for the gen-error $\mathrm{gen}_t$ (cf.\ Theorem~\ref{Thm:exact gen}).

Recall that for a given $R>0$,   $L$ is an   {\em $R$-sub-Gaussian random variable}~\citep{vershynin_2018} if  its {\em cumulant generating function} $\Lambda_L(\lambda):=\log\bbE[\exp(\lambda(L-\bbE[L]))]\le \exp(\lambda^2R^2/2)$ for all $\lambda\in\bbR$. If $L$ is $R$-sub-Gaussian, we write this   as $L\sim \mathrm{subG}(R)$. Furthermore, let us recall the following somewhat non-standard information quantities \citep{negrea2019information,haghifam2020sharpened}.
\begin{definition}
	For   random variables $X$, $Y$ and $U$, define the \emph{disintegrated mutual information} between $X$ and $Y$ given $U$ as $I_{U}(X;Y):=D(P_{X,Y|U}\|P_{X|U}\otimes P_{Y|U})$, and the \emph{disintegrated KL-divergence} between $P_X$ and $P_Y$ given $U$ as $D_U(P_X\|P_Y):=D(P_{X|U}\|P_{Y|U})$. These are $\sigma(U)$-measurable random variables.
	It follows that the conditional mutual information $I(X;Y|U)=\bbE_{U}[I_{U}(X;Y)]$  and the conditional KL-divergence $D(P_{X|U}\|P_{Y|U}| P_U)=\bbE_U[D_U(P_X\|P_Y)]$.
	
	For   distributions $P$, $Q$ and $V$, define the {\em cross-entropy} as $h(P,Q):=\bbE_{P}[-\log Q]$ and the {\em divergence between the cross-entropies} as $\HD(P\|Q|V):=h(P,V)-h(Q,V)$.
\end{definition}

 Let $\theta^{(t)}=(\theta_0,\ldots,\theta_t)$ for any $t\in[0:\tau]$ and $w=1$ for $t=0$.  In iterative SSL, we can characterize the gen-error as shown in Theorem \ref{Thm:gen thm} by applying the law of total expectation.
 
 \begin{subtheorem}{theorem}\label{Thm:gen thm}
 	\begin{theorem}[Gen-error upper bound for iterative SSL]\label{Coro: sub Gau gen}
 		Suppose $l(\theta,Z)\sim \mathrm{subG}(R)$ under $Z\sim P_Z$ for all $\theta\in\Theta$, then for any $t\in[0:\tau]$, %the upper bound for $\big|\mathrm{gen}_t(P_Z, P_X, \{P_{\theta_k|S_{\rml},S_{\rmu}}\}_{k=0}^t, \{f_{\theta_k}\}_{k=0}^{t-1} ) \big|$ is given by
 		\begin{align}\label{Eq: iterative IMI}
 			&\big|\mathrm{gen}_t(P_Z, P_X, \{P_{\theta_k|S_{\rml},S_{\rmu}}\}_{k=0}^t, \{f_{\theta_k}\}_{k=0}^{t-1} ) \big| \nn\\
 			&\leq \frac{w}{n}\sum_{i=1}^n \bbE_{\theta^{(t-1)}}\Big[\sqrt{2 R^2 I_{\theta^{(t-1)}}^{(i)}} \Big]+\frac{1-w}{m} \sum_{i\in\calI_t}   \bbE_{\theta^{(t-1)}}\Big[ \sqrt{2R^2 \big(I_{\theta^{(t-1)}}'^{(i)}+D_{\theta^{(t-1)}}'^{(i)}  \big) }\Big],
 		\end{align}
 		where $I_{\theta^{(t-1)}}^{(i)}:=I_{\theta^{(t-1)}}(\theta_t;Z_i)$, $I_{\theta^{(t-1)}}'^{(i)}:=I_{\theta^{(t-1)}}(\theta_t;X_i',\hatY_i')$, and $D_{\theta^{(t-1)}}'^{(i)}:=D_{\theta^{(t-1)}}(P_{X_i',\hatY_i'}\| P_Z)$.
 	\end{theorem}
 	
 	\begin{theorem}[Exact gen-error for iterative SSL]\label{Thm:exact gen}
 		Consider the NLL loss function $l(\theta,Z)=-\log p_{\theta}(Z)$, where $p_{\theta}(Z)$ is the likelihood of $Z$ under parameter $\theta$. For any $t\in[0:\tau]$,
 		\begin{align}
 			&\mathrm{gen}_t(P_Z, P_X, \{P_{\theta_k|S_{\rml},S_{\rmu}}\}_{k=0}^t, \{f_{\theta_k}\}_{k=0}^{t-1} )\nn\\
 			&= \bbE_{\theta^{(t)}} \bigg[\frac{w}{n}\sum_{i=1}^n   \HD^{(i)}_{\theta_t} +  \frac{1-w}{m}   \sum_{i\in\calI_t} \big( \HD'^{(i)}_{\theta^{(t)}}+\widetilde{\HD}'^{(i)}_{\theta^{(t)}}\big)  \bigg], \label{Eq:exact gent}
 		\end{align}
 		where 
% 		$\HD^{(i)}_{\theta_t}:=\HD(P_Z\|P_{Z_i|\theta_t}|p_{\theta_t})$, 
% 			$\HD'^{(i)}_{\theta^{(t)}} := \HD(P_Z\|P_{X_i',\hatY_i'|\theta^{(t-1)}}|p_{\theta_t})$, and\\ %\quad \mbox{and}\\
% 			 $\widetilde{\HD}'^{(i)}_{\theta^{(t)}} := \HD(P_{X_i',\hatY_i'|\theta^{(t-1)}}\| P_{X_i',\hatY_i'|\theta^{(t)}}|p_{\theta_t} )$.
 		\begin{align}
 			 \HD^{(i)}_{\theta_t}&:=\HD(P_Z\|P_{Z_i|\theta_t}|p_{\theta_t}), \qquad
 			\HD'^{(i)}_{\theta^{(t)}} := \HD(P_Z\|P_{X_i',\hatY_i'|\theta^{(t-1)}}|p_{\theta_t}),\quad \mbox{and}\\
 			 \widetilde{\HD}'^{(i)}_{\theta^{(t)}}&:= \HD(P_{X_i',\hatY_i'|\theta^{(t-1)}}\| P_{X_i',\hatY_i'|\theta^{(t)}}|p_{\theta_t} ).
 		\end{align}
 	\end{theorem}
 \end{subtheorem}

The proof of Theorem~\ref{Coro: sub Gau gen} is provided in Appendix \ref{pf of Thm: iterative gen bound}, in which we provide a general  upper bound not only applicable to sub-Gaussian loss functions. The proof of Theorem~\ref{Thm:exact gen} is provided in Appendix~\ref{pf of Thm:exact gen}. Specifically, for NLL loss functions, Theorem~\ref{Thm:exact gen} provides an {\em exact} characterization of the gen-error at each iteration. This is in stark contrast to most works on information-theoretic generalization error in which only {\em bounds} are provided. 

In contrast to \citet[Theorem 1]{bu2020tightening} and  \citet[Theorem 1]{wu2020information} which pertain to supervised learning, Theorem \ref{Thm:gen thm} characterizes the gen-error at each iteration during the pseudo-labelling and training process.  Note that the quantities in Theorem \ref{Thm:exact gen} satisfy
$\HD^{(i)}_{\theta_t}=I_{\theta^{(t-1)}}^{(i)}+D(P_Z\|p_{\theta_t})-D(P_{Z_i|\theta_t}\|p_{\theta_t})$ and $\widetilde{\HD}'^{(i)}_{\theta^{(t)}}=I_{\theta^{(t-1)}}'^{(i)}+D_{\theta^{(t-1)}}(P_{X_i',\hatY_i'}\|p_{\theta_t})-D_{\theta^{(t-1)}}(P_{X_i',\hatY_i'|\theta_t}\|p_{\theta_t})$. Thus, it is plausible that the upper bound based on $I_{\theta^{(t-1)}}^{(i)}$ and $I_{\theta^{(t-1)}}'^{(i)}$ in Theorem \ref{Coro: sub Gau gen} can help to understand and control the exact gen-error.
Intuitively,  the mutual information between the individual input data sample $Z_i$ and the output model parameter $\theta_t$ in Theorem~\ref{Coro: sub Gau gen} and the cross-entropy divergences $\HD^{(i)}_{\theta_t}$, $\widetilde{\HD}'^{(i)}_{\theta^{(t)}} $ in Theorem~\ref{Thm:exact gen} both measure the extent to which the algorithm is sensitive to each data example at each iteration $t$. The KL-divergence between the underlying $P_Z$ and pseudo-labelled distribution $P_{X_i',\hatY_i'}$ in Theorem~\ref{Coro: sub Gau gen} and the cross-entropy divergence $\HD'^{(i)}_{\theta^{(t)}}$ in Theorem~\ref{Thm:exact gen} measure how effectively the pseudo-labelling process works.  As $n\to\infty$ and $m\to\infty$, we show that the mutual information (as well as $\HD^{(i)}_{\theta_t}$ and $\widetilde{\HD}'^{(i)}_{\theta^{(t)}} $) vanishes but the divergences $D_{\theta^{(t-1)}}'^{(i)}$ and $\HD'^{(i)}_{\theta^{(t)}}$ do not, which reflects the impact  of pseudo-labelling on the gen-error. 

%we show that $I_{\theta^{(t-1)}}(\theta_t;X_i',\hatY_i')$ tends to $0$ (in probability), which means that there are sufficient training data such that the algorithm can generalize well. On the other hand, the impact on the generalization error of pseudo-labelling is reflected in the KL-divergence $D_{\theta^{(t-1)}}(P_{X_i',\hat{Y}'_i}\| P_Z)$ and this term does not necessarily vanish as $n,m\to\infty$. We quantify this precisely in Remark \ref{rmk:bGMM} in Section \ref{Sec:main results}.

%In iterative learning algorithms, it is usually difficult to directly calculate the mutual information and KL-divergence between the input and the final output \citep{paninski2003estimation,nguyen2010estimating,mcallester2020formal}. 

In iterative learning algorithms, by applying the law of total expectation and conditioning the information-theoretic quantities on the output model parameters $\theta^{(t-1)}=\{\theta_1,\ldots,\theta_{t-1}\}$ from previous iterations, we are able to calculate the gen-error iteratively. 
In the next section, we apply the exact iterated gen-error in Theorem \ref{Thm:exact gen} to a classification problem under a specific generative model---the bGMM. This simple model allows us to derive a   tractable characterization on the gen-error as a function of iteration number $t$ that we can compute numerically.

\section{Main Results on bGMM}\label{Sec:main results}
We now particularize the iterative semi-supervised classification  setup to the bGMM. We evaluate   \eqref{Eq:exact gent}   to understand the effect of multiple self-training rounds on the gen-error.
%In the following, we use the bold letters to denote vectors. For example, the notations $X$ and $\theta$ used in the previous sections are replaced with $\bX$ and $\btheta$, respectively. Given any $\btheta$, we let $p_{\btheta}$ denote the probability density function under the parameter $\btheta$.

\subsection{Iterative SSL under bGMM}
Fix a unit vector $\bmu\in\bbR^d$ and a scalar $\sigma\in\bbR_+ = (0,\infty)$. Under the bGMM with mean $\bmu$ and standard deviation (std.\ dev.) $\sigma$ (bGMM($\bmu,\sigma$)), we assume that the distribution of any labelled data example $(\bX,Y)$ can be specified as follows.
Let  $\calY=\{-1,+1\}$, $Y \sim P_Y=\mathrm{unif}\{-1,+1\}$, and $\bX|Y\sim \calN(Y\bmu,\sigma^2 \bI_d)$, where $\bI_d$ is the identity matrix of size $d\times d$.

The random vector $\bX$ is distributed according to the mixture distribution $$p_{\bmu}=\frac{1}{2}\calN(\bmu,\sigma^2\bI_d)+\frac{1}{2}\calN(-\bmu,\sigma^2\bI_d).$$
%\begin{align}
%	p_{\bmu}(\bx)=&\frac{1}{2\sqrt{(2\pi)^d }\sigma^d }\exp\bigg(-\frac{1}{2\sigma^2}(\bx-\bmu)^\top(\bx-\bmu) \bigg) \nn\\
%	&+\frac{1}{2\sqrt{(2\pi)^d}\sigma^d }\exp\bigg(-\frac{1}{2\sigma^2}(\bx+\bmu)^\top(\bx+\bmu) \bigg), \quad \forall \bx\in\bbR^d.
%\end{align}
In the unlabelled dataset $S_{\rmu}$, each $\bX_i'$ for $i\in[1:\tau m]$ is drawn i.i.d.\ from   $p_{\bmu}$.

Let $\Theta\subset\bbR^d$ such that $\bmu\in\Theta$. For any $\btheta\in\Theta$, under the bGMM($\btheta,\sigma$), the joint distribution of any pair of $(\bX,Y)\in\calZ$ is given by $ \calN(Y\btheta,\sigma^2 \bI_d)\otimes P_Y$. The NLL loss function can be expressed as
\begin{align}
	 l(\btheta,(\bX,Y)) 
	&=  -\log p_{\btheta}(\bX,Y)  
	=  -\log \big(P_Y(Y)p_{\btheta}(\bX|Y) \big) \nn\\*
	&=-\log \frac{1}{2\sqrt{(2\pi)^d}\sigma^d } + \frac{1}{2\sigma^2}(\bX-Y\btheta)^\top(\bX-Y\btheta).
\end{align}
The population risk minimizer is given by $\argmin_{\btheta\in\Theta}\bbE_{\bX,Y}[l(\btheta,(\bX,Y))]=\bmu$.

%In the following, we assume $c\in(\tilc_1,+\infty)$, where $\tilc_1$ is a constant dependent on $\sigma$. 
%given by
%\begin{align}
%	\tilc_1:=\sqrt{\bigg(2\Phi\bigg(\frac{1}{\sigma}\bigg)+\frac{2\sigma}{\sqrt{2\pi}}\bigg)^2+\frac{2\sigma^2}{\pi}}.
%\end{align}
%We also define another constant $\tilc_2$ that is only dependent on $\sigma$
%\begin{align}
%	\tilc_2=1+\sqrt{\bigg(1+2\Phi\bigg(\frac{1}{\sigma}\bigg)+\frac{2\sigma}{\sqrt{2\pi}}\bigg)^2+\frac{2\sigma^2}{\pi}}.
%\end{align}

%
%Furthermore, for any $\btheta\in\Theta$ and any $\bX\in\calX$, we consider the following predictor $f_{\btheta}$ to assign a label to $\bX$
%\begin{align}
%	\hatY=f_{\btheta}(\bX)=\sgn(\btheta^\top \bX).
%\end{align}

Under this setup, the iterative SSL procedure is shown in Figure \ref{Fig:system}, but the labelled dataset $S_{\rml}$ is only used to train in the initial round $t=0$; we discuss the reuse of $S_\rml$ in all iterations in Corollary~\ref{Coro:reuse labeldata gen bound}. That is, in \eqref{Eq: def of empirical risk}, we set $w=0$. The algorithm operates in the following steps.
\begin{itemize}[leftmargin= 12 pt, topsep=-3pt,itemsep=0pt]
	\item  \textbf{Step 1: Initial round $t=0$ with $S_{\rml}$:} By minimizing the empirical risk of labelled dataset $S_{\rml}$
%	\vspace{-0.5em}
	{\setlength{\abovedisplayskip}{3pt}
		\setlength{\belowdisplayskip}{3pt}
	\begin{align}\label{Eq:L_Sl}
		L_{S_{\rml}}(\btheta)=\frac{1}{n}\sum_{i=1}^n l(\btheta,(\bX_i,Y_i)) 
		\overset{\rmc}{=} \frac{1}{2\sigma^2 n}\sum_{i=1}^n(\bX_i-Y_i\btheta)^\top(\bX_i-Y_i\btheta),
	\end{align}	}
	where $\stackrel{\rmc}{=}$ means that both sides differ by a constant independent of $\btheta$,  we obtain the minimizer
	{\setlength{\abovedisplayskip}{3pt}
		\setlength{\belowdisplayskip}{3pt}
	\begin{align}
		\btheta_0=\argmin_{\btheta\in\Theta}L_{S_{\rml}}(\btheta)=\frac{1}{n}\sum_{i=1}^n Y_i \bX_i. \label{Eq: def theta_0}
	\end{align}}

	\item \textbf{Step 2: Pseudo-label data in $S_{\rmu}$:} At each iteration $t\in[1:\tau]$, for any $i\in \calI_t$, we use $\btheta_{t-1}$ to assign a pseudo-label for $\bX_i'$, that is, $\hatY_i'=f_{\btheta_{t-1}}(\bX_i')=\sgn(\btheta_{t-1}^\top \bX_i')$.

	\item \textbf{Step 3: Refine the model:} We then use the pseudo-labelled dataset $\hatS_{\rmu,t}$ to train the new model. By minimizing the empirical risk of $\hatS_{\rmu,t}$
	{\setlength{\abovedisplayskip}{3pt}
		\setlength{\belowdisplayskip}{3pt}
	\begin{align}\label{Eq:L_hatS}
		L_{\hatS_{\rmu,t}}(\btheta) = \frac{1}{m}\sum_{i\in\calI_t } l(\btheta,(\bX'_i,\hatY_i')) 
		 \overset{\rmc}{=}  \frac{1}{2\sigma^2 m} \sum_{i\in\calI_t}(\bX'_i-\hatY_i'\btheta)^\top(\bX'_i-\hatY_i'\btheta),
	\end{align}}
	we obtain the new model parameter
	{\setlength{\abovedisplayskip}{4pt}
		\setlength{\belowdisplayskip}{4pt}
	\begin{align}
		& \btheta_t =\frac{1}{m}\sum_{i\in\calI_t} \hatY_i'\bX'_i = \frac{1}{m} \sum_{i\in\calI_t}  \sgn(\btheta_{t-1}^\top \bX'_i)\bX'_i. \label{Eq:theta_t}
	\end{align}}
	If $t<\tau$, go back to Step 2.
	
\end{itemize}

\subsection{Definitions}
To state our result succinctly, we first define some non-standard notations and functions.
From \eqref{Eq: def theta_0}, we know that $\btheta_0\sim\calN(\bmu,\frac{\sigma^2}{n} \bI_d)$ and inspired by \citet{oymak2021theoretical}, we can decompose $\btheta_0$ as  $$\btheta_0=\Big(1+\frac{\sigma}{\sqrt{n}}\xi_0\Big)\bmu+\frac{\sigma}{\sqrt{n}}\bmu^{\bot},$$
%\begin{align}
%	\btheta_0%&=\bmu+\frac{\sigma}{\sqrt{n}}\bxi\\
%	%&=\bmu+\frac{\sigma}{\sqrt{n}}(\xi_0\bmu+\bmu^{\bot})\\
%	&=\bigg(1+\frac{\sigma}{\sqrt{n}}\xi_0\bigg)\bmu+\frac{\sigma}{\sqrt{n}}\bmu^{\bot},
%\end{align}
where $\xi_0\sim\calN(0,1)$, $\bmu^{\bot}\sim\calN(\bzero,\bI_d-\bmu\bmu^\top)$, and $\bmu^{\bot}$ is perpendicular to $\bmu$ and independent of $\xi_0$ (the details of this decomposition are provided in Appendix \ref{pf of Thm:exact gen GMM}). 

Given a pair of vectors $(\ba,\bb)$, define their correlation coefficient as $\rho(\ba,\bb):=\frac{\langle\ba,\bb\rangle}{\|\ba\|_2\|\bb\|_2}$. 
The correlation coefficient between the estimated and true parameters is %$\btheta_0$ and $\bmu$ is
\begin{align}
     \alpha(\xi_0,\bmu^{\bot})  :=  \rho(\btheta_0,\bmu) =  \frac{1+\frac{\sigma}{\sqrt{n}}\xi_0}{\sqrt{(1 + \frac{\sigma}{\sqrt{n}}\xi_0)^2 + \frac{\sigma^2}{n}\|\bmu^{\bot}\|_2^2}} . \label{Eq:rho_0}
\end{align}
Let  $\beta(\xi_0,\bmu^{\bot})=\sqrt{1-\alpha(\xi_0,\bmu^{\bot})^2}$. We abbreviate  $\alpha(\xi_0,\bmu^{\bot})$ and $\beta(\xi_0,\bmu^{\bot})$ to $\alpha$ and $\beta$ respectively in the following.  Then the normalized vector $\btheta_0 /\|\btheta_0\|_2$ can be decomposed as follows
\begin{align}
	\bar{\btheta}_0:= \frac{\btheta_0 }{\|\btheta_0\|_2}=\alpha\bmu+ \beta \bup, \label{Eq:theta0 decomp mu up}
\end{align}
where $\bup=\bmu^{\bot}/\|\bmu^{\bot}\|_2$. Let
$\bar{\btheta}_0^{\bot}:=(2\beta^2\bmu-2\alpha\beta\bup)/\sigma$, which is a vector perpendicular to $\bar{\btheta}_0$.

Let $\rmQ(\cdot):=1-\Phi(\cdot)$. Define the {\em correlation evolution function} $F_\sigma : [-1,1]\to[-1,1]$ that  quantifies the increase to the correlation (between the current model parameter and the optimal one) and improvement to the generalization error as the iteration counter increases from $t$ to $t+1$:  
%{\setlength{\belowdisplayskip}{-5pt}
\begin{align}
	        F_{\sigma}(x)&:=\frac{J_{\sigma}(x)}{\sqrt{J_{\sigma}^2(x)+K_{\sigma}^2(x)}} ,\quad \mbox{where} \label{eqn:Fsigma}\\ %,  
%	\frac{1-2\rmQ \big(\frac{x}{\sigma}\big)+\frac{2\sigma x}{\sqrt{2\pi}}\exp(-\frac{x^2}{2\sigma^2})}{\sqrt{\big(1-2\rmQ \big(\frac{x}{\sigma}\big)+\frac{2\sigma x}{\sqrt{2\pi}}\exp(-\frac{x^2}{2\sigma^2}) \big)^2+\frac{2\sigma^2(1-x^2)}{\pi}\exp(-\frac{x^2}{\sigma^2})}}. %, \quad \forall x\in[-1,+1]. 
	         ~~ J_{\sigma}(x)&:=1-2\rmQ \bigg(\frac{x}{\sigma}\bigg)+\frac{2\sigma x}{\sqrt{2\pi}}\exp\bigg(-\frac{x^2}{2\sigma^2}\bigg), \quad\mbox{and} \label{Eq:J func} \\
	K_{\sigma}(x)&:=\frac{2\sigma \sqrt{1-x^2}}{\sqrt{2\pi}}\exp\bigg(-\frac{x^2}{2\sigma^2}\bigg). \label{Eq:K func}
\end{align}
The $t^{\mathrm{th}}$ iterate of the function  $F_{\sigma}$ is defined recursively as $F_{\sigma}^{(t)}:=F_{\sigma}\circ F_{\sigma}^{(t-1)}$
with $F_{\sigma}^{(0)}(x)=x$.
As shown in Figure~\ref{Fig: Fsig}, for any fixed $\sigma$, we can see that  $F_{\sigma}^{(2)}(x)\geq F_{\sigma}(x)\geq x$ for $x\geq 0$ and $F_{\sigma}^{(2)}(x)<F_{\sigma}(x)<x$ for $x< 0$. It can also be easily deduced that for any $t\in[0:\tau]$, $F_{\sigma}^{(t+1)}(x) \geq F_{\sigma}^{(t)}(x)$ for any $x\geq 0$ and $F_{\sigma}^{(t+1)}(x)<F_{\sigma}^{(t)}(x)$ for any $x< 0$. This important observation implies  that if the correlation $\alpha$, defined in~\eqref{Eq:rho_0}, is positive,  $F_{\sigma}^{(t)}(\alpha)$ increases with $t$; and vice versa. Moreover, as shown in Figure \ref{Fig: Fsig_sig} in Appendix~\ref{pf of Thm:exact gen GMM}, by varying $\sigma$, we observe that a smaller $\sigma$ results in a  larger $|F_{\sigma}(x)|$.
\begin{figure}
	\vspace{-5pt}
	\centering
	\includegraphics[width=0.4\linewidth]{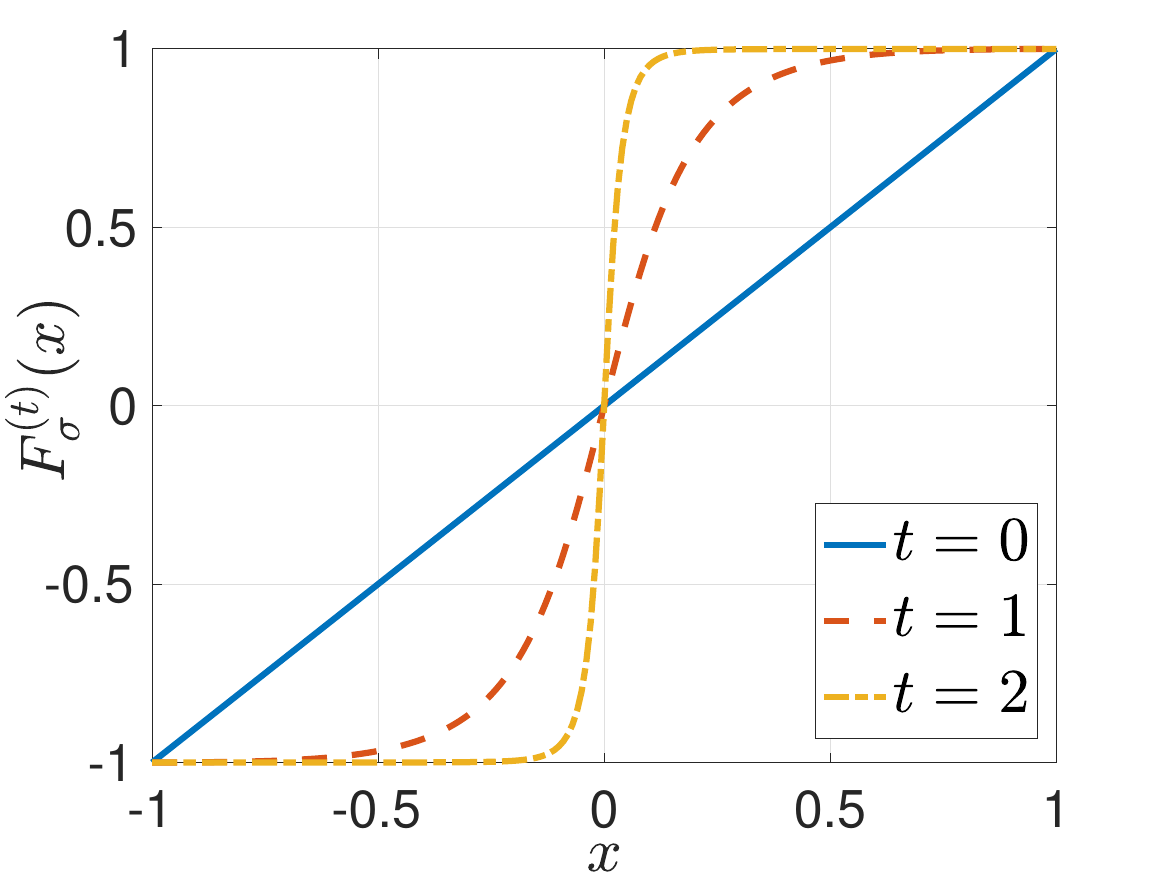}
	\vspace{-7pt}
	\caption{$F_{\sigma}^{(t)}(x)$ versus  $x$ for different $t$ when $\sigma=0.5$.}
	\label{Fig: Fsig}
%	\vspace{-15pt}
\end{figure}

\subsection{Main Theorem}\label{SubSec:main thm}

By applying the result in Theorem \ref{Thm:exact gen}, the following theorem provides an exact characterization for the generalization error at each iteration $t$ for $m$ large enough. 
%Let $\delta_{m,c-\tilc_1,d}:=2d\exp(-\frac{m(\veps-\tilc_2)^2}{2\sigma^2} )$. 
\begin{theorem}[Exact gen-error for iterative SSL under bGMM]\label{Thm:exact gen GMM}
Fix any $\sigma\in\bbR_+$, $d\in\bbN$. 
The gen-error at $t=0$ is given by
%\vspace{-7pt}
\begin{align}
	&\mathrm{gen}_0(P_\bZ, P_\bX, P_{\btheta_0|S_{\rml},S_{\rmu}} )=\frac{d}{n}. \label{Eq:exact gen_0}
\end{align}
Let $\alpha=\alpha(\xi_0,\bmu^{\bot})$. For each $t\in[1:\tau]$, for almost all sample paths (i.e., almost surely), % there exists a vanishing sequence $\epsilon_{m,t}' = o(1)$ ($\epsilon_{m,t}'\to 0 $ as $m\to\infty$), such that
\begin{align}
	&\mathrm{gen}_t(P_{\bZ}, P_{\bX}, \{P_{\btheta_k|S_{\rml},S_{\rmu}}\}_{k=0}^t, \{f_{\btheta_k}\}_{k=0}^{t-1} ) \nn\\
	&=\bbE_{\xi_0,\bmu^{\bot}} \bigg[\frac{(m-1)(J_{\sigma}^2(F_{\sigma}^{(t-1)}(\alpha)) +  K_{\sigma}^2(F_{\sigma}^{(t-1)}(\alpha)))}{m\sigma^2}  -\frac{J_{\sigma}(F_{\sigma}^{(t-1)}(\alpha))}{\sigma^2}\bigg]+o(1), \label{Eq:exact gent bGMM}
\end{align}
where $o(1)$ is a term that vanishes as $m\to\infty$.
%where $\epsilon_{m,1}'\equiv 0$ and $\epsilon_{m,t}'=\epsilon_{m,t}+\frac{d\sigma^2+1}{m\sigma^2}$ for $t\geq 2$, and .
%where $\alpha(\xi_0,\bmu^{\bot})=\rho(\btheta_0,\bmu)$ in \eqref{Eq:rho_0}, $\xi_0\sim\calN(0,1)$, $\bmu^{\bot}\sim\calN(0,\bI_d-\bmu\bmu^\top)$ is a random vector perpendicular to $\bmu$ and independent of $\xi_0$.
\end{theorem}
The proof of Theorem \ref{Thm:exact gen GMM} is provided in Appendix \ref{pf of Thm:exact gen GMM}. Several remarks are in order.

First, the gen-error at $t=0$ corresponds to the asymptotic result of supervised maximum likelihood estimation in works by \citet{akaho2000nonmonotonic} and \citet{aminian2021exact}.
We numerically plot the quantity  in \eqref{Eq:exact gent bGMM}, $g_{\sigma}^{(m)}(x):=((m-1)(J_{\sigma}^2(x)+K_{\sigma}^2(x))-mJ_{\sigma}(x))/\sigma^2$, for $x\in[-1,1]$ in Figure \ref{Fig: g_sig_m} in Appendix \ref{pf of Thm:exact gen GMM}, which shows that for all $\sigma_1>\sigma_2$, $g_{\sigma_1}^{(m)}(x)>g_{\sigma_2}^{(m)}(x)$ when $x>0$. From \eqref{Eq:rho_0}, we can see that $\alpha$ is close to 1 of high probability, which means that $\sigma\mapsto g_{\sigma}(x)$ is monotonically increasing in $\sigma$ with high probability. As a result, \eqref{Eq:exact gent bGMM} increases as $\sigma$ increases. This is consistent with the intuition that when the training data have larger overlap between classes, it is more difficult to generalize well. Moreover, $F_{\sigma}^{(t)}(\alpha)$ is also close to 1 of high probability, and thus \eqref{Eq:exact gent bGMM} saturates with $t$ quickly.

Second, by ignoring the $o(1)$ term, %letting $\epsilon_{m,t}'\to 0$, 
we compare the theoretical $\mathrm{gen}_t$ (cf.~\eqref{Eq:exact gen_0} and \eqref{Eq:exact gent bGMM}) and the empirical gen-error from the repeated synthetic experiments with $d=2$, $n=10$ and $m=1000$, as shown in Figure \ref{Fig:exact gen diff sig}. It can be seen that the theoretical  $\mathrm{gen}_t$   matches the empirical gen-error well, which means that the characterization in \eqref{Eq:exact gent bGMM} serves as a useful rule-of-thumb for how the gen-error changes over the SSL iterations.
When the variance is small (e.g., $\sigma^2=0.6^2$), as shown in Figure~\ref{Fig:exactgen_d=2_sig06}, the gen-error decreases significantly from $t=0$ to $t=1$ and then quickly converges to a non-zero constant.  Recall the correlation evolution function $F_{\sigma}$ in \eqref{eqn:Fsigma}. Given any pair of $(\xi_0,\bmu^{\bot})$, if $\alpha(\xi_0,\bmu^{\bot})>0$, $F_\sigma^{(t)}(\alpha(\xi_0,\bmu^{\bot}))>F_\sigma^{(t-1)}(\alpha(\xi_0,\bmu^{\bot}))$ for all $t\in[1:\tau]$, as shown in Figure~\ref{Fig: Fsig}. This means that if the quality of the labelled data $S_\rml$ is reasonably good, by using $\btheta_0$ which is learned from $S_\rml$, the generated pseudo-labels for the unlabelled data are largely correct. Then the subsequent parameters $\btheta_t$ for $t\ge1$ learned from the large number of pseudo-labelled data examples can improve the generalization error. With sufficiently large amount of training data,   algorithm   converges at very early stage.  In addition, for more general cases (e.g.,  non-diagonal class covariance matrices), it takes more iterations for the gen-error to reach a plateau, as shown in Figure~\ref{Fig:emp_gen_cov}.

When the variance is large (e.g. $\sigma^2=3^2$), as shown in Figure~\ref{Fig:exactgen_d=2_sig3},  the gen-error increases with iteration $t$. The result shows that when the overlap between different classes is large enough, using the unlabelled data may not be able to improve the generalization performance. The intuition is that at the initial iteration with a limited number of labelled data, the learned parameter $\btheta_0$ cannot pseudo-label the unlabelled data with sufficiently high accuracy. Thus, the unlabelled data is not labelled well by the pseudo-labelling operation and hence, cannot help to improve the generalization error. To gain more insight, in Figure \ref{Fig:exactgen_vs_sig}, we numerically plot $\mathrm{gen}_1$ versus different values of $\sigma$ under the same setting. It is interesting to find that there exists a $\sigma_0$ such that for $\sigma<\sigma_0$, $\mathrm{gen}_1<\mathrm{gen}_0$, which means the gen-error can be reduced with the help of abundant unlabelled data, while for $\sigma>\sigma_0$, using the unlabelled data can even harm the generalization performance.

\begin{figure}[!t]
	\centering
	\subfigure[$\sigma=0.6$]{
	\includegraphics[width=0.4\linewidth]{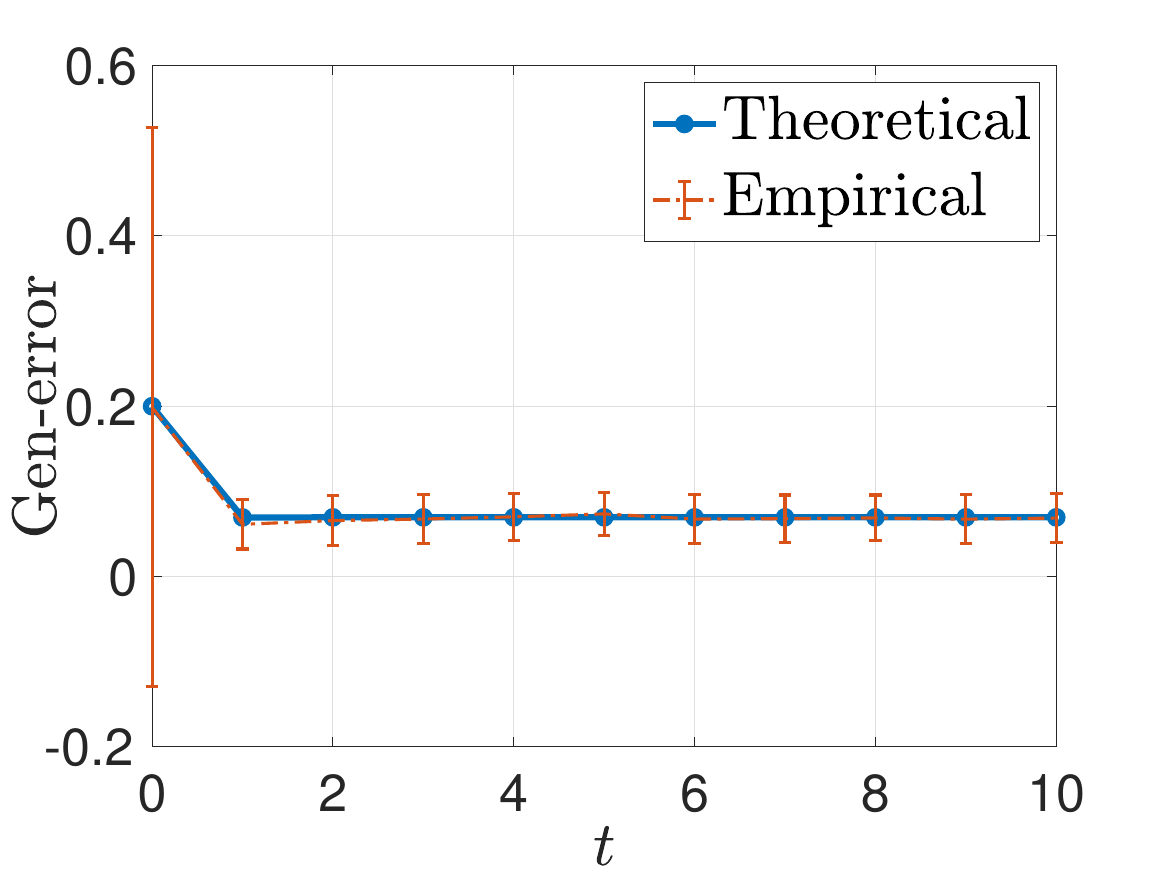}
	\label{Fig:exactgen_d=2_sig06}
	}
	\hspace{0pt}
	\subfigure[$\sigma=3$]{
		\includegraphics[width=0.4\linewidth]{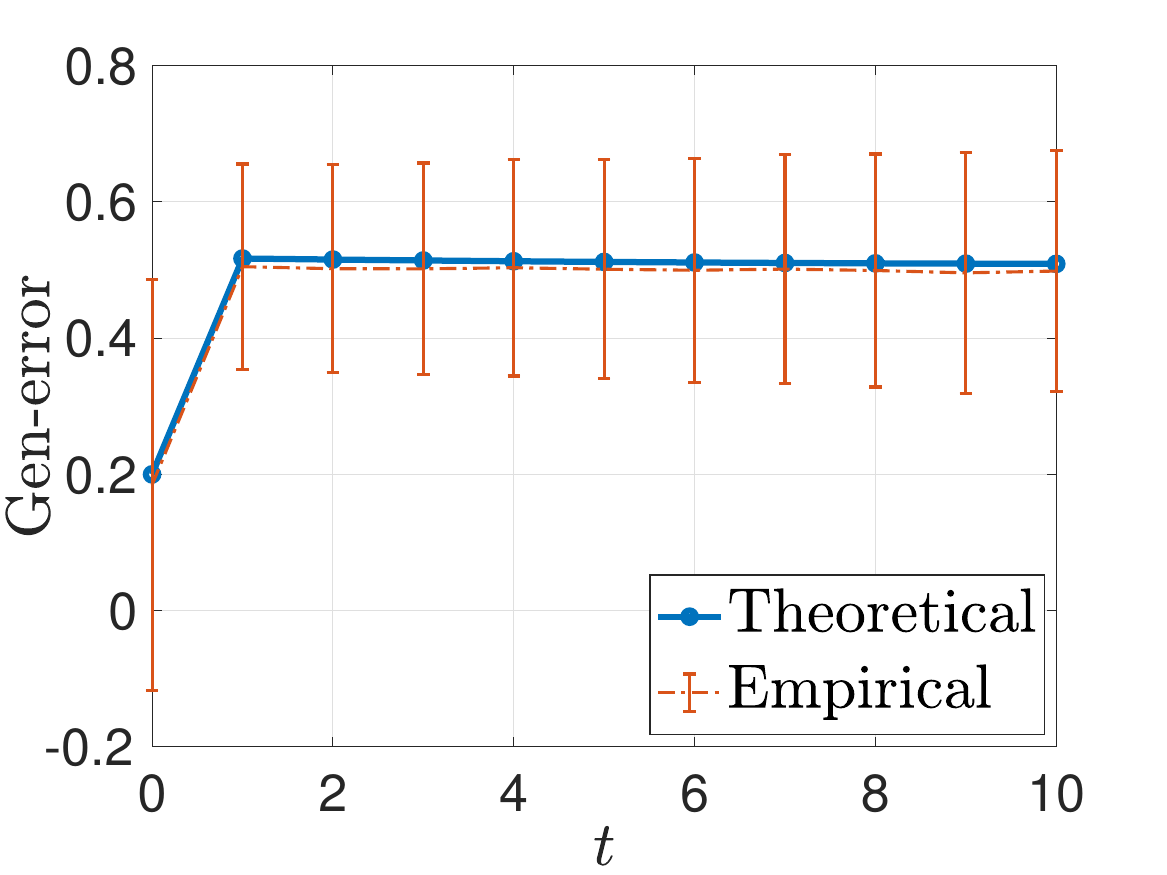}
		\label{Fig:exactgen_d=2_sig3}
	}
%	\vspace{-8pt}
	\caption{Comparison of the theoretical $\mathrm{gen}_t$ and the empirical gen-error at each iteration~$t$.}
	\label{Fig:exact gen diff sig}
\end{figure}
\begin{figure}[!t]
	\centering
	\begin{minipage}{0.4\linewidth}
		\centering
		\includegraphics[width=1\linewidth]{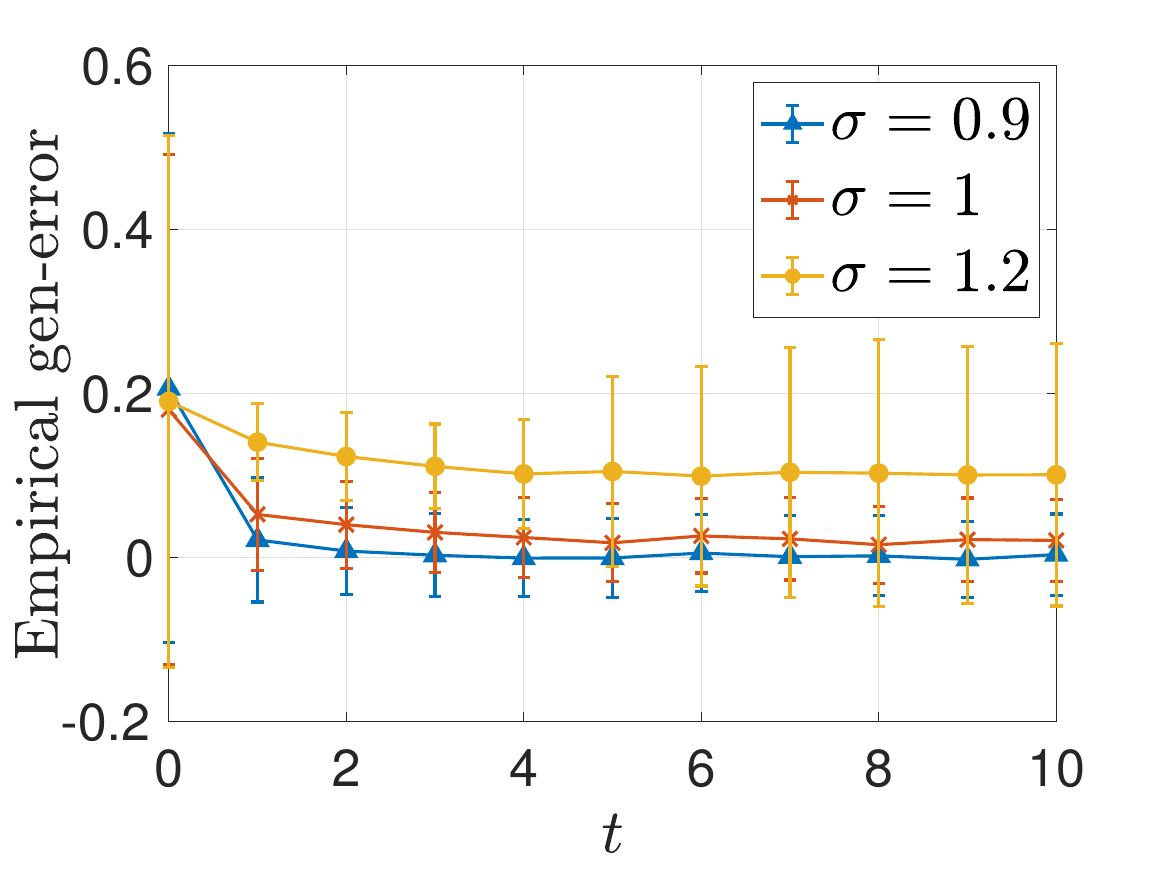}
		\vspace{-20pt}
 		\caption{Empirical gen-error with covariance matrix $\sigma^2\times[0.6, 0.3; 0.3, 0.8]$.}
		\label{Fig:emp_gen_cov}
	\end{minipage}
	\hspace{5pt}
	\begin{minipage}{0.4\linewidth}
	\centering
	\includegraphics[width=1\linewidth]{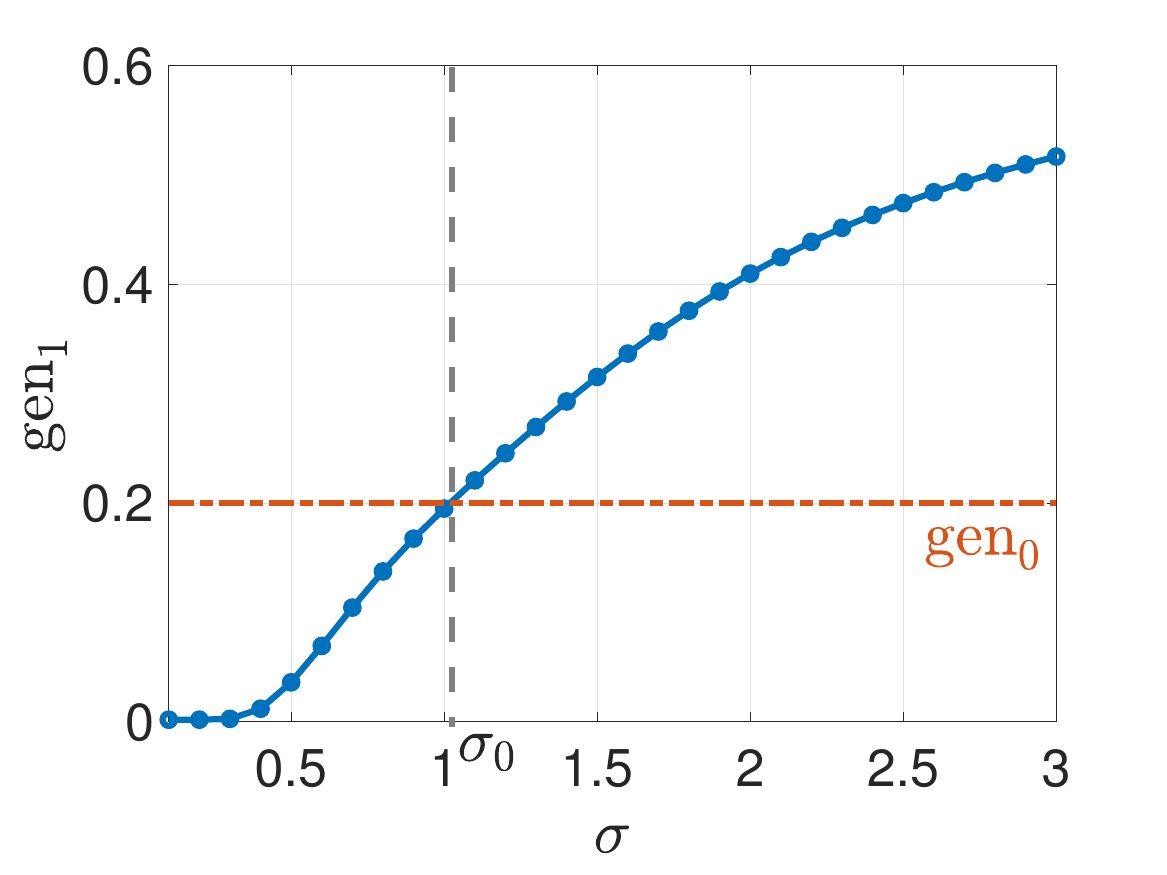}
	\vspace{-20pt}
	\caption{ Theoretical gen-error at $t\!=\!1$ versus  different std.\ devs.\ $\sigma\in[0.1,3]$. %in comparison with $\mathrm{gen}_0$.
	}
	\label{Fig:exactgen_vs_sig}
	\end{minipage}
	\vspace{-15pt}
\end{figure}

Third, let us examine the effect of $n$, the number of labelled training samples. By expanding~$\alpha$, defined in~\eqref{Eq:rho_0}, using a Taylor series, we have
\begin{align}
	\alpha%=\frac{1}{\sqrt{1+\frac{\frac{\sigma^2}{n}\|\bmu^{\bot}\|_2^2}{(1+\frac{\sigma}{\sqrt{n}}\xi_0)^2}}}
	=1-\frac{\sigma^2}{2n}\|\bmu^{\bot}\|_2^2+o\bigg(\frac{1}{n}\bigg).  \label{Eq:alpha taylor}
\end{align}
It can be seen that as $n$ increases, $\alpha$ converges to $1$ in probability. Suppose the dimension $d=2$ and $\bmu=(1,0)$. Then $\bmu^{\bot}=[0,\mu_2^{\bot}]$ where $\mu_2^{\bot}\sim\calN(0,1)$. By letting $m\to\infty$, the gen-error $\mathrm{gen}_1$ (cf. \eqref{Eq:exact gent bGMM}) can be rewritten as
%\vspace{-0.6em}
\begin{align}
	  \mathrm{gen}_1
	%		\leq \sqrt{\frac{(c_2-c_1)^2}{2} }~\bbE_{\mu_2^{\bot}}\bigg[\sqrt{ G_{\sigma}\bigg(1-\frac{\sigma^2(\mu_2^{\bot})^2}{2n}+o\bigg(\frac{1}{n}\bigg) \bigg)} ~\bigg]\\
	%	&\approx\sqrt{\frac{(c_2-c_1)^2}{2}}~\int_{-\infty}^{\infty}\frac{1}{\sqrt{2\pi}}e^{-\frac{x^2}{2}} \sqrt{ G_{\sigma}\bigg(1-\frac{\sigma^2 x^2}{2n} \bigg)} ~\rmd x\\
	\approx \int_{-\sqrt{2}}^{\sqrt{2}}\sqrt{\frac{n}{\pi\sigma^2} }\, e^{-\frac{ny^2}{\sigma^2}} \, g_{\sigma}^\infty(1-y^2) ~\rmd y, 
\end{align}
where $g_{\sigma}^{(\infty)}(1-y^2):=(J_\sigma^2(1-y^2)+K_\sigma^2(1-y^2)-J_\sigma(1-y^2))/\sigma^2$ and thus, $\mathrm{gen}_1$ is a decreasing function of $n$. We further deduce that for any $t$, $\mathrm{gen}_t$ is decreasing in $n$.

Fourth, we consider an ``enhanced'' scenario in which the labelled data in $S_{\rml}$ are reused in each iteration. Set $w=\frac{n}{n+m}$ in~\eqref{Eq: def of empirical risk}. We can  extend Theorem \ref{Thm:exact gen GMM} to   Corollary \ref{Coro:reuse labeldata gen bound}  provided in Appendix \ref{Append:reuse S_l}.
It can be seen from Figure~\ref{Fig:gen_bd_with_labelled} that $\mathrm{gen}_t$ still decreases from $t=0$ to $1$ and saturates afterwards. We find that when $\sigma=0.6$, $n=10$, $m=1000$, the gen-error is almost the same as that one in Figure~\ref{Fig:exactgen_d=2_sig06}, which means that for large enough $\frac{m}{n}$, reusing the labelled data does not necessarily help to improve the generalization performance. Moreover, when $m=100$, $\mathrm{gen}_t$ is higher than that for $m=1000$, which coincides with the intuition that increasing the number of unlabelled data helps to reduce the generalization error.

Fifth, it is natural to wonder what the effect is when $m$, the number of unlabelled data examples, is held fixed and $n$, the number of labelled data examples, increases. In  Figure~\ref{Fig: exact gen0vsgen1}, we numerically plot $\mathrm{gen}_0=\frac{d}{n}$ in~\eqref{Eq:exact gen_0} and (the theoretical) $\mathrm{gen}_1$ in~\eqref{Eq:exact gent bGMM} for $n$ ranging from $2$ to $50$, $m=1000$, $\sigma=0.6$ and $d=2$. As $n$ increases, $\mathrm{gen}_0$ and $\mathrm{gen}_1$ both decrease, which is as expected. However, when $n$ is larger than a certain value ($30$ in this case),  we find that $\mathrm{gen}_0$ becomes smaller than $\mathrm{gen}_1$. This implies that with sufficiently many labelled training data, the generalization error based on the labelled training data is already sufficiently low, and incorporating the pseudo-labelled data in fact adversely affects the generalization error. Understanding this phenomenon precisely is an interesting avenue for future work.
\begin{figure}[t]
\centering
\includegraphics[width=0.4\linewidth]{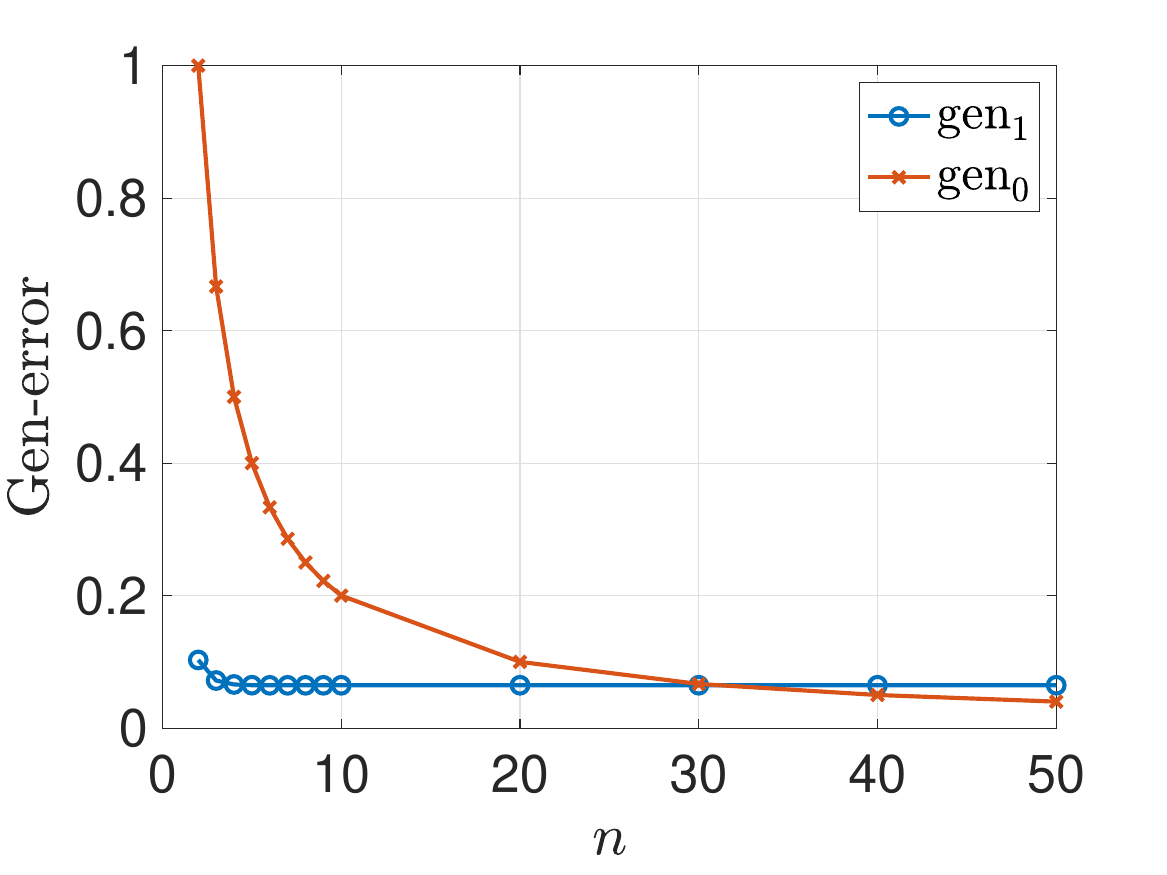}
		\caption{$\mathrm{gen}_0$ vs. $\mathrm{gen}_1$ for different $n$.}
		\label{Fig: exact gen0vsgen1}
\end{figure}

Finally, to verify the validity of the gen-error upper bound in Theorem~\ref{Coro: sub Gau gen}, we further apply the bound to this setup and prove that the upper bound exhibits   similar behaviour of the evolution of gen-error as $t$ increases. See Appendix \ref{pf of Thm:gen bound GMM}.

\begin{figure}[!t]
	\vspace{-3pt}
	\centering
	\begin{minipage}{0.4\linewidth}
		\centering
		\includegraphics[width=1\linewidth]{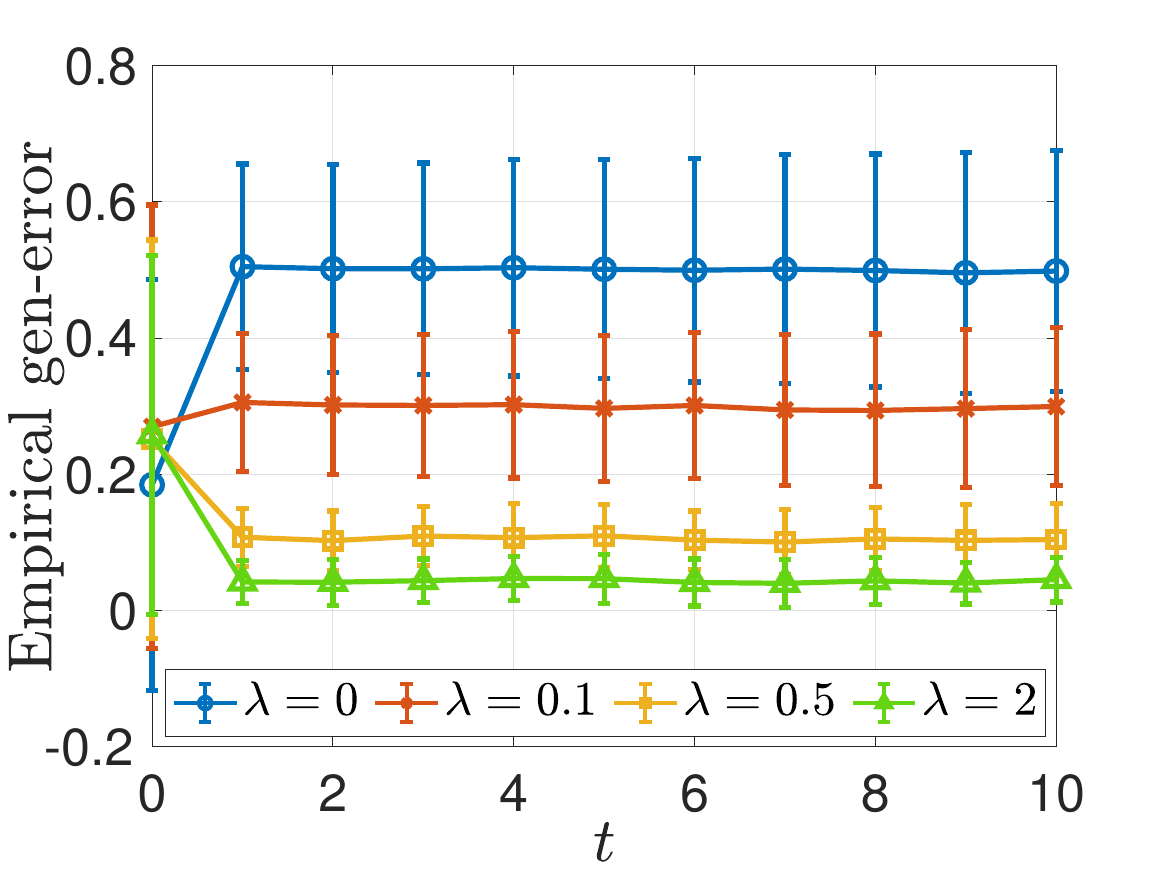}
		\caption{Empirical gen-error versus  $t$ for $\sigma=3$ for different $\lambda$.}
		\label{fig:empgenvstdifflamsig3}
	\end{minipage}
	\hspace{5pt}
	\begin{minipage}{0.4\linewidth}
		\centering
		\includegraphics[width=1\linewidth]{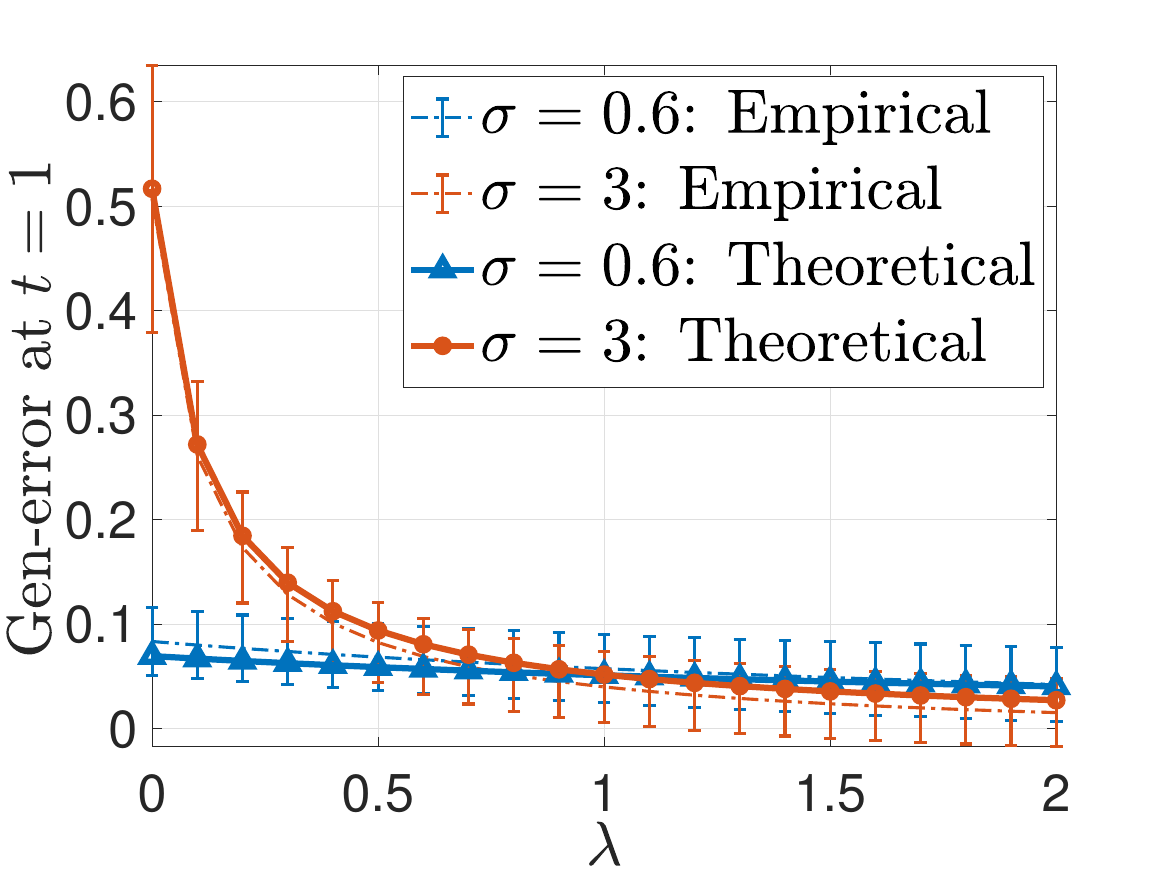}
		\caption{Theoretical and empirical $\mathrm{gen}_1$ vs.\  $\lambda$ for different~$\sigma$.}
		\label{fig:gen1_bd_and_emp_diffsig}
	\end{minipage}
%	\vspace{-2em}
\end{figure}

\section{Improving the Gen-Error for Difficult Problems via Regularization}\label{Sec:regularization}
In Section \ref{SubSec:main thm}, it is shown that for difficult classification problems with large class conditional variance, the gen-error increases after using pseudo-labelled data. The reason is that the learned initial parameter $\btheta_0$ can only generate low-accurate pseudo-labels and thus the pseudo-labelled data cannot help improve the generalization performance.  In this section, we prove that by adding regularization to the loss function,  we can mitigate the undesirable increase of gen-error across the pseudo-labelling iterations.

Since $\mathrm{gen}_0$ in \eqref{Eq:exact gen_0} does not depend on data variance $\sigma^2$, here we focus on subsequent iterations $t\in[1:\tau]$. By considering the $\ell_2$ regularization (i.e., adding $\frac{\lambda}{2}\|\btheta\|_2^2$ to \eqref{Eq:L_hatS}), we obtain the new parameters (cf.~\eqref{Eq:theta_t}) as follows
%{\setlength{\abovedisplayskip}{5pt}
%	\setlength{\belowdisplayskip}{5pt}
\begin{align}
	\btheta_t^{\mathrm{reg}}=\frac{\btheta_t}{1+\sigma^2\lambda}, \quad\forall\, t\in[1:\tau]. \label{Eq: def theta_t reg}
\end{align}%}
The  derivations are provided in Appendix \ref{pf of Thm:gen bound GMM reg}. By applying Theorem \ref{Thm:exact gen GMM}, the following theorem provides a characterization for the gen-error of the case with regularization at each iteration $t$ for $m$ large enough. Let $\mathrm{gen}_t^{\reg}$ denotes the gen-error of the case with regularization and we drop the fixed quantities $(P_\bZ, P_\bX, \{P_{\btheta_k|S_{\rml},S_{\rmu}}\}_{k=0}^t, \{f_{\btheta_k}\}_{k=0}^{t-1} )$ for notational simplicity.
\begin{theorem}[Gen-error with regularization]\label{Thm:gen bound GMM reg}
	Fix any  $d\in\bbN$, and $\sigma, \lambda \in\bbR_+$.
	The gen-error at any $t\in[1:\tau]$ is 
	{\setlength{\abovedisplayskip}{6pt}
		\setlength{\belowdisplayskip}{4pt}
	\begin{align}
		\mathrm{gen}_t^{\reg}=\frac{\mathrm{gen}_t}{1+\sigma^2\lambda}.
		 \label{Eq:gent reg}
	\end{align}}
\end{theorem}
The proof of Theorem~\ref{Thm:gen bound GMM reg} is provided in Appendix~\ref{pf of Thm:gen bound GMM reg}. From~\eqref{Eq:gent reg}, we   observe that as $\lambda$ increases, the gen-error decreases. 
In Figure~\ref{fig:empgenvstdifflamsig3}, we first empirically show that regularization can help mitigate the increase of gen-error during SSL iterations for hard-to-distinguish classes by comparing the empirical gen-error under $\lambda=0, 0.1, 0.5, 2$ when $\sigma=3$ and $d=2$. Then in Figure \ref{fig:gen1_bd_and_emp_diffsig}, we plot the theoretical gen-error in \eqref{Eq:gent reg} versus $\lambda$ when $t=1$ for the cases with small and large  variances, i.e., $\sigma^2=0.6^2$ and $\sigma^2=3^2$. We also  compare the theoretical results with the empirical gen-errors, which turn out to corroborate the theoretical ones. For the case with smaller variance, the improvement on gen-error is barely visible as $\lambda$ increases.  For the case with larger variance, the decrease of the gen-error is more pronounced, which implies that $\ell_2$-regularization can effectively mitigate the impact on the gen-error induced by large class overlapping and pseudo-labels with low accuracy.

The adept reader might naturally wonder why  one would not set $\lambda\to\infty$ in \eqref{Eq:gent reg}, which results in the gen-error tending to zero, which presumably is a desirable phenomenon. However, ultimately, what we wish to control is the expected population risk, which, according to \eqref{eqn:decomp}, is the sum of the expected empirical risk on the training data and the gen-error. Even if the gen-error is zero, the expected population risk might be large. Hence,   as~$\lambda$ increases, we see a tradeoff between the gen-error and the empirical risk.

\begin{table*}[!t]\small
	\centering
	\begin{minipage}{\linewidth}
		\centering
	\begin{tabular}{c|c|c|c}
		\toprule[2pt]
		Classes & RGB-mean $\ell_2$ distance & RGB-variance $\ell_2$ distance & Difficulty \\ \midrule[2pt]
		horse-ship & 0.0180  & 3.90e-05 & Easy    \\ \hline
		automobile-truck &  0.0038 & 7.06e-05  & Moderate \\ \hline
		cat-dog & 0.0007 & 4.95e-05 & Challenging  \\ \hline
	\end{tabular}
	\caption{The $\ell_2$ distances between the RGB-mean and RGB-variance of different pairs of classes from the CIFAR10 dataset.}
	\label{Table:RGB mean var}
	\vspace{-5pt}
	\end{minipage}
%\end{table*}
%\begin{table*}[!t]\small
%	\centering
	\begin{minipage}{\linewidth}
		\centering
	\begin{tabular}{c !{\vrule width 2pt} c|c|c|c|c|c}
		\toprule[1pt]
		Classes & \makecell[c]{horse\\$\to$ship}  & \makecell[c]{ship \\ $\to$horse} & \makecell[c]{automobile \\ $\to$truck}  & \makecell[c]{truck \\ $\to$automobile} & \makecell[c]{cat \\ $\to$dog} & \makecell[c]{dog \\ $\to$cat} \\ \hline
		Number & 17 & 3 & 61 & 64  & 93 & 137  \\ 
		\hline
		Difficulty & \multicolumn{2}{c|}{Easy} & \multicolumn{2}{c|}{Moderate}  & \multicolumn{2}{c}{Challenging} \\ 
		\bottomrule[1pt]
	\end{tabular}
	\caption{Number of images misclassified out of 1000 \citep{liu2018unsupervised}.}
	\label{Table:confusion matrix}
	\end{minipage}
\end{table*}

\section{Experiments on Benchmark Datasets}\label{Sec:experiments}

%In Sections \ref{Sec:preliminaries}, \ref{Sec:main results} and \ref{Sec:regularization}, we theoretically analyse the gen-error across the iterations for iterative SSL with pseudo-labelling and especially for the case of bGMM classification.

 To further illustrate that our theory is indeed behind the empirical behaviour of the iterative self-training with pseudo-labelling, in this section, we conduct experiments on real-world benchmark datasets, which demonstrates that our theoretical results on the bGMM example can also reflect the training dynamics on more realistic real-world tasks. The code  to reproduce all the experiments can be found at \url{https://github.com/HerianHe/GenErrorSSL_2022.git}.

Recall that in the bGMM example, a higher standard deviation $\sigma$ represents a higher in-class variance, larger class-overlap, and consequently higher difficulty in classification. By a whitening argument, this also holds for bGMMs with non-isotropic covariance matrices. In our experiments on real-world data, we use the difficulty level of classification to  mimic different in-class variances of bGMM.  We pick two easy-to-distinguish class pairs (``automobile" and ``truck", ``horse'' and ``ship'') from the CIFAR-10 dataset  \citep{krizhevsky2009learning}  as an analogy to bGMM with small in-class variance, and one difficult-to-distinguish class pair (``cat'' and ``dog'') from the same dataset as an analogy to bGMM with large in-class variance. %The detailed discussions are provided in the subsequent paragraphs. 
Furthermore, to extend the analogy to multi-class classification, we conduct experiments on the 10-class MNIST dataset to gain more intuition.

We train deep neural networks (DNNs) via an iterative self-learning strategy (under the same setting as Figure \ref{Fig:system}) to perform binary and multi-class classification. In the first iteration, we only use a few labelled data examples to initialize the DNN with  a sufficient number of epochs.
%train the model for a relatively large number of epochs so that the training loss will converge to a small value and the model is initialized well. 
In the subsequent iterations, we first sample a subset of unlabelled data and generate pseudo-labels for them via the model trained from the previous iteration. Then we update the model for a small number of epochs with both the labelled and pseudo-labelled data.

\textbf{Experimental settings:}
For binary classification, we collect pairs of classes of images, i.e., ``automobile" and ``truck", ``horse'' and ``ship'', and ``cat'' and ``dog'' from the CIFAR10 \citep{krizhevsky2009learning} dataset. In this dataset, each class has $5000$ images for training and $1000$ images for testing. For each selected pair of classes, we manually divide the $10000$ training images into two sets: 
%We use the whole set of images in the selected pair of classes and divide them into two sets, i.e., 
the labelled training set with $500$ images and the unlabelled training set with $9500$ images. We train a convolutional neural network, ResNet-10 \citep{he2016deep}, and use stochastic gradient descent (SGD) optimizer to minimize the cross-entropy loss. In PyTorch, the cross-entropy loss is defined as the negative logarithm of the output softmax probability corresponding to the true class, which is analogous to the NLL of the data under the parameters of the neural network.
%via the self-learning strategy to perform the binary classification. 
For the task with $\ell_2$-regularization, we train the neural network by setting different weight decay parameters (equivalent to $\lambda \times$ learning rate). In each pseudo-labelling iteration, we sample $2500$ unlabelled images. The complete training procedure lasts for $50$ self-training iterations.

We further validate our theoretical contributions on a multi-class classification problem in which  we train a ResNet-6 model with the cross-entropy loss to perform $10$-class handwritten digits classification on the MNIST \citep{lecun1998gradient} dataset. We sample $51000$ images from the training set, which contains $6000$ images for each of the ten classes. We divide them into two sets, i.e., a labelled training set with $1000$ images and an unlabelled set with $50000$ images. The optimizer and training iterations follow those in the aforementioned binary classification tasks without regularization.

\begin{figure*}[t]
	\vspace{-10pt}
	\centering
	\subfigure[``horse-ship'': gen-error]{
		\includegraphics[width=0.23\linewidth]{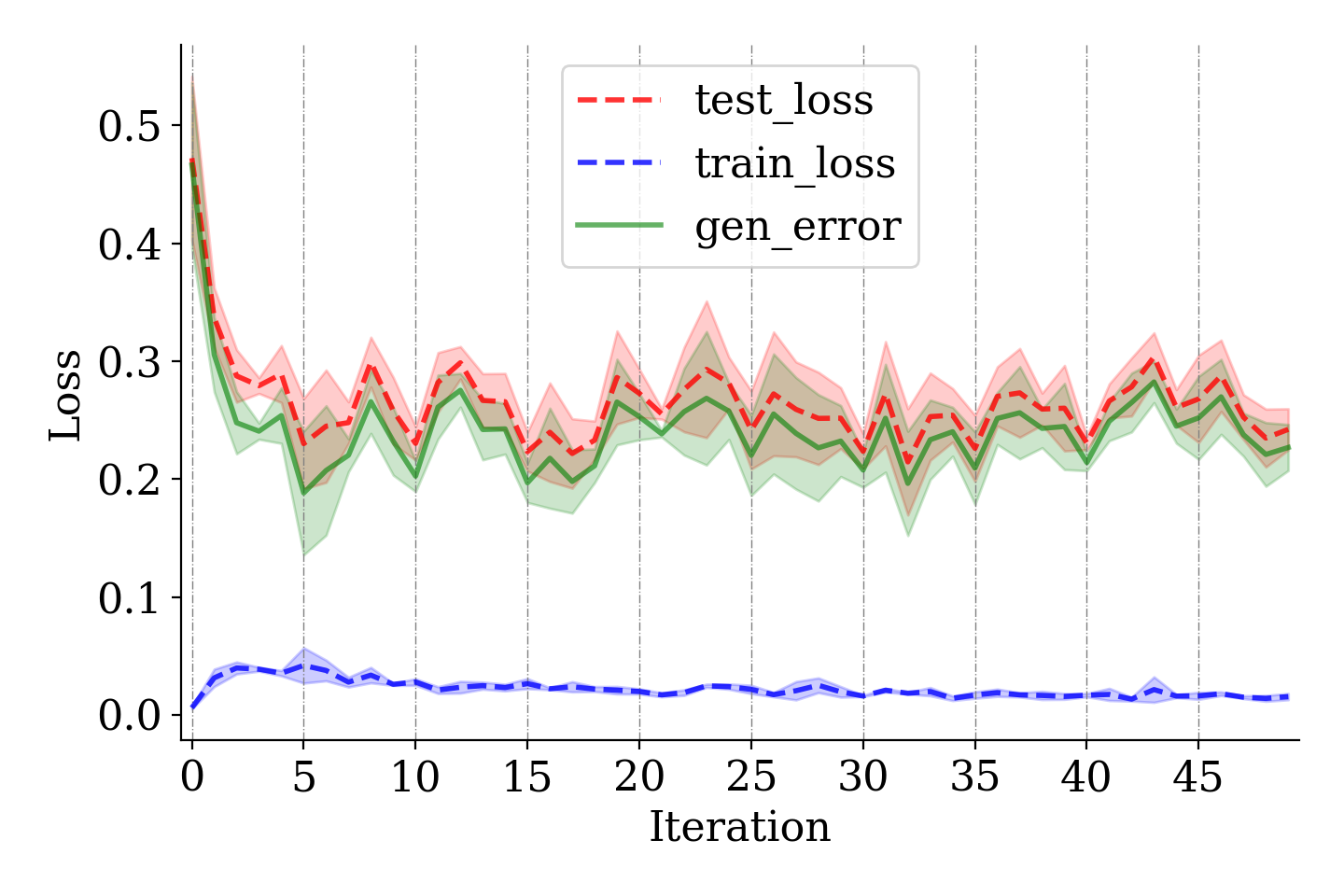}
		\label{Fig:bin_cifar10_horse_ship loss}}
	%	\hspace{-10pt}
	\subfigure[``automobile-truck'': gen-error]{
		\includegraphics[width=0.23\linewidth]{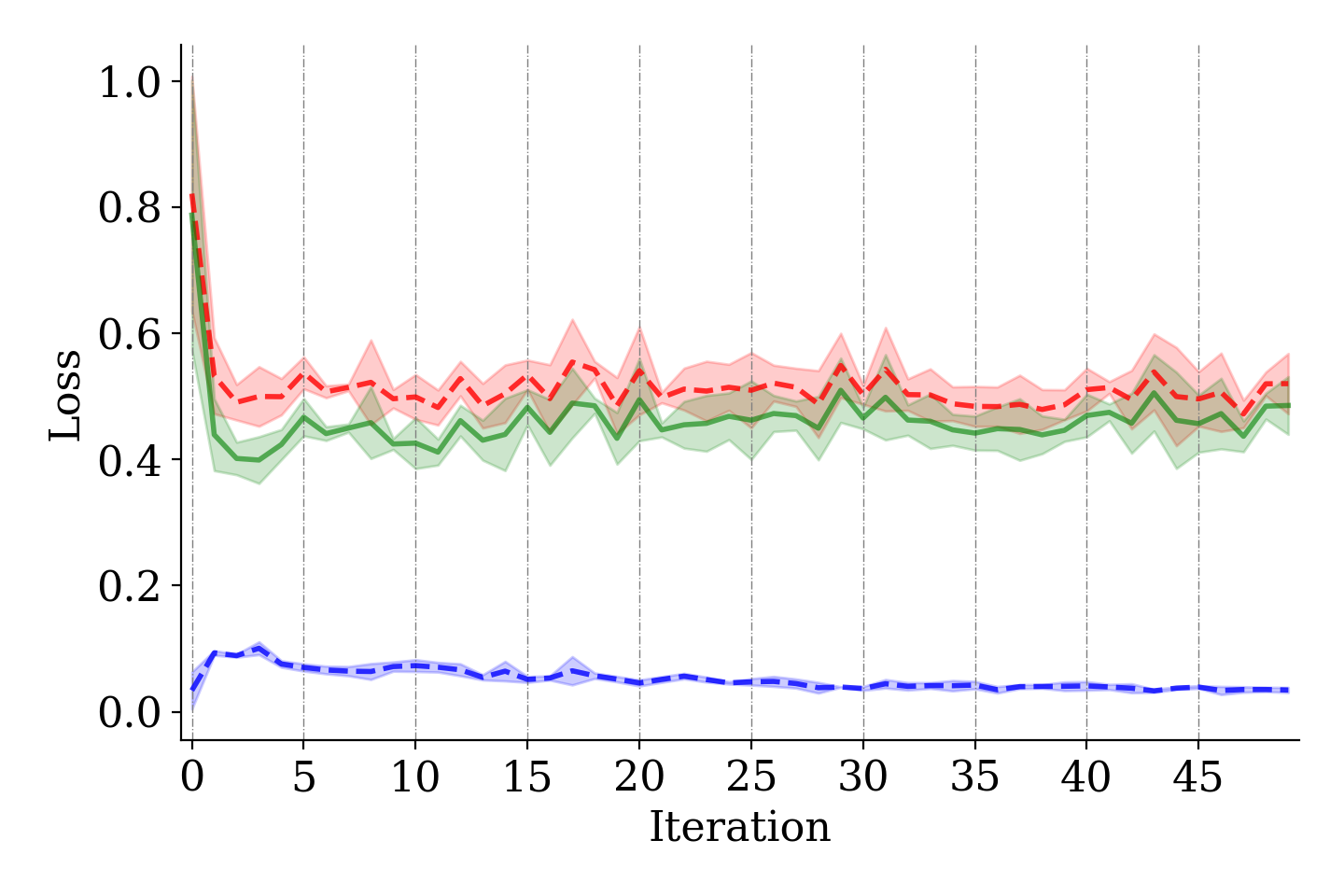}
		\label{Fig:bin_cifar10_car_truck loss}}
	\hspace{-10pt}
	\subfigure[MNIST: gen-error]{
		\includegraphics[width=0.23\linewidth]{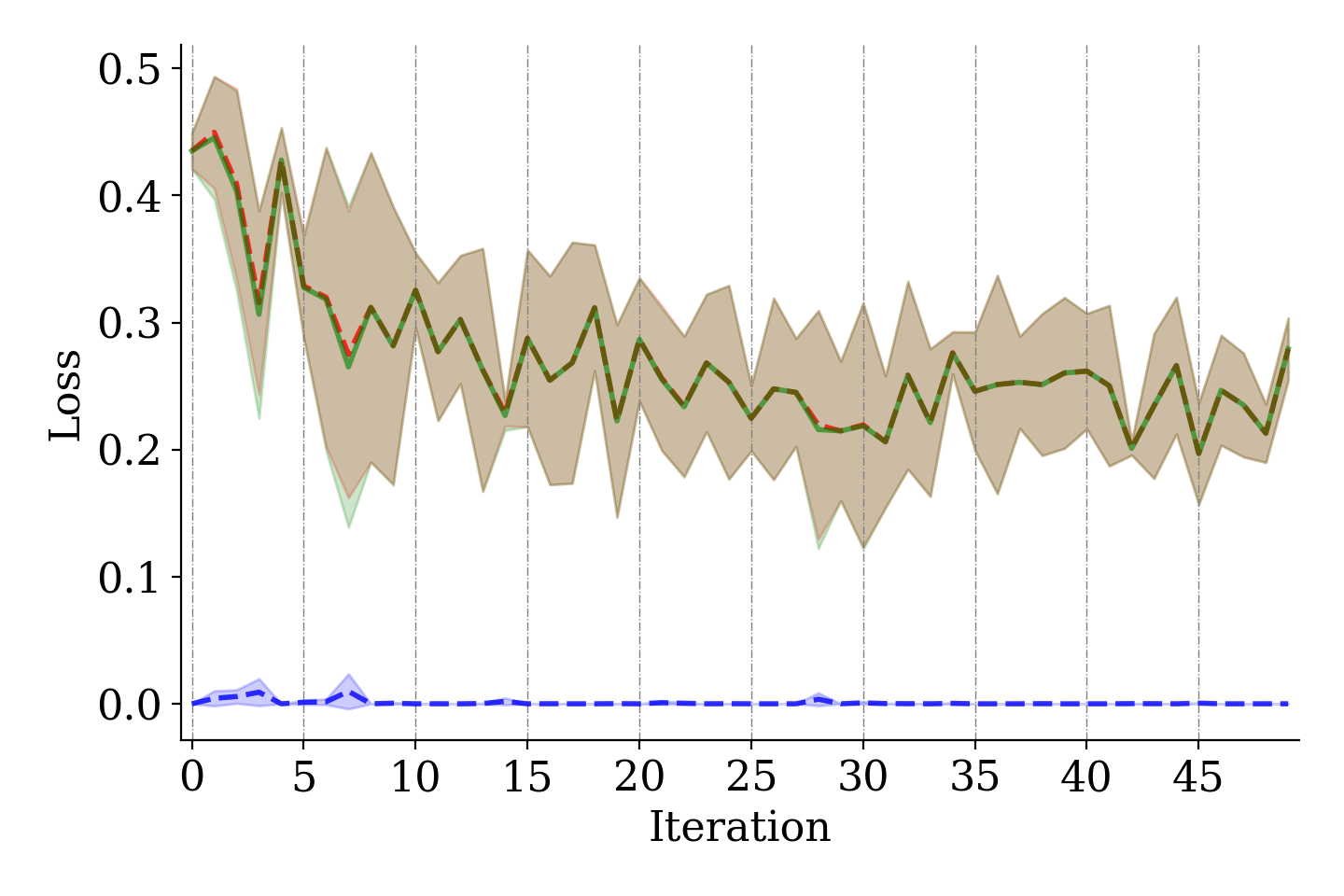}
		\label{Fig:mul_mnist loss}}
	\hspace{-5pt}
	\subfigure[``cat-dog'': gen-error]{
		\includegraphics[width=0.23\linewidth]{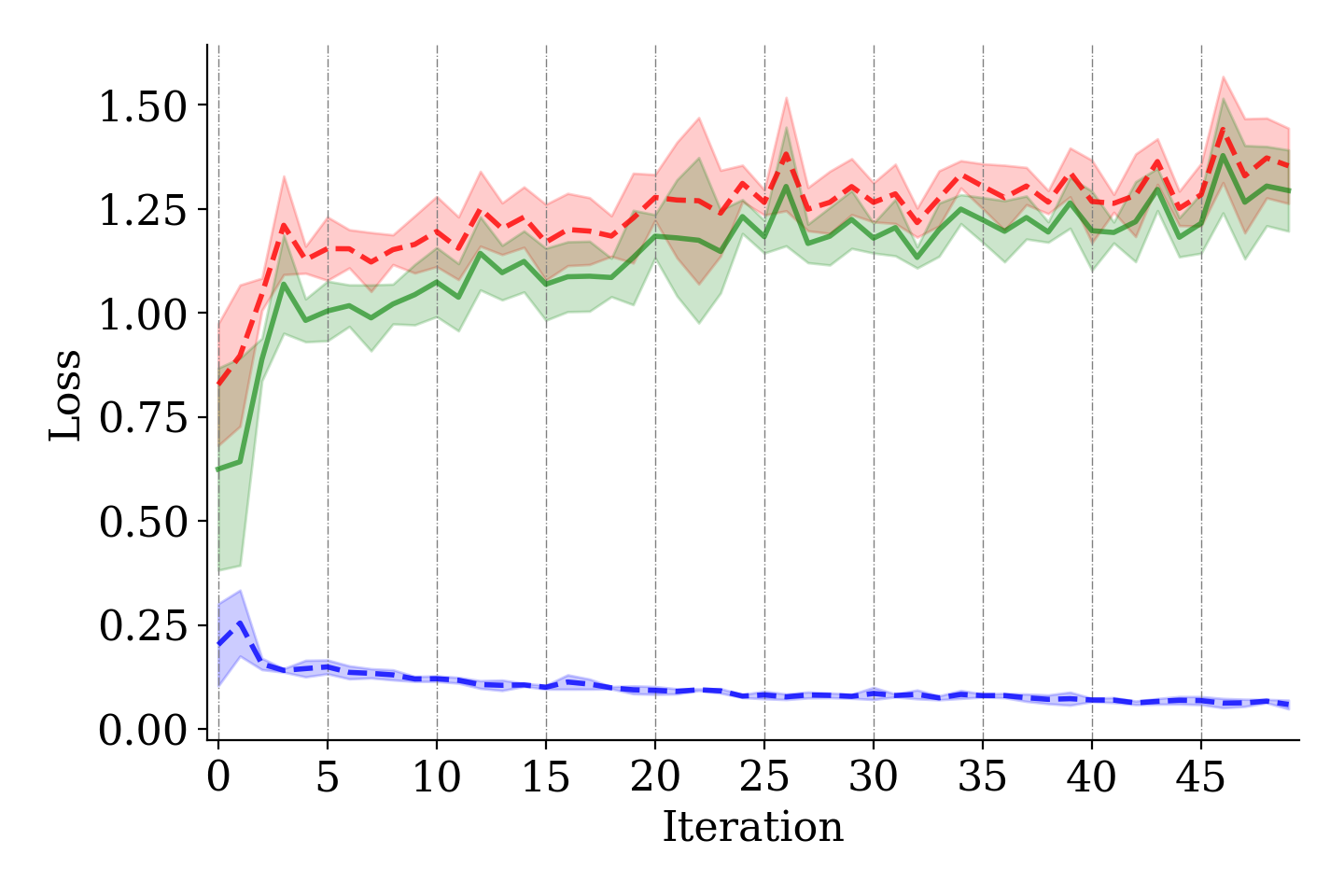}
		\label{Fig:bin_cifar10_cat_dog loss}}
	\hspace{-10pt}
	\raisebox{20pt}{\subfigure[``horse-ship'': accuracy]{
			\includegraphics[width=0.23\linewidth]{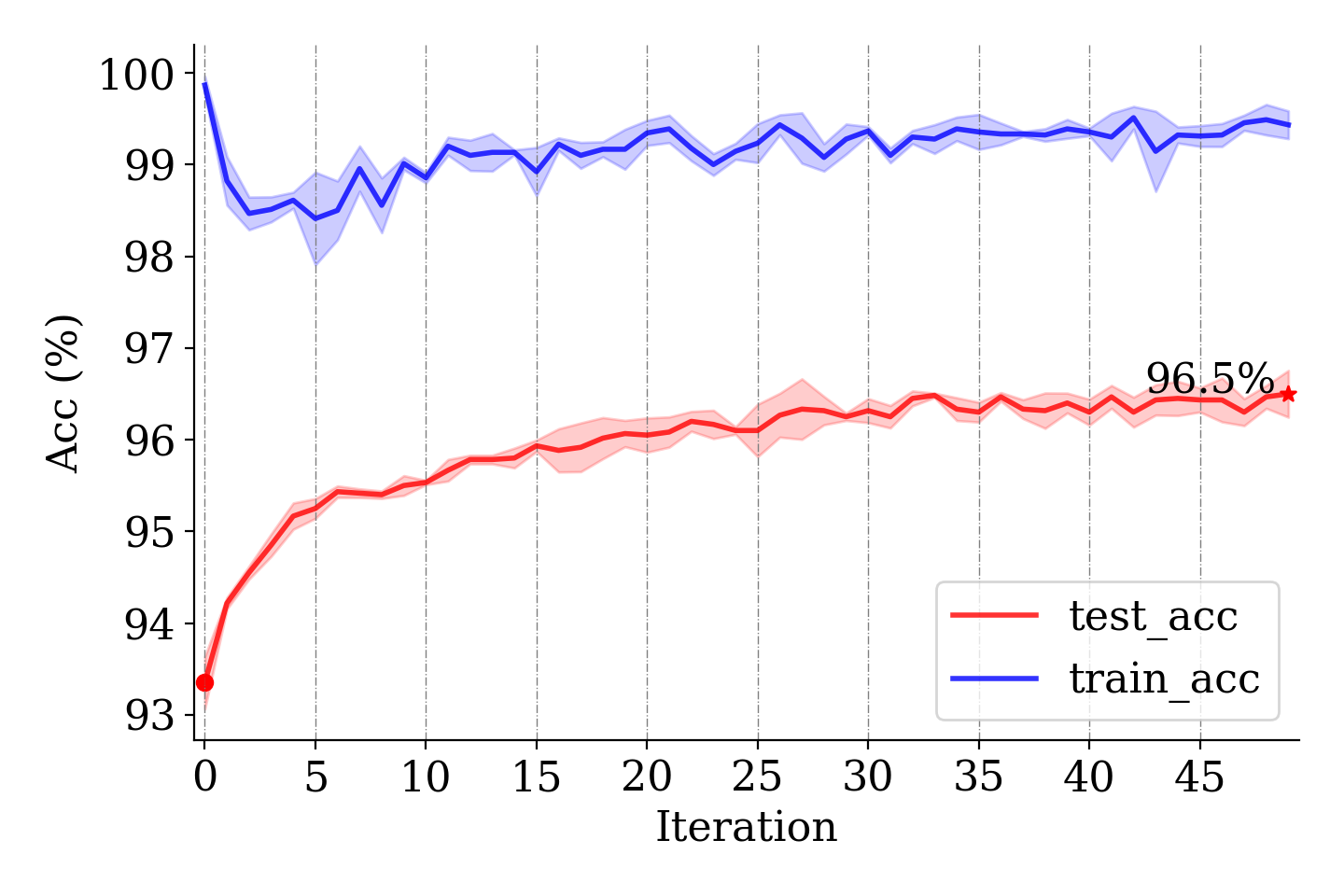}
			\label{Fig:bin_cifar10_horse_ship acc}}}
	\hspace{-0pt}
	\raisebox{20pt}{\subfigure[``automobile-truck'': accuracy]{
			\includegraphics[width=0.23\linewidth]{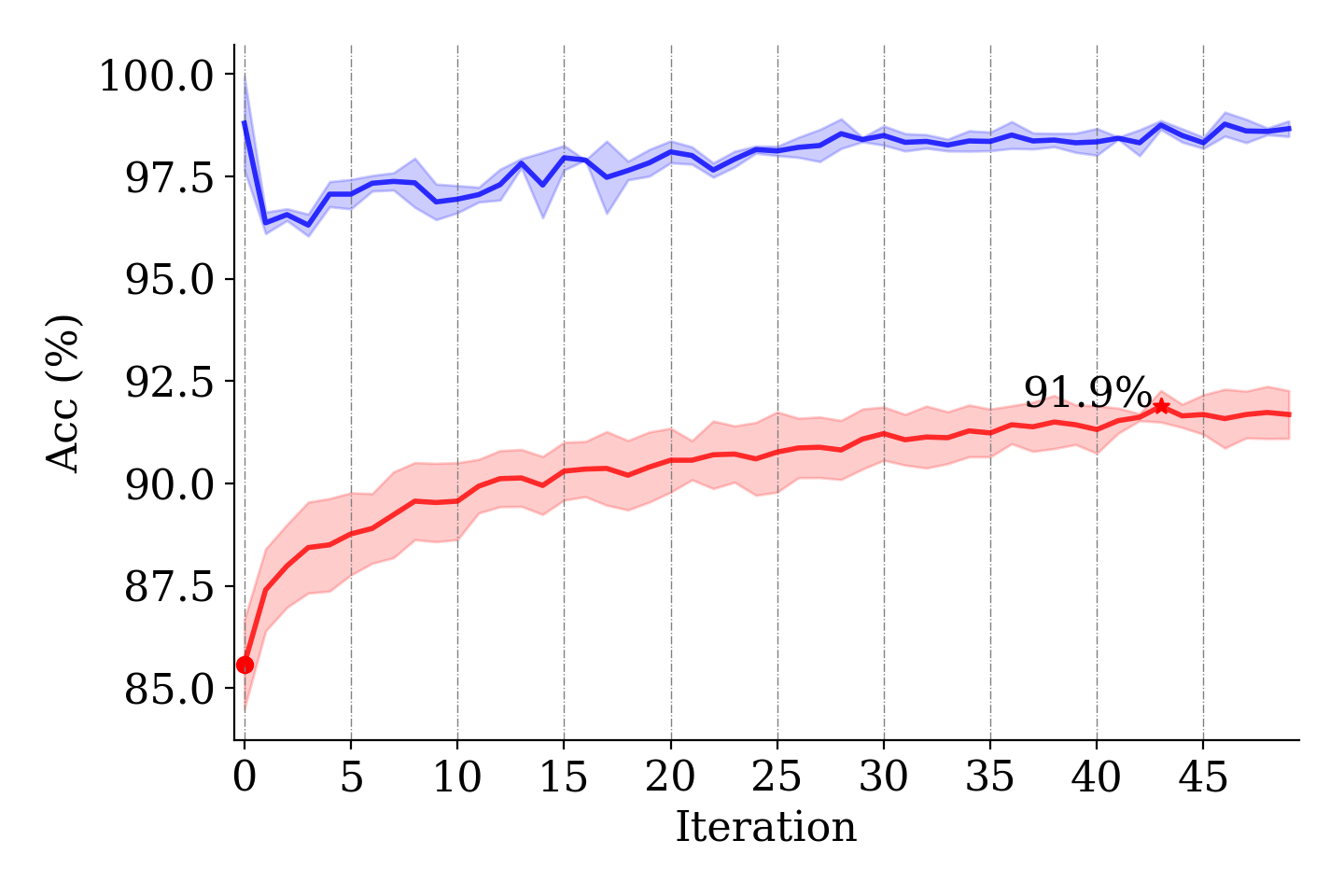}
			\label{Fig:bin_cifar10_car_truck acc}}}
	\hspace{-10pt}
	\raisebox{20pt}{\subfigure[MNIST: accuracy]{
			\includegraphics[width=0.23\linewidth]{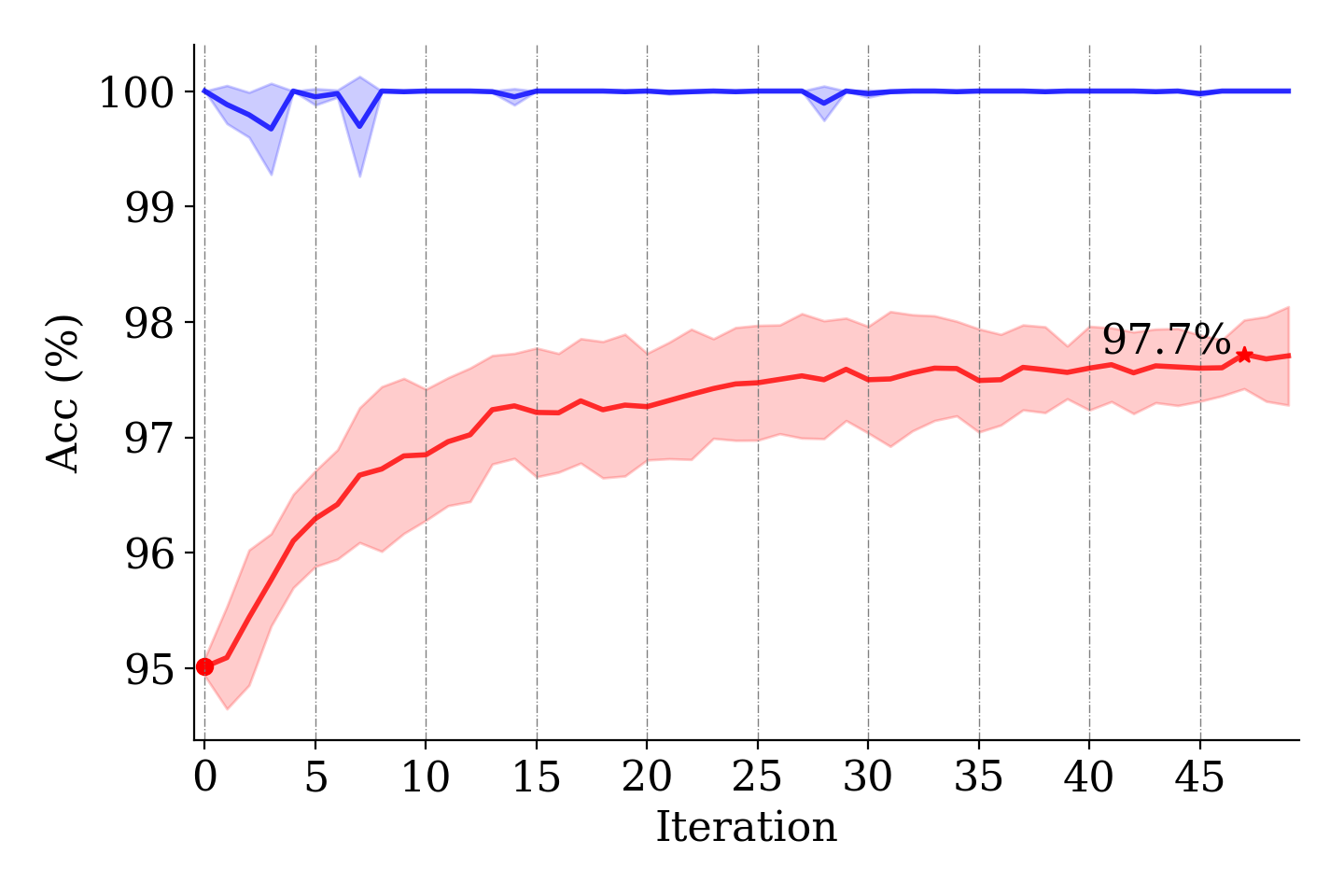}
			\label{Fig:mul_mnist acc}}}
	\hspace{-5pt}
	\raisebox{20pt}{\subfigure[``cat-dog'': accuracy]{
			\includegraphics[width=0.23\linewidth]{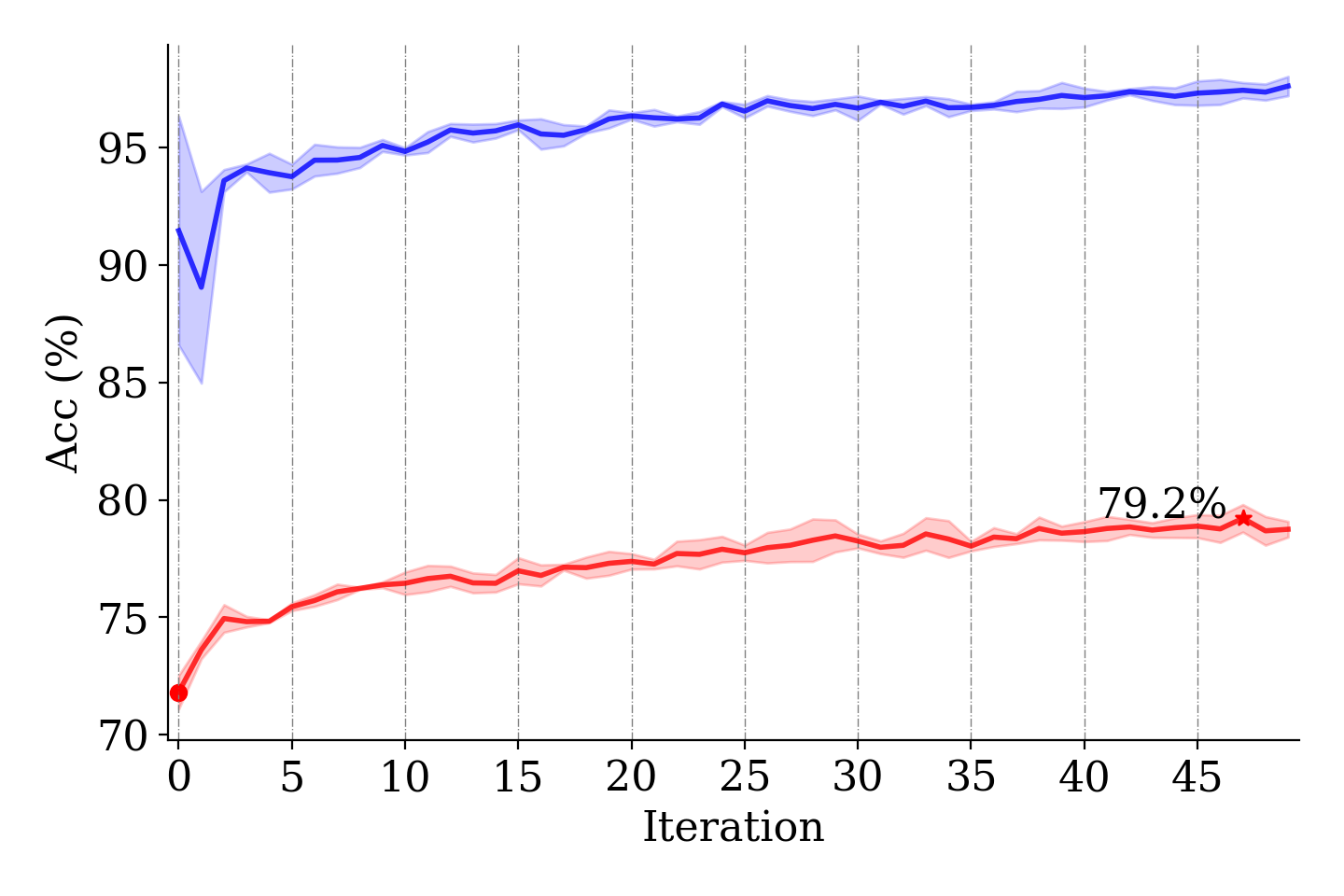}
			\label{Fig:bin_cifar10_cat_dog acc}}}
	\hspace{-10pt}
	\raisebox{20pt}{\subfigure[``cat-dog'': gen-error with weight decay set to  $5\times 10^{-4}$.]{
			\includegraphics[width=0.23\linewidth]{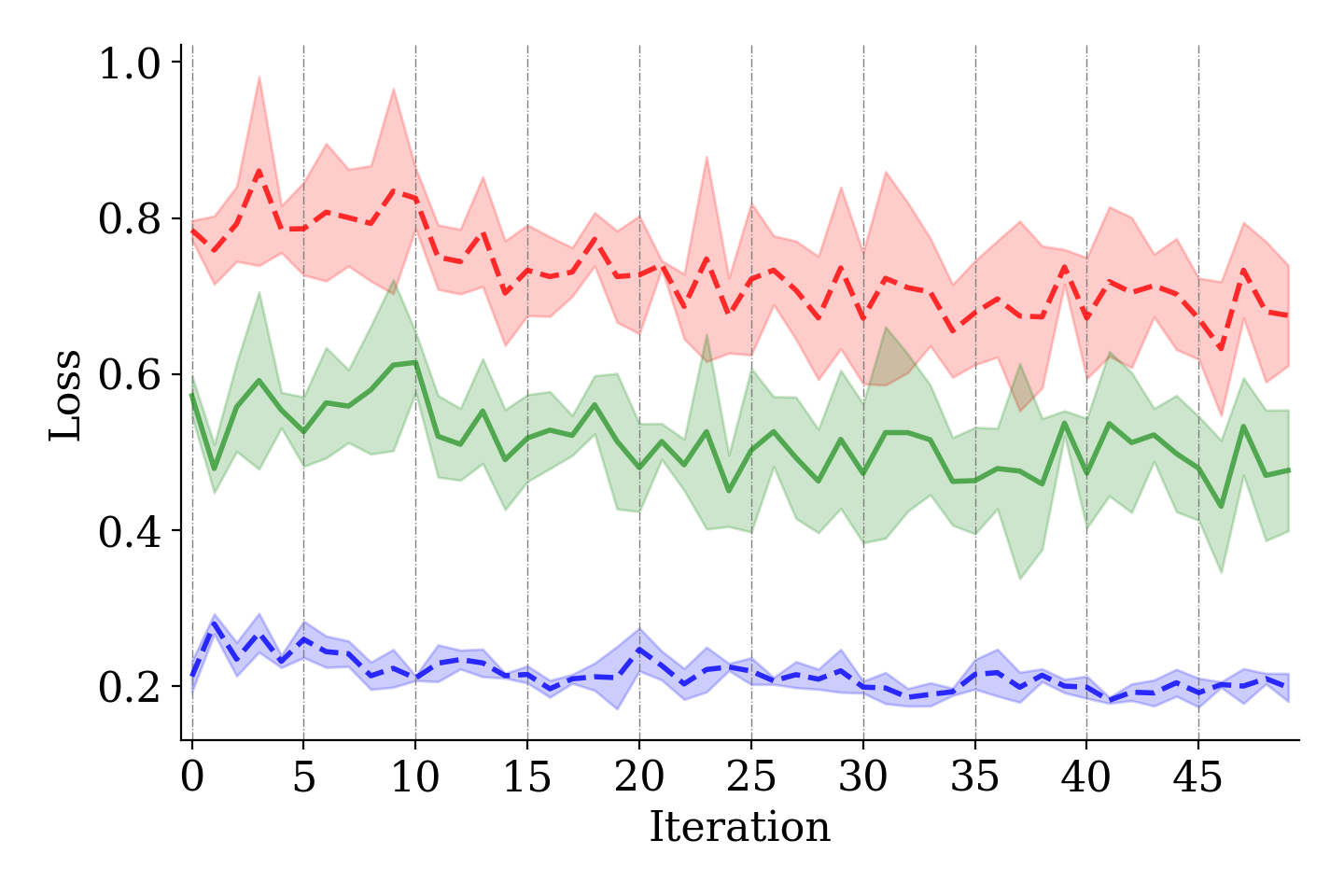}
			\label{Fig:bin_cifar10_dot_cat reg loss}}}
	\hspace{5pt}
	\raisebox{20pt}{\subfigure[``cat-dog'': accuracy with weight decay set to $5\times 10^{-4}$.]{
			\includegraphics[width=0.23\linewidth]{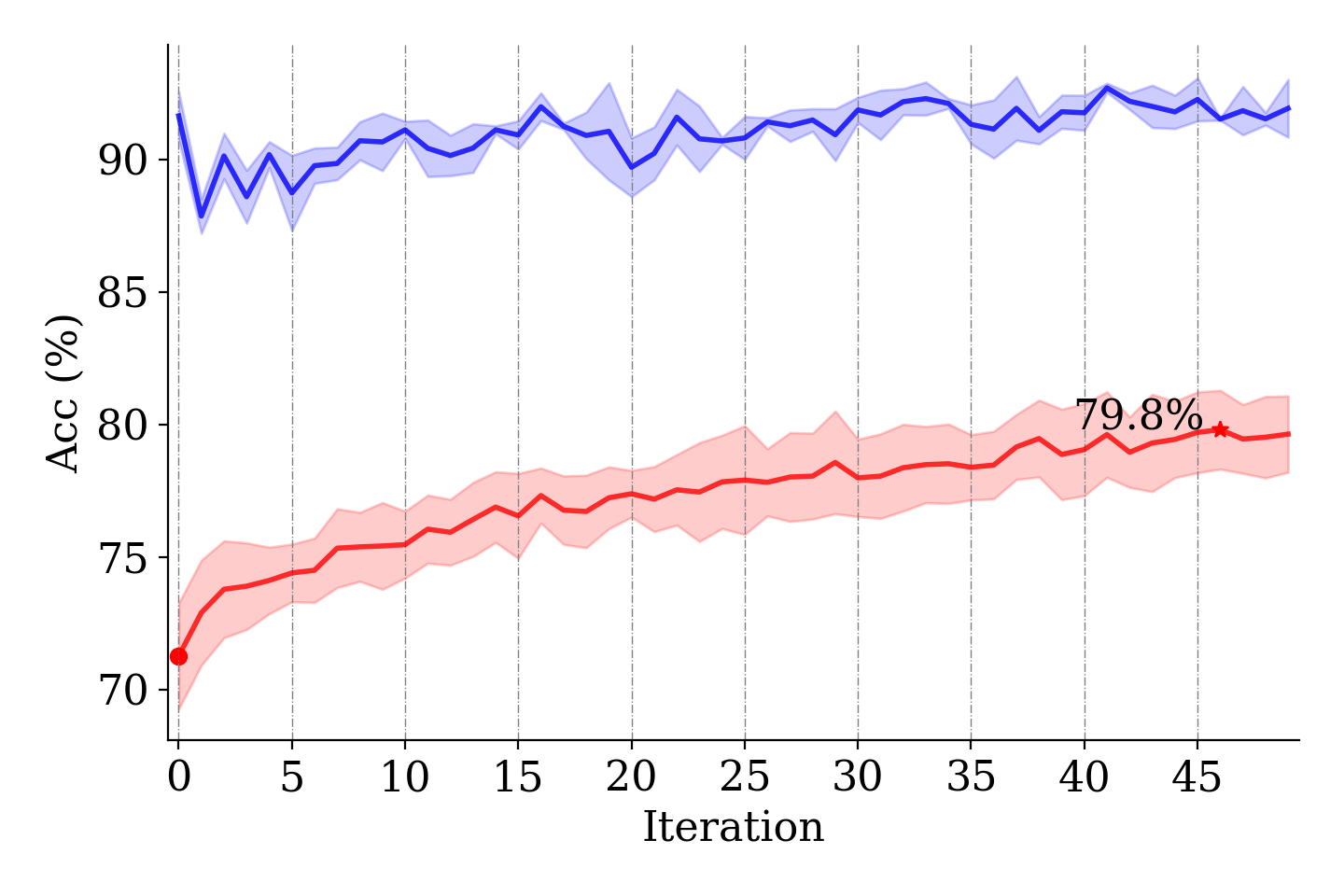}
			\label{Fig:bin_cifar10_dot_cat reg acc}}	}
	\hspace{10pt}
	\raisebox{20pt}{\subfigure[``cat-dog'': gen-error after convergence versus  weight decay.]{
			\includegraphics[height=0.16\linewidth, width=0.23\linewidth]{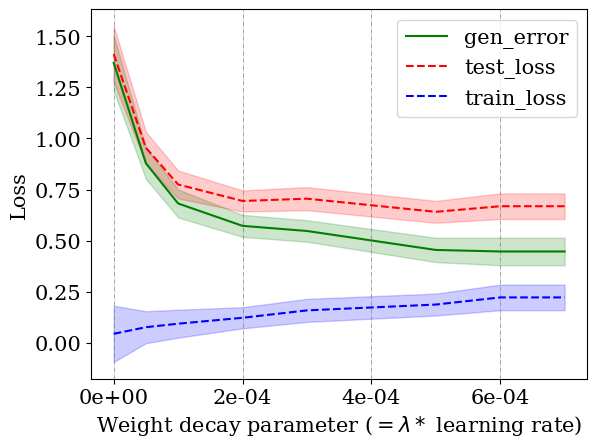}
			\label{Fig:bin_cifar10_dot_cat wd}} }
	\vspace{-7pt}
	\caption{(a)--(c) \& (e)--(g): easier-to-distinguish classes ``horse'' vs.\ ``ship'' and ``automobile'' vs.\ ``truck''; (d),(h), (i)--(k): harder-to-distinguish classes ``cat'' vs.\ ``dog''.}
	\label{Fig:real exp}
	\vspace{-10pt}
\end{figure*}

\textbf{Experimental observations:}
We perform each experiment $3$ times and report the average test and training (cross entropy)  losses, the gen-error, and test and training accuracies in Figures \ref{Fig:real exp}. 
To illustrate the difficulty level of classification for each pair, we first calculated the mean and variance of the RGB (i.e., the red-green-blue color values) values of the images to show the difference of the images between the two classes.
In Table \ref{Table:RGB mean var}, we display the RGB means and variances of the test data in six classes taken from the CIFAR10 dataset. % and measure the $\ell_2$ distance of the means and variances inside each pair. 
We observe that the RGB variances of each pair are almost $0$ (and small compared to the RGB-mean $\ell_2$ distances), and thus, the  RGB-mean $\ell_2$ distance is indicative of the difficulty of the classification task. Indeed, a smaller RGB-mean $\ell_2$ distance implies a higher overlap of the two classes and consequently,  greater difficulty in distinguishing them. Therefore,   the ``cat-dog'' pair, which is more  difficult to  disambiguate compared to the ``horse-ship'' and ``automobile-truck'' pairs, is analogous to the bGMM with large variance (i.e. large overlap between the positive and negative classes).  
Furthermore, in Table \ref{Table:confusion matrix}, we quote the commonly-used confusion matrix for the CIFAR10 dataset in \cite[Fig.~7]{liu2018unsupervised}, which quantifies how many out of 1000 images of each class are misclassified to any other class. It is obvious that fewer misclassified images indicates lower classification difficulty, which corresponds with Table \ref{Table:RGB mean var}. These two tables   provide an indication of the level of difficulty to distinguish different pairs of classes.

In Figures~\ref{Fig:bin_cifar10_horse_ship loss}--\ref{Fig:mul_mnist loss}, for easier-to-distinguish classes (based on the high classification accuracy and low loss, as well as Tables \ref{Table:RGB mean var} and \ref{Table:confusion matrix}),
 the gen-error appears to have relatively large reduction in the early training iterations and then fluctuates around a constant value afterwards. For example, in Figure~\ref{Fig:bin_cifar10_horse_ship loss}, the gen-error converges to around $0.25 $ after $5$ iterations; in Figure~\ref{Fig:bin_cifar10_car_truck loss}, it converges to around $0.45$ after $5$ iterations. For multi-classification of MNIST in Figure~\ref{Fig:mul_mnist loss}, the gen-error also converges to around $0.25$ after $20$ iterations. These results corroborate the theoretical and empirical analyses in the bGMM case with small variance, which again verifies that  Theorem~\ref{Thm:exact gen GMM} and Corollary~\ref{Coro:reuse labeldata gen bound} can shed light on the empirical gen-error on benchmark datasets. It also reveals that the generalization performance of iterative self-training on real datasets from relatively distinguishable classes can be quickly improved with the help of unlabelled data. In Figures \ref{Fig:bin_cifar10_horse_ship acc}, \ref{Fig:bin_cifar10_car_truck acc} and \ref{Fig:mul_mnist acc}, we also show that the test accuracy increases with the iterations and has significant improvement compared to the initial iteration when only labelled data are used.
 % In Figure~\ref{Fig:bin_cifar10_horse_ship}, the highest accuracy has about a $4\%$ increase from the initial point; in Figure~\ref{Fig:bin_cifar10_car_truck}, there is about   a $10\%$ increase; and in Figure \ref{Fig:mul_mnist}, there is about a $3\%$ increase. 

 In Figures \ref{Fig:bin_cifar10_cat_dog loss} and \ref{Fig:bin_cifar10_cat_dog acc}, we perform another binary classification experiment on the harder-to-distinguish pair, ``cat'' and ``dog'' (based on low accuracy at the initial point as well as Tables \ref{Table:RGB mean var} and \ref{Table:confusion matrix}).
We observe that the gen-error (and the test loss) does not decrease across the self-training iterations even though the test accuracy increases. This again corroborates the result in Figure \ref{Fig:exactgen_d=2_sig3} for the bGMM with large variance. The fact that both the  test loss and test accuracy appear to increase with the iteration is, in fact, not contradictory.  To intuitively explain this, in binary classification using the softmax (hence, logistic)  function to predict the output classes, suppose the learned probability of a data example belonging to its true class is $p \in (1/2, 1]$, the classification is correct. In other words, the accuracy is 100\%. However, when $p$ (i.e., the classification confidence) decreases towards $(1/2)^+$, the corresponding decision margin $2p-1$ \citep{cao2019learning} also decreases and the test loss  $-\log p$ increases commensurately. Thus, when the decision margin is small, even though the test accuracy may increase as the iteration counter increases, the test loss may also increase at the same time; this represents our lack of confidence.

We further investigate the effect of $\ell_2$-regularization on ``cat-dog'' classification. In Figures \ref{Fig:bin_cifar10_dot_cat reg loss} and \ref{Fig:bin_cifar10_dot_cat reg acc}, we show that by setting the weight decay parameter to be $0.0005$, the increase of gen-error for the ``cat-dog'' classification task can be mitigated and the test accuracy is improved by 0.6\% as well; compare this to Figures \ref{Fig:bin_cifar10_cat_dog loss} and \ref{Fig:bin_cifar10_cat_dog acc}. In Figure \ref{Fig:bin_cifar10_dot_cat wd}, we plot the average gen-error over the last 10 iterations versus the weight decay parameter; this is   shown to decrease as the weight decay increases (compared to Figure \ref{fig:gen1_bd_and_emp_diffsig}). %This agrees with the our theoretical and synthetic results in Figure \ref{fig:gen1_bd_and_emp_diffsig}.  
In summary, our above observations correspond to that  for the bGMM, namely that the unlabelled data do not always help to improve the gen-error but adding regularization can help to compensate for the undesirable impact. 

\begin{figure}[t]
	\centering
	\includegraphics[width=.5\linewidth]{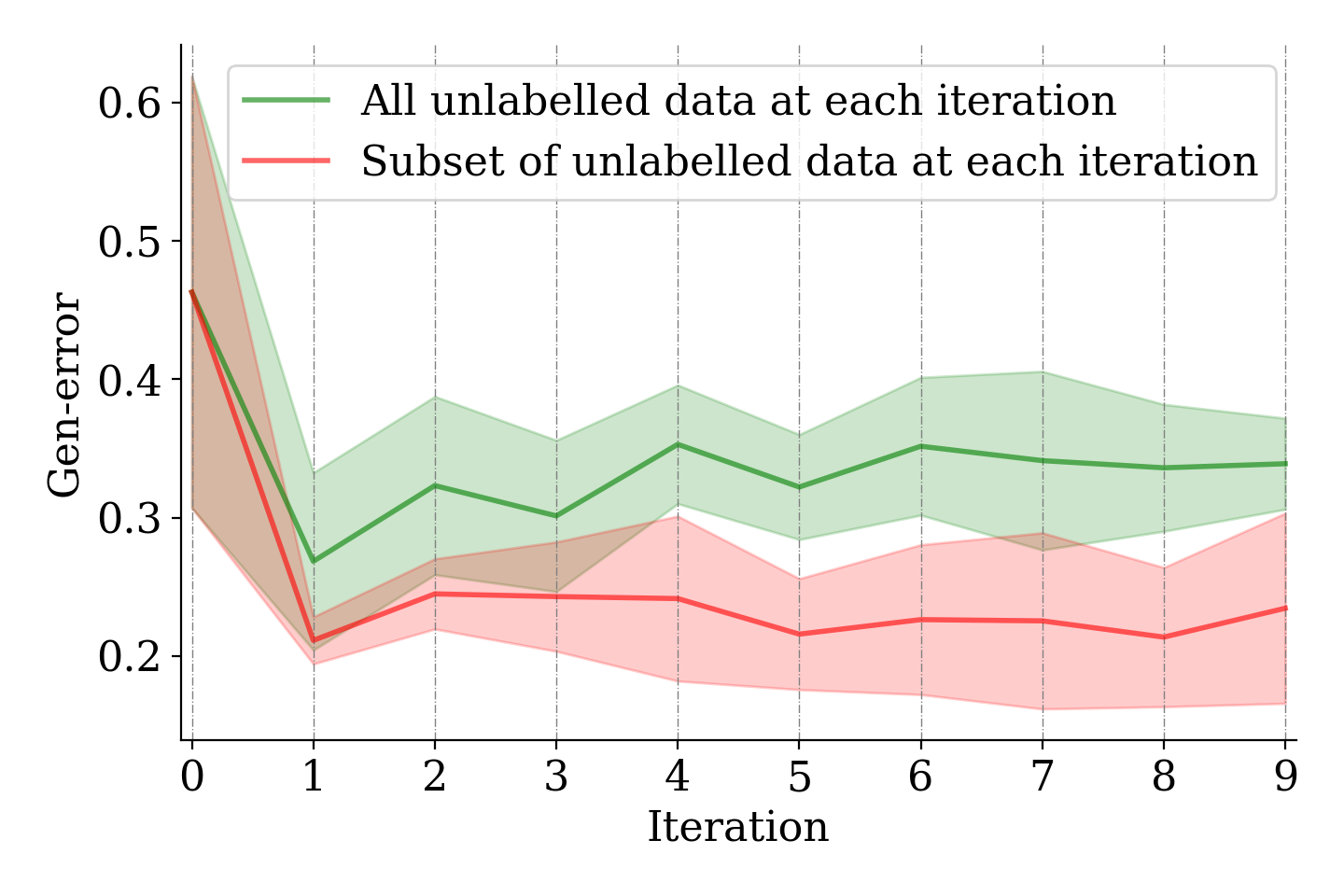}
	\caption{Comparison of the  gen-error  for the ``horse'' and ``ship'' classification task}
	\label{Fig: horse-ship all vs subset}
\end{figure}

Furthermore, we study the effect of reusing {\em all} the unlabelled data at each iteration.
	Under the same experimental setup as above, we conduct an additional experiment on the ``horse-ship'' pair in the CIFAR-10 dataset by using {\em all} $9,500$ unlabelled images in \emph{each} iteration. The self-training procedure lasts for $10$ iterations.  Figure \ref{Fig: horse-ship all vs subset}  compares the gen-error of this additional experiment with the one of the same experiment in Figure~\ref{Fig:bin_cifar10_horse_ship loss}. We find that when using the unlabelled data all at once, as the pseudo-labelling iteration increases, the gen-error is even higher than that for our original setup. This can possibly be attributed to overfitting.

\section{Concluding Remarks and Future Work}
In this paper, we have analyzed the gen-error of iterative SSL algorithms that pseudo-label large amounts of unlabelled data to progressively refine the parameters of a given model. We particularized the general bounds and exact expressions on the gen-error  for the bGMM to gain some theoretical insight into the problem. These were then corroborated by experiments on benchmark datasets. The theoretical analyses and experimental results reinforce the main message of this paper---namely, that in the low-class-overlap or easy-to-classify scenario, pseudo-labelling can help to reduce the gen-error. On the other hand, for the high-class-overlap or difficult-to-classify scenario, pseudo-labelling can in fact hurt. Thus, the key takeaway from our paper is that practitioners should be judicious in adopting pseudo-labelling techniques, for they may degrade the overall performance.

There are three avenues for future research. First, our analytical results are only applicable to the bGMM. This yields valuable insights, but the model is admittedly restrictive. Generalizing our analyses to other statistical models for classification such as logistic regression  will  be instructive. Secondly, our work focuses on the gen-error. Often bounds on the {\em population risk} are desired as the population risk is the key determinant of the performance of classification algorithms. Bounding the population risk in the SSL setting would thus be interesting. Finally, analyzing other families of SSL algorithms beyond those that utilize pseudo-labelling would provide a clearer theoretical picture about the utility of  SSL.  

\subsection*{Acknowledgements}

We sincerely thank the reviewers and action editor for their meticulous reading and insightful comments that led to a significantly improved version of the paper. 

This research/project is supported by the National Research Foundation Singapore and DSO National Laboratories under the AI Singapore Programme (AISG Award No: AISG2-RP-2020-018), by  an
  NRF Fellowship (A-0005077-01-00), and by Singapore Ministry of Education (MOE) AcRF Tier 1 Grants (A-0009042-01-00 and A-8000189-01-00).
%%and  Singapore Ministry of Education AcRF Tier 1 Grants (A-0009042-01-00, A-8000189-01-00). This
%research is also supported by the National Research Foundation, Singapore under its AI
%Singapore Programme (AISG Award No: AISG2-RP-2020-018). Any opinions findings
%and conclusions or recommendations expressed in this material are those of the author(s)
%and do not reflect the views of National Research Foundation, Singapore.

\appendix
\section*{Appendices}

\section{Proof of Theorem \ref{Coro: sub Gau gen}}\label{pf of Thm: iterative gen bound}

%Before the proof, let us define some notation.
%The \emph{cumulant generating function} (CGF) of a random variable $L\in\bbR$ is $\Lambda_L(\lambda):=\log\bbE_L\big[e^{\lambda(L-\bbE[L])}\big]$ for all  $\lambda\in\bbR$.
%%\begin{align}
%%	\Lambda_L(\lambda):=\log\bbE_L\big[e^{\lambda(L-\bbE[L])}\big], \quad \text{for all } \lambda\in\bbR.
%%\end{align} 
%Note that $\Lambda_L(0)=\Lambda_L'(0)=0$ and $\Lambda_L(\lambda)$ is convex. Then for any $L\sim \mathrm{subG}(R)$, it means $\Lambda_L(\lambda)\leq \frac{R^2 \lambda^2}{2}$, for all $\lambda\in\bbR$.

We commence with some notation. For any convex function $\psi:[0,b) \mapsto \bbR$, its \emph{Legendre dual} $\psi^*$ is defined as $\psi^*(x):=\sup_{\lambda\in[0,b)} \lambda x-\psi(\lambda)$ for all $x\in[0, \infty)$.
%\begin{align}
%	\psi^*(x):=\sup_{\lambda\in[0,b)} \lambda x-\psi(\lambda), \quad \text{for all } x\in[0, \infty). \label{Def: Legendre dual}
%\end{align}
According to \citet[Lemma 2.4]{boucheron2013concentration}, when $\psi(0)=\psi'(0)=0$, $\psi^*(x)$ is a nonnegative convex and nondecreasing function on $[0, \infty)$. Moreover, for every $y\geq 0$, its generalized inverse function $\psi^{*-1}(y):=\inf\{x \geq 0: \psi^*(x)\geq y\}$ is concave and can be rewritten as $\psi^{*-1}(y)=\inf_{\lambda\in[0,b)}\frac{y+\psi(\lambda)}{\lambda}$.
%\begin{align}
%	\psi^{*-1}(y)&=\inf_{\lambda\in[0,b)}\frac{y+\psi(\lambda)}{\lambda}. \label{Def: inverse Legendre dual}
%\end{align}

We first introduce the following theorem that is applicable to more general loss functions.
\begin{theorem}\label{Thm: iterative gen bound}
	For any $\tilde{\theta}_t\in\Theta$, let $\psi_{-}(\lambda,\tilde{\theta}_t)$ and $\psi_{+}(\lambda,\tilde{\theta}_t)$ be convex functions of $\lambda$ and $\psi_{+}(0,\tilde{\theta}_t)=\psi_{+}'(0,\tilde{\theta}_t)=\psi_{-}(0,\tilde{\theta}_t)=\psi_{-}'(0,\tilde{\theta}_t)=0$.
	Assume  that $\Lambda_{l(\tilde{\theta}_t,\tilZ)}(\lambda,\tilde{\theta}_t)\leq \psi_{+}(\lambda,\tilde{\theta}_t)$ for all $\lambda\in[0,b_+)$ and $\Lambda_{l(\tilde{\theta}_t,\tilZ)}(\lambda,\tilde{\theta}_t)\leq \psi_{-}(\lambda,\tilde{\theta}_t)$ for $\lambda\in(b_-,0]$ under distribution $P_{\tilZ|\theta^{(t-1)}}=P_Z$, where $0<b_+ \leq \infty$ and $-\infty \leq b_-<0$. Let $\psi_+(\lambda)=\sup_{\tilde{\theta}_t}\psi_{+}(\lambda,\tilde{\theta}_t)$ and $\psi_-(\lambda)=\sup_{\tilde{\theta}_t}\psi_{-}(\lambda,\tilde{\theta}_t)$. We have
	\begin{align}
		&\mathrm{gen}_t(P_Z, P_X, \{P_{\theta_k|S_{\rml},S_{\rmu}}\}_{k=0}^t, \{f_{\theta_k}\}_{k=0}^{t-1} ) \nn\\
		&\leq \frac{w}{n}\sum_{i=1}^n\bbE_{\theta^{(t-1)}}\left[\psi_-^{*-1}(I_{\theta^{(t-1)}}(\theta_t;Z_i))\right] \nn\\
		&\quad  +\frac{1-w}{m}\sum_{i\in\calI_t}\bbE_{\theta^{(t-1)}}\left[\psi_-^{*-1}\big(I_{\theta^{(t-1)}}(\theta_t;X_i',\hatY_i')+D_{\theta^{(t-1)}}(P_{X_i',\hatY_i'}\| P_Z) \big)\right],
		\end{align}
		and
		\begin{align}
		&-\mathrm{gen}_t(P_Z, P_X, \{P_{\theta_k|S_{\rml},S_{\rmu}}\}_{k=0}^t, \{f_{\theta_k}\}_{k=0}^{t-1} ) \nn\\
		&\leq \frac{w}{n}\sum_{i=1}^n\bbE_{\theta^{(t-1)}}\left[\psi_+^{*-1}(I_{\theta^{(t-1)}}(\theta_t;Z_i))\right]  \nn\\
		&\quad  +\frac{1-w}{m}\sum_{i\in\calI_t}\bbE_{\theta^{(t-1)}}\left[\psi_+^{*-1}\big(I_{\theta^{(t-1)}}(\theta_t;X_i',\hatY_i')+D_{\theta^{(t-1)}}(P_{X_i',\hatY_i'}\| P_Z) \big)\right],
	\end{align}
	where $P_{X_i',\hatY_i'|\theta^{(t-1)}}(x,y|\hat{\theta}^{(t-1)})=P_X(x)\mathbbm{1}\{y=f_{\hat{\theta}_{t-1}}(x)\}$ for any $x\in\calX$, $y\in\calY$ and $\hat{\theta}^{(t-1)}\in\Theta^{t-1}$, and $P_{Z|\theta^{(t-1)}}=P_Z$.
\end{theorem}
%The proof of Theorem \ref{Thm: iterative gen bound} is provided in Appendix \ref{pf of Thm: iterative gen bound}. 
\begin{proof}
Consider the Donsker--Varadhan variational representation of the KL-divergence between any two distributions $P$ and $Q$ on $\calX$:
\begin{align}
	D(P\|Q)=\sup_{g\in\calG}\Big\{\bbE_{X\sim P}[g(X)]-\log\bbE_{X\sim Q}[e^{g(X)}]\Big\},
\end{align}
where the supremum is taken over the set of  measurable functions in $\calG=\{g:\calX \mapsto \bbR: \bbE_{X\sim Q}[e^{g(X)}]<\infty\}$.

Recall that $\tilde{\theta}_t$ and $\tilZ$ are independent copies of $\theta_t$ and $Z$ respectively, such that $P_{\tilde{\theta}_t,\tilZ}=Q_{\theta_t}\otimes P_Z$, $P_{\tilde{\theta}_t,\tilZ|\theta^{(t-1)}}=P_{\theta_t|\theta^{(t-1)}}\otimes P_Z$.
For any iterative SSL algorithm, by applying the law of total expectation, the generalization error can be rewritten as
\begin{align}
	&\mathrm{gen}_t(P_Z, P_X, \{P_{\theta_k|S_{\rml},S_{\rmu}}\}_{k=0}^t, \{f_{\theta_k}\}_{k=0}^{t-1} ) \nn\\ &=w\bigg(\bbE_{\theta_t}[\bbE_{Z}[l(\theta_t,Z)]]-\frac{1}{n}\sum_{i=1}^n \bbE_{\theta_t,Z_i}[l(\theta_t,Z_i)] \bigg) \nn\\
	&\quad +(1-w)\bigg(\bbE_{\theta_t}[\bbE_{Z}[l(\theta_t,Z)]]-\frac{1}{m}\sum_{i \in\calI_t}\bbE_{\theta_t,X'_i,\hatY'_i}[l(\theta_t,(X'_i,\hatY'_i))] \bigg) \\
	&=\frac{w}{n}\sum_{i=1}^n \bigg(\bbE_{\tilde{\theta}_t,\tilZ}[l(\tilde{\theta}_t,\tilZ)]-\bbE_{\theta_t,Z_i}[l(\theta_t,Z_i)]\bigg) \nn\\*
	&\quad +\frac{1-w}{m}\sum_{i\in\calI_t}\bigg(\bbE_{\tilde{\theta}_t,\tilZ}[l(\tilde{\theta}_t,\tilZ)]-\bbE_{\theta_t,X'_i,\hatY'_i}[l(\theta_t,(X'_i,\hatY'_i))]  \bigg)\\
	&=\frac{w}{n}\sum_{i=1}^n \bbE_{\theta^{(t-1)}}\Big[ \bbE_{\tilde{\theta}_t,\tilZ}\big[l(\tilde{\theta}_t,\tilZ)|\theta^{(t-1)}\big]-\bbE_{\theta_t,Z_i}\big[l(\theta_t,Z_i)|\theta^{(t-1)}\big] \Big]  \nn\\*
	&\quad +\frac{1  -  w}{m}\sum_{i\in\calI_t}  \bbE_{\theta^{(t-1)}}\Big[\bbE_{\tilde{\theta}_t,\tilZ}\big[l(\tilde{\theta}_t,\tilZ)|\theta^{(t-1)} \big]  -  \bbE_{\theta_t,X'_i,\hatY'_i}\big[l(\theta_t,(X'_i,\hatY'_i)) |\theta^{(t-1)} \big]  \Big]. \label{Eq: condition expectation gen_t}
\end{align}

Note that $\psi_+(\lambda)=\sup_{\tilde{\theta}_t}\psi_{+}(\lambda,\tilde{\theta}_t)$ and $\psi_-(\lambda)=\sup_{\tilde{\theta}_t}\psi_{-}(\lambda,\tilde{\theta}_t)$ are convex, and so their Legendre duals $\psi_-^*$, $\psi_+^*$, and the corresponding inverses are well-defined. 
%\eqref{Def: Legendre dual} and \eqref{Def: inverse Legendre dual}.

Let $\check{l}(\theta,z)=l(\theta,z)-\bbE_{Z}[l(\theta,Z)]$. We have the fact that $\bbE_{\tilZ}[\check{l}(\tilde{\theta}_t,\tilZ)]=0$ for any $\tilde{\theta}_t$. Again, by the Donsker--Varadhan variational representation of the KL-divergence, for any fixed $\theta^{(t-1)}$ and any $\lambda\in[0,b_+)$, we have
\begin{align}
	&I_{\theta^{(t-1)}}(\theta_t;Z)=D(P_{\theta_t,Z|\theta^{(t-1)}}\|P_{\theta_t|\theta^{(t-1)}}\otimes P_Z) \nn\\
	&\geq \bbE_{\theta_t,Z}[\lambda \check{l}(\theta_t,Z)|\theta^{(t-1)}]-\log \bbE_{\tilde{\theta}_t,\tilZ}[e^{\lambda \check{l}(\tilde{\theta}_t,\tilZ)}|\theta^{(t-1)}]\\
	&= \bbE_{\theta_t,Z}[\lambda \check{l}(\theta_t,Z)|\theta^{(t-1)}]-\log \bbE_{\tilde{\theta}_t|\theta^{(t-1)}}\bbE_{\tilZ}\big[e^{\lambda \check{l}(\tilde{\theta}_t,\tilZ)}\big] \\
	&=\bbE_{\theta_t,Z}[\lambda \check{l}(\theta_t,Z)|\theta^{(t-1)}]-\log \bbE_{\tilde{\theta}_t|\theta^{(t-1)}}\big[\exp\big(\Lambda_{l(\tilde{\theta}_t,\tilZ)}(\lambda,\tilde{\theta}_t) \big) \big] \label{Eq: use def Lambda(theta)}\\
	&\geq \lambda\bbE_{\theta_t,Z}[l(\theta_t,Z)-\bbE_{Z}[l(\theta_t,Z)]|\theta^{(t-1)}]-\log \bbE_{\tilde{\theta}_t|\theta^{(t-1)}}\big[\exp(\psi_+(\lambda,\tilde{\theta}_t)) \big] \label{Eq: use Lambda(theta) bound} \\
	&\geq \lambda\bbE_{\theta_t,Z}[l(\theta_t,Z)-\bbE_{Z}[l(\theta_t,Z)]|\theta^{(t-1)}]-\psi_+(\lambda) \label{Eq: use sup psi}\\
	&=\lambda\big(\bbE_{\theta_t,Z}[l(\theta_t,Z)|\theta^{(t-1)}]-\bbE_{\tilde{\theta}_t,\tilZ}[l(\tilde{\theta}_t,\tilZ)|\theta^{(t-1)}] \big)-\psi_+(\lambda).
\end{align}
where \eqref{Eq: use def Lambda(theta)} follows from the definition of $\Lambda_{l(\tilde{\theta}_t,\tilZ)}(\lambda,\tilde{\theta}_t)$ in \eqref{Def:CGF fixed theta}, \eqref{Eq: use Lambda(theta) bound} follows from the assumption that $\Lambda_{l(\tilde{\theta}_t,\tilZ)}(\lambda,\tilde{\theta}_t)\leq \psi_+(\lambda,\tilde{\theta}_t)$ for all $\lambda\in[0,b_+)$, and \eqref{Eq: use sup psi} follows because  $\psi_+(\lambda)=\sup_{\tilde{\theta}_t}\psi_{+}(\lambda,\tilde{\theta}_t)$. Thus, we have
\begin{align}
	&\bbE_{\theta_t,Z}[l(\theta_t,Z)|\theta^{(t-1)}]-\bbE_{\tilde{\theta}_t,\tilZ}[l(\tilde{\theta}_t,\tilZ)|\theta^{(t-1)}] \nn\\
	&\leq \inf_{\lambda\in[0,b_+)}\frac{I_{\theta^{(t-1)}}(\theta_t;Z)+\psi_{+}(\lambda)}{\lambda}=\psi_+^{*-1}\big(I_{\theta^{(t-1)}}(\theta_t;Z)\big). \label{Eq: psi_+ 1}
\end{align}
Similarly, for $\lambda\in(b_-,0]$,
\begin{align}
	&\bbE_{\tilde{\theta}_t,\tilZ}[l(\tilde{\theta}_t,\tilZ)|\theta^{(t-1)}]-\bbE_{\theta_t,Z}[l(\theta_t,Z)|\theta^{(t-1)}] \nn\\
	&\leq \inf_{\lambda\in[0,-b_-)}\frac{I_{\theta^{(t-1)}}(\theta_t;Z)+\psi_{-}(\lambda)}{\lambda}=\psi_-^{*-1}\big(I_{\theta^{(t-1)}}(\theta_t;Z)\big). \label{Eq: psi_- 1}
\end{align}

By applying the same techniques, for any pair of pseudo-labelled random variables $(X',\hatY')$ used at iteration $t$ and any $\lambda\in[0,b_+)$, we have
\begin{align}
	&I_{\theta^{(t-1)}}(\theta_t;X',\hatY')+D_{\theta^{(t-1)}}(P_{X',\hatY'}\| P_Z) \nn\\
	&=D_{\theta^{(t-1)}}(P_{\theta_t,X',\hatY'}\|P_{\theta_t}\otimes P_{X',\hatY'})+D_{\theta^{(t-1)}}(P_{\theta_t}\otimes P_{X',\hatY'}\| P_{\theta_t}\otimes P_Z) \label{Eq:I+D bound begin}\\
	&\geq \bbE_{\theta_t,X',\hatY'}[\lambda l(\theta_t,(X',\hatY')) | \theta^{(t-1)}]-\log \bbE_{\theta_t}\big[\bbE_{X',\hatY'}[e^{\lambda l(\theta_t,(X',\hatY'))} |\theta^{(t-1)}] |\theta^{(t-1)} \big] \nn\\*
	&\quad +\bbE_{\theta_t}\big[\bbE_{X',\hatY'}[\lambda l(\theta_t,(X',\hatY')) |\theta^{(t-1)}] | \theta^{(t-1)} \big]-\log \bbE_{\theta_t}\big[\bbE_{Z}[e^{\lambda l(\theta_t,Z)}] |\theta^{(t-1)} \big]\\
	&\geq \bbE_{\theta_t,X',\hatY'}[\lambda l(\theta_t,(X',\hatY'))|\theta^{(t-1)}]-\log \bbE_{\theta_t}\big[\bbE_{Z}[e^{\lambda l(\theta_t,Z)}] | \theta^{(t-1)} \big] \label{Eq: use Jensen ineq} \\
	&=\lambda \Big(\bbE_{\theta_t,X',\hatY'}[\lambda l(\theta_t,(X',\hatY')) |\theta^{(t-1)}]-\bbE_{\theta_t}\big[\bbE_{Z}[l(\theta_t,Z)] | \theta^{(t-1)} \big] \Big) \nn\\
	&\quad -\log \bbE_{\tilde{\theta}_t|\theta^{(t-1)}}\big[\exp\big(\Lambda_{l(\tilde{\theta}_t,\tilZ)}(\lambda,\tilde{\theta}_t) \big) \big]\\
	&\geq \lambda \Big(\bbE_{\theta_t,X',\hatY'}[\lambda l(\theta_t,(X',\hatY'))]-\bbE_{\theta_t}\big[\bbE_{Z}[l(\theta_t,Z)] \big] \Big)-\psi_+(\lambda) \label{Eq:I+D bound end},
\end{align}
where \eqref{Eq: use Jensen ineq} follows from Jensen's inequality.
Thus, we get
\begin{align}
	&\bbE_{\theta_t,X'_i,\hatY'_i}\big[l(\theta_t,(X',\hatY')) |\theta^{(t-1)} \big]-\bbE_{\tilde{\theta}_t,\tilZ}\big[l(\tilde{\theta}_t,\tilZ)|\theta^{(t-1)} \big]\nn\\
	&\leq \psi_+^{*-1}\big(I_{\theta^{(t-1)}}(\theta_t;X',\hatY')+D_{\theta^{(t-1)}}(P_{X',\hatY'}\| P_Z) \big) \label{Eq: psi_+ 2}
\end{align}
and 
\begin{align}
	&\bbE_{\tilde{\theta}_t,\tilZ}\big[l(\tilde{\theta}_t,\tilZ)|\theta^{(t-1)} \big]-\bbE_{\theta_t,X'_i,\hatY'_i}\big[l(\theta_t,(X',\hatY')) |\theta^{(t-1)} \big] \nn\\
	& \leq \psi_-^{*-1}\big(I_{\theta^{(t-1)}}(\theta_t;X',\hatY')+D_{\theta^{(t-1)}}(P_{X',\hatY'}\| P_Z) \big). \label{Eq: psi_- 2}
\end{align}
The proof is completed  by applying inequalities \eqref{Eq: psi_+ 1}, \eqref{Eq: psi_- 1}, \eqref{Eq: psi_+ 2} and \eqref{Eq: psi_- 2} to the expansion of $\mathrm{gen}_t$ in~\eqref{Eq: condition expectation gen_t}. %, the proof is completed.
\end{proof}

Let $\tilde{\theta}_t$ and $\tilZ$ be independent copies of $\theta_t$ and $Z$ respectively, such that $P_{\tilde{\theta}_t,\tilZ}=Q_{\theta_t}\otimes P_{Z}$, where $Q_{\theta_t}$ is the marginal distribution of $\theta_t$. 
For any fixed $\tilde{\theta}_t\in\Theta$, let the cumulant generating function (CGF) of $l(\tilde{\theta}_t,\tilZ)$ be
\begin{align}
	\Lambda_{l(\tilde{\theta}_t,\tilZ)}(\lambda,\tilde{\theta}_t):=\log \bbE_{\tilZ}[e^{\lambda(l(\tilde{\theta}_t,\tilZ)-\bbE_{\tilZ}[l(\tilde{\theta}_t,\tilZ)])}]. \label{Def:CGF fixed theta}
\end{align}

When the loss function $l(\theta,Z)\sim \mathrm{subG}(R)$ under $Z\sim P_Z$ for any $\theta\in\Theta$, we have $\Lambda_{l(\tilde{\theta}_t,\tilZ)}(\lambda,\tilde{\theta}_t)\leq \frac{R^2\lambda^2}{2}$ for all $\lambda\in\bbR$.
Then we can let $\psi_{-}(\lambda,\tilde{\theta}_t)=\psi_{+}(\lambda,\tilde{\theta}_t)=\frac{R^2\lambda^2}{2}$ for all $\tilde{\theta}_t\in\Theta$. Hence, $\psi_{+}(\lambda)=\psi_{-}(\lambda)=\sup_{\tilde{\theta}_t\in\Theta}\frac{R^2 \lambda^2}{2}=\frac{R^2 \lambda^2}{2}$ and  $\psi_+^{*-1}(y)=\psi_-^{*-1}(y)=\sqrt{2 R^2 y}$ for any $y\geq 0$.
Finally, Theorem \ref{Coro: sub Gau gen} can then be directly obtained from Theorem \ref{Thm: iterative gen bound}.

\section{Proof of Theorem \ref{Thm:exact gen}}\label{pf of Thm:exact gen}
Recall that $\tilde{\theta}_t$ and $\tilZ$ are independent copies of $\theta_t$ and $Z$ respectively, such that $P_{\tilde{\theta}_t,\tilZ}=Q_{\theta_t}\otimes P_Z$, $P_{\tilde{\theta}_t,\tilZ|\theta^{(t-1)}}=P_{\theta_t|\theta^{(t-1)}}\otimes P_Z$.

Recall the gen-error given in \eqref{Eq: condition expectation gen_t}. The first term in \eqref{Eq: condition expectation gen_t} can be rewritten as
\begin{align}
	&\bbE_{\theta^{(t-1)}}\Big[ \bbE_{\tilde{\theta}_t,\tilZ}\big[l(\tilde{\theta}_t,\tilZ)|\theta^{(t-1)}\big]-\bbE_{\theta_t,Z_i}\big[l(\theta_t,Z_i)|\theta^{(t-1)}\big] \Big] \nn\\
	&=\bbE_{\tilde{\theta}_t,\tilZ}\big[-\log p_{\tilde{\theta}_t}\big]-\bbE_{\theta_t,Z_i}\big[-\log p_{\theta_t}\big] \\
	&=\bbE_{\theta_t}\Big[h(P_Z,p_{\theta_t})-h(P_{Z_i|\theta_t},p_{\theta_t}) \Big]\\
	&=\bbE_{\theta_t}\Big[\HD(P_Z\|P_{Z_i|\theta_t}|p_{\theta_t}) \Big]. \label{Eq:first term gen}
\end{align}
The second term in \eqref{Eq: condition expectation gen_t} can be rewritten as
\begin{align}
	&\bbE_{\tilde{\theta}_t,\tilZ}[l(\tilde{\theta}_t,\tilZ)]-\bbE_{\theta_t,X'_i,\hatY'_i}[l(\theta_t,(X'_i,\hatY'_i))] \nn\\
	&=\bbE_{\theta_t,Z_i}[-\log p_{\theta_t} ] -\bbE_{\theta_t,X'_i,\hatY'_i}[-\log p_{\theta_t} ] \\
	&=\bbE_{\theta^{(t-1)}}\Big[\bbE_{\theta_t|\theta^{(t-1)}}\big[h(P_Z,p_{\theta_t})- h(P_{X_i',\hatY_i'|\theta^{(t-1)}},p_{\theta_t})+h(P_{X_i',\hatY_i'|\theta^{(t-1)}},p_{\theta_t}) \nn\\
	&\quad -h(P_{X_i',\hatY_i'|\theta^{(t)}},p_{\theta_t}) \big]  \Big] \\
	&= \bbE_{\theta^{(t)}}\Big[\HD(P_Z \| P_{X_i',\hatY_i'|\theta^{(t-1)}} | p_{\theta_t})+\HD(P_{X_i',\hatY_i'|\theta^{(t-1)}} \| P_{X_i',\hatY_i'|\theta^{(t)}} |p_{\theta_t} )   \Big]. \label{Eq:second term gen}
\end{align}
By combining \eqref{Eq:first term gen} and \eqref{Eq:second term gen}, the gen-error is finally given by
\begin{align}
	&\mathrm{gen}_t(P_Z, P_X, \{P_{\theta_k|S_{\rml},S_{\rmu}}\}_{k=0}^t, \{f_{\theta_k}\}_{k=0}^{t-1} )\nn\\ &=\frac{w}{n}\sum_{i=1}^{n}\bbE_{\theta_t}\Big[\HD(P_Z\|P_{Z_i|\theta_t}|p_{\theta_t}) \Big] \nn\\
	&\quad +\frac{w}{m}\sum_{i\in\calI_t} \bbE_{\theta^{(t)}}\Big[\HD(P_Z \| P_{X_i',\hatY_i'|\theta^{(t-1)}} | p_{\theta_t})+\HD(P_{X_i',\hatY_i'|\theta^{(t-1)}} \| P_{X_i',\hatY_i'|\theta^{(t)}} |p_{\theta_t} )   \Big]\\
	&=\bbE_{\theta^{(t)}}\bigg[\frac{w}{n}\sum_{i=1}^{n}\HD(P_Z\|P_{Z_i|\theta_t}|p_{\theta_t}) +\frac{w}{m}\sum_{i\in\calI_t} \big(\HD(P_Z \| P_{X_i',\hatY_i'|\theta^{(t-1)}} | p_{\theta_t})\nn\\*
	&\quad +\HD(P_{X_i',\hatY_i'|\theta^{(t-1)}} \| P_{X_i',\hatY_i'|\theta^{(t)}} |p_{\theta_t} ) \big)  \bigg].
\end{align}
Theorem \ref{Thm:exact gen} is thus proved.

\section{Proof of Theorem \ref{Thm:exact gen GMM}}\label{pf of Thm:exact gen GMM}
In the following, we abbreviate $\mathrm{gen}_t(P_{\bZ}, P_{\bX}, \{P_{\btheta_k|S_{\rml},S_{\rmu}}\}_{k=0}^t, \{f_{\btheta_k}\}_{k=0}^{t-1} )$ as $\mathrm{gen}_t$, if there is no risk of confusion. When the labelled data are not reused in the subsequent iterations, for $t\geq 1$, $w=0$.
\begin{itemize}
	\item \textbf{Iteration $t=0$:} Since $Y_i\bX_i\overset{\text{i.i.d.}}{\sim} \calN(\bmu,\sigma^2 \bI_d)$, we have $\btheta_0\sim \calN(\bmu,\frac{\sigma^2}{n} \bI_d)$. The gen-error $\mathrm{gen}_0$ is given by
	\begin{align}
		\mathrm{gen}_0&=\bbE_{\btheta_0} \bigg[\bbE_{\bZ}[-\log p_{\btheta_0}(\bZ)]- \bbE_{\bZ_i|\btheta_0}[-\log p_{\btheta_0}(\bZ_i)] \bigg] \nn\\
		&=\int Q_{\btheta_0}(\btheta)(P_{\bZ}(\bz)-P_{\bZ_i|\btheta_0}(\bz|\btheta))\log\frac{1}{p_{\btheta}(\bz)} \rmd \bz \rmd \btheta\\
		&=\frac{1}{2\sigma^2} \int Q_{\btheta_0}(\btheta)(P_{\bZ}(\bx,y)-P_{\bZ_i|\btheta_0}(\bx,y|\btheta)) \big(\bx^\top\bx-2y\btheta^\top\bx+\btheta^\top\btheta \big) \rmd \bx \rmd y \rmd \btheta\\
		&=-\frac{1}{2\sigma^2}\int P_{\bZ}(\bx,y)(Q_{\btheta_0}(\btheta)-P_{\btheta_0|\bZ_i}(\btheta|\bx,y))2y\btheta^\top\bx ~\rmd \bx \rmd y \rmd \btheta\\
		&=-\frac{1}{2\sigma^2}\int \Big( P_{\bZ}(\bx,1)(Q_{\btheta_0}(\btheta)-P_{\btheta_0|\bZ_i}(\btheta|\bx,1)) \nn\\*
		&\qquad \qquad \quad -  P_{\bZ}(\bx,-1)(Q_{\btheta_0}(\btheta)-P_{\btheta_0|\bZ_i}(\btheta|\bx,-1)) \Big)2\btheta^\top\bx ~\rmd \bx \rmd \btheta\\
		&=-\frac{1}{\sigma^2}\int \Big(\frac{\bmu-\bx}{n} P_{\bZ}(\bx,1)- \frac{\bmu+\bx}{n}P_{\bZ}(\bx,-1) \Big)^\top\bx ~\rmd \bx\\
		&=-\frac{1}{\sigma^2}\bigg(-\frac{d\sigma^2}{2n}-\frac{d\sigma^2}{2n} \bigg)\\
		&=\frac{d}{n}.
	\end{align}
	
	\item \textbf{Pseudo-label using $\btheta_0$:} For any $i\in[1:m]$ and $X_i'\in S_{\rmu}$, the pseudo-label is
	\begin{align}
		\hat{Y}_{i}'=\sgn(\btheta_0^\top \bX'_i).
	\end{align}
	Given any pair of $(\xi_0,\bmu^{\bot})$, $\btheta_0$ is fixed and $\{\hat{Y}_{i}' \}_{i\in[1:m]}$ are conditionally i.i.d.\ from  $P_{\hatY'|\xi_0,\bmu^{\bot}}\in\calP(\calY)$.
	Recall the pseudo-labelled dataset is defined as $\hat{S}_{u,1}=\{(\bX'_i,\hat{Y}_{i}')\}_{i=1}^m$. 
	
	Since $\btheta_0\sim\calN(\bmu,\frac{\sigma^2}{n} \bI_d)$, inspired by \citet{oymak2021theoretical}, we can decompose it as follows:
	\begin{align}\label{Eq:theta_0}
		\btheta_0&=\bmu+\frac{\sigma}{\sqrt{n}}\bxi=\bmu+\frac{\sigma}{\sqrt{n}}(\xi_0\bmu+\bmu^{\bot})=\bigg(1+\frac{\sigma}{\sqrt{n}}\xi_0\bigg)\bmu+\frac{\sigma}{\sqrt{n}}\bmu^{\bot}, 
	\end{align}
	where $\bxi\sim\calN(0,\bI_d)$, $\xi_0\sim\calN(0,1)$, $\bmu^{\bot} \perp \bmu$, $\bmu^{\bot}\sim\calN(0,\bI_d-\bmu\bmu^\top)$ and $\bmu^{\bot}$ is independent of $\xi_0$.
	
	Recall the correlation between $\btheta_0$ and $\bmu$ given in~\eqref{Eq:rho_0}, the decomposition of $\bar{\btheta}_0$ in~\eqref{Eq:theta0 decomp mu up} and $\alpha, \beta$.
	Since $\sgn(\btheta_0^\top \bX_i')=\sgn(\bar{\btheta}_0^\top \bX_i')$, in the following we can analyze the normalized parameter $\bar{\btheta}_0$ instead. 
	
	%	Assume the initial model parameter $\btheta_0$ and the mean vector $\bmu$ have correlation $\rho(\btheta_0,\bmu)=\alpha>0$. Let $\beta=\sqrt{1-\alpha^2}$. We can then decompose $\btheta_0$ as follows
	%	\begin{align}
		%		\btheta_0=\alpha \bmu+ \beta \bup, \label{Eq:theta0 decomp mu up}
		%	\end{align}
	%	where $\bup$ is a unit norm vector orthogonal to $\bmu$. 
	
	Given any $(\xi_0,\bmu^{\bot})$, $\alpha$ is fixed, and for any $i\in\bbN$, let us define a Gaussian noise vector $\bg_i\sim\calN(0,\bI_d)$ and decompose it as follows
	\begin{align}
		\bg_i=g_{0,i} \bmu+ g_i \bup + \bg_i^{\bot}, \label{Eq: Gaussian vec decomp1}
	\end{align}
	where $g_{0,i}, g_i \sim \calN(0,1)$, $\bg_i^{\bot}\sim \calN(0,\bI_d-\bmu\bmu^\top-\bup\bup^\top)$, $\bg_i^{\bot}\perp \bmu$, $\bg_i^{\bot}\perp \bup$, and $g_{0,i}, g_i, \bg^{\bot}_i$ are mutually independent.
	
	For any sample $\bX'_i\sim \calN(\bmu,\sigma^2 \bI_d)$, we can decompose it as
	\begin{align}
		\bX_i'&=\bmu+\sigma \bg_i=\bmu +\sigma (g_{0,i} \bmu+ g_i \bup + \bg_i^{\bot}). \label{Eq:X_i decomp}
	\end{align}  
	Then we have
	\begin{align}
		\bar{\btheta}_0^\top \bX_i'&=(\alpha \bmu+ \beta \bup)^\top (\bmu+\sigma \bg_i)\\
		&=\alpha+\sigma(\alpha \bmu+ \beta \bup)^\top(g_{0,i} \bmu+ g_i \bup + \bg_i^{\bot})\\
		&=\alpha+\sigma(\alpha g_{0,i}+\beta g_i)=:\alpha+\sigma h_i. \label{Eq:theta_0X_i decomp}
	\end{align}
	Note that $h_i \sim \calN(0,1)$ for any $\alpha\in[-1,1]$. 
	Similarly, for any sample $\bX'_i\sim\calN(-\mu,\sigma^2 \bI_d)$, we have 
	\begin{align}
		\bX'_i=-\bmu+\sigma\bg_i
	\end{align}
	and 
	\begin{align}
		\bar{\btheta}_0^\top \bX_i'&=-\alpha+\sigma h_i.
	\end{align}
	
	Denote the true label of $\bX'_i$ as $Y'_i$ and $P_{Y'_i}=P_Y= \mathrm{unif}(\{-1,+1\})$.
	The probability that the pseudo-label $\hatY_i'$ is equal to $1$ is given by
	\begin{align}
		\Pr(\hatY_i'=1)&=\Pr\big(\bar{\btheta}_0^\top \bX'_i>0\big)\\
		&=\frac{1}{2} \Pr \big(\bar{\btheta}_0^\top \bX'_i>0| Y'_i=1 \big)+\frac{1}{2} \Pr \big(\bar{\btheta}_0^\top \bX'_i>0| Y'_i=-1\big)\\
		&=\frac{1}{2}\bbE_\alpha\big[\Pr(\alpha+\sigma h_i >0) \big]+\frac{1}{2}\bbE_\alpha\big[\Pr\big(-\alpha+\sigma h_i >0 \big) \big]\\
		&=\frac{1}{2}\bbE_\alpha\bigg[\rmQ \bigg(-\frac{\alpha}{\sigma} \bigg)\bigg]+\frac{1}{2}\bbE_\alpha\bigg[\rmQ \bigg(\frac{\alpha}{\sigma} \bigg) \bigg]=\frac{1}{2}. \label{Eq:hatY prob}
	\end{align}
	We also have $\Pr(\hatY_i'=-1)=1-\Pr(\hatY_i'=1)=1/2$, and so $P_{\hatY_i'}=P_Y$.
	
	\item \textbf{Iteration $t=1$:} Recall \eqref{Eq:theta_t} and the new model parameter learned from the pseudo-labelled dataset $\hatS_{\rmu,1}$ is given by
	\begin{align}
		\btheta_1&=\frac{1}{m}\sum_{i=1}^m \hatY_i'\bX'_i=\frac{1}{m}\sum_{i=1}^m\sgn(\btheta_0^\top \bX'_i)\bX'_i=\frac{1}{m}\sum_{i=1}^m\sgn(\bar{\btheta}_0^\top \bX'_i)\bX'_i.  \label{Eq: theta1}
	\end{align}
	First let us calculate the conditional expectation of $\btheta_1$ given $\btheta_0$.	
	Given any $(\xi_0,\bmu^{\bot})$, for any $j\in[1:m]$, let $\bmu_1^{\xi_0,\bmu^{\bot}}:=\bbE[\sgn(\bar{\btheta}_0^\top \bX'_j)\bX'_j|\xi_0,\bmu^{\bot}]$ and $\bbP_{\xi_0,\bmu^{\bot}}$ denotes the probability measure under the parameters $(\xi_0,\bmu^{\bot})$.

	The expectation $\bmu_1^{\xi_0,\bmu^{\bot}}$ can be calculated as follows: 
	\begin{align}
		\bmu_1^{\xi_0,\bmu^{\bot}}&=\bbE[\sgn(\bar{\btheta}_0^\top \bX'_j)\bX'_j|\xi_0,\bmu^{\bot}] \\
		&=\bbE_{Y_j'}[~\bbE[\sgn(\bar{\btheta}_0^\top \bX'_j)\bX'_j ~|~ \xi_0,\bmu^{\bot},Y_j']~]\\
		&=\frac{1}{2}\bbE[\sgn(\bar{\btheta}_0^\top \bX'_j)\bX'_j ~|~\xi_0,\bmu^{\bot}, Y_j'=-1]+\frac{1}{2}\bbE[\sgn(\bar{\btheta}_0^\top \bX'_j)\bX'_j ~|~\xi_0,\bmu^{\bot}, Y_j'=1].
	\end{align}

	In contrast to~\eqref{Eq: Gaussian vec decomp1}, here we decompose the Gaussian random vector $\bg_j\sim\calN(0,\bI_d)$ in another way
	\begin{align}
		\bg_j=\tilg_j \bar{\btheta}_0+\tilde{\bg}_j^{\bot}, \label{Eq:new decomp Gaussian vec}
	\end{align}
	where $\tilg_j\sim\calN(0,1)$,  $\tilde{\bg}_j^{\bot}\sim\calN(0,\bI_d-\bar{\btheta}_0\bar{\btheta}_0^\top)$, $\tilg_j$ and $\tilde{\bg}_j^{\bot}$ are mutually independent and $\tilde{\bg}_j^{\bot} \perp \bar{\btheta}_0$. 
	
	Then we decompose $\bX_j'$ and $\bar{\btheta}_0^\top\bX_j'$ as
	\begin{align}
		\bX_j'&=Y_j'\bmu+\sigma\tilg_j \bar{\btheta}_0+\sigma\tilde{\bg}_j^{\bot}, \text{~and} \label{Eq:X new decomp} \\ 
		\bar{\btheta}_0^\top\bX_j'&=Y_j'\alpha+\sigma \tilg_j. \label{Eq:theta*X new decomp}
	\end{align}		
	Then we have
	\begin{align}
		&\bbE[\sgn(\bar{\btheta}_0^\top \bX'_j)\bX'_j ~|~ \xi_0,\bmu^{\bot}, Y_j'=-1] \nn\\
		&=\bbE[\sgn(-\alpha+\sigma \tilg_j)(-\bmu+\sigma\tilg_j \bar{\btheta}_0+\sigma\tilde{\bg}^{\bot}) ~|~ \xi_0,\bmu^{\bot}]\\
		&=-\bbE[\sgn(-\alpha+\sigma \tilg_j)| \xi_0,\bmu^{\bot}]\bmu+ \sigma\bbE[\sgn(-\alpha+\sigma \tilg_j)\tilg_j|\xi_0,\bmu^{\bot}]\bar{\btheta}_0 \nn\\
		&\quad +\sigma\bbE[\sgn(-\alpha+\sigma \tilg_j)\tilde{\bg}^{\bot}| \xi_0,\bmu^{\bot}]\\
		&=-\bbE[\sgn(-\alpha+\sigma \tilg_j)| \xi_0,\bmu^{\bot}]\bmu+ \sigma\bbE[\sgn(-\alpha+\sigma \tilg_j)\tilg_j|\xi_0,\bmu^{\bot}]\bar{\btheta}_0 \label{Eq:simplify E[theta_1]},
	\end{align}
	where \eqref{Eq:simplify E[theta_1]} follows since $\tilde{\bg}^{\bot}$ is independent of $\tilg_j$ and $\bbE[\tilde{\bg}^{\bot}]=0$.
	
	For the first term in \eqref{Eq:simplify E[theta_1]}, recall $\tilg_j\sim\calN(0,1)$ and we have
	\begin{align}
		&-\bbE[\sgn(-\alpha+\sigma \tilg_j)| \xi_0,\bmu^{\bot}]\bmu=\bigg(1-2\rmQ \bigg(\frac{\alpha}{\sigma}\bigg) \bigg)\bmu \label{Eq:1st term in simplify E[theta_1]}.
	\end{align}
	For the second term in \eqref{Eq:simplify E[theta_1]}, we have
	\begin{align}
		&\bbE[\sgn(-\alpha+\sigma \tilg_j)\tilg_j|\xi_0,\bmu^{\bot}]\bar{\btheta}_0 \nn\\
		&=\bigg(-\int_{-\infty}^{\frac{\alpha}{\sigma}}\frac{1}{\sqrt{2\pi}}e^{-\frac{g^2}{2}}g ~\rmd g\!+\!\int_{\frac{\alpha}{\sigma}}^{\infty}\frac{1}{\sqrt{2\pi}}e^{-\frac{g^2}{2}}g ~\rmd g \bigg)\bar{\btheta}_0
		\!=\!\frac{2}{\sqrt{2\pi}}\exp\bigg(-\frac{\alpha^2}{2\sigma^2}\bigg)\bar{\btheta}_0. \label{Eq:2nd term in simplify E[theta_1]}
	\end{align}
	By combining \eqref{Eq:1st term in simplify E[theta_1]} and \eqref{Eq:2nd term in simplify E[theta_1]}, we have
	\begin{align}
		\bbE\big[\sgn\big(\bar{\btheta}_0^\top \bX'_j\big)\bX'_j ~|~\xi_0,\bmu^{\bot}, Y_j'=-1 \big]  
		&=\bigg(1-2\rmQ \bigg(\frac{\alpha}{\sigma}\bigg) \bigg)\bmu+\frac{2\sigma}{\sqrt{2\pi}}\exp\bigg(-\frac{\alpha^2}{2\sigma^2}\bigg)\bar{\btheta}_0,
	\end{align}
	and similarly, 
	\begin{align}
		\bbE\big[\sgn\big(\bar{\btheta}_0^\top \bX'_j\big)\bX'_j ~|~\xi_0,\bmu^{\bot}, Y_j'=1 \big]  
		&=\bigg(2\rmQ \bigg(-\frac{\alpha}{\sigma}\bigg)-1 \bigg)\bmu+\frac{2\sigma}{\sqrt{2\pi}}\exp\bigg(-\frac{\alpha^2}{2\sigma^2}\bigg)\bar{\btheta}_0.
	\end{align}
	Thus, recall $J_{\sigma}$ and $K_{\sigma}$ defined in \eqref{Eq:J func} and \eqref{Eq:K func}, $\bar{\btheta}_0=\alpha\bmu+\beta \bup$ and $\bmu_1^{\xi_0,\bmu^{\bot}}$ is given by
	\begin{align}
		&\bmu_1^{\xi_0,\bmu^{\bot}}=\bbE[\sgn(\btheta_0^\top \bX'_j)\bX'_j|\xi_0,\bmu^{\bot}] \nn\\
		&=\bigg(1-2\rmQ \bigg(\frac{\alpha}{\sigma}\bigg) \bigg)\bmu+\frac{2\sigma}{\sqrt{2\pi}}\exp\bigg(-\frac{\alpha^2}{2\sigma^2}\bigg)\bar{\btheta}_0\\
		&=\bigg(1-2\rmQ \bigg(\frac{\alpha}{\sigma}\bigg)+\frac{2\sigma \alpha}{\sqrt{2\pi}}\exp\bigg(-\frac{\alpha^2}{2\sigma^2}\bigg) \bigg)\bmu+\frac{2\sigma \beta}{\sqrt{2\pi}}\exp\bigg(-\frac{\alpha^2}{2\sigma^2}\bigg)\bup\\
		&=J_{\sigma}(\alpha(\xi_0,\bmu^{\bot}))\bmu+K_{\sigma}(\alpha(\xi_0,\bmu^{\bot}))\bup. \label{Eq:mu1 decomp}
	\end{align}
	From \eqref{Eq:exact gent}, the gen-error at $t=1$ is given by
	\begin{align}
		\mathrm{gen}_1&=\frac{1}{m}\sum_{i=1}^{m}\bbE_{\xi_0,\bmu^{\bot}}\bbE_{\btheta_1|\xi_0,\bmu^{\bot}}\Big[\HD(P_{\bX_i',\hatY_i'|\xi_0,\bmu^{\bot}}\| P_{\bX_i',\hatY_i'|\xi_0,\bmu^{\bot},\btheta_1}| p_{\btheta_1}) \nn\\
		&\qquad +\HD(P_{\bZ}\| P_{\bX_i',\hatY_i'|\xi_0,\bmu^{\bot}}|p_{\btheta_1}) \Big]. \label{Eq:exact gen1}
	\end{align}
	Next, we calculate the two $\HD$ terms in \eqref{Eq:exact gen1} respectively.
	\begin{itemize}
		\item 	\textbf{Calculate $\bbE_{\btheta_1|\xi_0,\bmu^{\bot}}\big[  \HD(P_{\bX_i',\hatY_i'|\xi_0,\bmu^{\bot}}\| P_{\bX_i',\hatY_i'|\xi_0,\bmu^{\bot},\btheta_1}| p_{\btheta_1})\big]$:}
		\begin{align}
			&\bbE_{\btheta_1|\xi_0,\bmu^{\bot}}\big[  \HD(P_{\bX_i',\hatY_i'|\xi_0,\bmu^{\bot}}\|P_{\bX_i',\hatY_i'|\xi_0,\bmu^{\bot},\btheta_1}| p_{\btheta_1})\big]\nn\\
			&=\bbE_{\btheta_1|\xi_0,\bmu^{\bot}}\Big[h(P_{\bX_i',\hatY_i'|\xi_0,\bmu^{\bot}},p_{\btheta_1})-h(P_{\bX_i',\hatY_i'|\xi_0,\bmu^{\bot},\btheta_1},p_{\btheta_1}) \Big]\\
			&=\frac{1}{2\sigma^2}\int Q_{\btheta_1|\xi_0,\bmu^{\bot}}(\btheta|\xi_0,\bmu^{\bot}) \big(P_{\bX_i',\hatY_i'|\xi_0,\bmu^{\bot}}(\bx,y|\xi_0,\bmu^{\bot}) \nn\\
			&\qquad  \qquad -P_{\bX_i',\hatY_i'|\xi_0,\bmu^{\bot},\btheta_1}(\bx,y|\xi_0,\bmu^{\bot},\btheta) \big)(\bx^\top\bx-2y\btheta^\top\bx+\btheta^\top \btheta) \rmd \bx \rmd y \rmd \btheta \\
			&=-\frac{1}{\sigma^2}\int P_{\bX_i',\hatY_i'|\xi_0,\bmu^{\bot}}(\bx,y|\xi_0,\bmu^{\bot}) \big(Q_{\btheta_1|\xi_0,\bmu^{\bot}}(\btheta|\xi_0,\bmu^{\bot}) \nn\\*
			&\qquad \qquad  -P_{\btheta_1|\bX_i',\hatY_i',\xi_0,\bmu^{\bot}}(\btheta| \bx,y,\xi_0,\bmu^{\bot}) \big)(y\btheta^\top\bx) \rmd \bx \rmd y \rmd \btheta\\
			&=-\frac{1}{\sigma^2}\int P_{\bX_i',\hatY_i'|\xi_0,\bmu^{\bot}}(\bx,y|\xi_0,\bmu^{\bot}) \bigg(\frac{\bmu_1^{\xi_0,\bmu^{\bot}}-y\bx}{m} \bigg)^\top  (y\bx) \rmd \bx \rmd y \\
			&=\frac{1}{m\sigma^2}\Big(\bbE[\bX_i'^\top\bX_i'|\xi_0,\bmu^{\bot}]-(\bmu_1^{\xi_0,\bmu^{\bot}})^\top\bmu_1^{\xi_0,\bmu^{\bot}}\Big)\\
			&=\frac{d\sigma^2+\bmu^\top\bmu-(\bmu_1^{\xi_0,\bmu^{\bot}})^\top\bmu_1^{\xi_0,\bmu^{\bot}}}{m\sigma^2}\\
			&=\frac{d\sigma^2+1-J_{\sigma}^2(\alpha)-K_{\sigma}^2(\alpha)}{m\sigma^2}.
		\end{align}
		
		\item \textbf{Calculate $\HD(P_{\bZ}\| P_{\bX_i',\hatY_i'|\xi_0,\bmu^{\bot}}|p_{\btheta_1})$:}
		Given any $(\xi_0,\bmu^{\bot})$, in the following, we drop the condition on $\xi_0,\bmu^{\bot}$ for notational simplicity.
		Since $P_{\hatY_i'}(1)=P_{\hatY_i'}(-1)=P_{Y_i'}(1)=P_{Y_i'}(-1)=\frac{1}{2}$ (cf. \eqref{Eq:hatY prob}), we have
		\begin{align}
			&\HD(P_{\bZ}\| P_{\bX_i',\hatY_i'|\xi_0,\bmu^{\bot}}|p_{\btheta_1})=h(P_{\bZ},p_{\btheta_1})-h(P_{\bX_i',\hatY_i'|\xi_0,\bmu^{\bot}},p_{\btheta_1}) \nn\\
			&=\frac{1}{2}\int (P_{\bX|Y=1}(\bx)-P_{\bX_i'|\hatY_i'=1}(\bx))\log\frac{1}{ P_Y(1)p_{\btheta_1}(\bx|1)} \rmd \bx \nn\\
			&\quad +\frac{1}{2}\int (P_{\bX|Y=-1}(\bx)-P_{\bX_i'|\hatY_i'=-1}(\bx))\log\frac{1}{ P_Y(-1)p_{\btheta_1}(\bx|-1)} \rmd \bx \\
			&=\frac{1}{2}\int (P_{\bX|Y=1}(\bx)-P_{\bX_i'|\hatY_i'=1}(\bx))\log\frac{1}{ p_{\btheta_1}(\bx|1)} \rmd \bx \nn\\
			&\quad +\frac{1}{2}\int (P_{\bX|Y=-1}(\bx)-P_{\bX_i'|\hatY_i'=-1}(\bx))\log\frac{1}{ p_{\btheta_1}(\bx|-1)} \rmd \bx. \label{Eq:decomp cross-entropy div}
		\end{align}
		Since given $\btheta_1$, $p_{\btheta_1}(\bx|\cdot)$ is a Gaussian distribution, for any $y\in\{\pm 1\}$, we have
		\begin{align}
			&\frac{1}{2}\int P_{\btheta_1|\xi_0,\bmu^{\bot}}(\btheta) \big(P_{\bX|Y}(\bx|\by)-P_{\bX_i'|\hatY_i'}(\bx|y) \big)\log \frac{1}{ p_{\btheta}(\bx|y)} \rmd \bx  \rmd \btheta \nn\\
			&=\frac{1}{4\sigma^2}\int P_{\btheta_1|\xi_0,\bmu^{\bot}}(\btheta) \big(P_{\bX|Y}(\bx|y)-P_{\bX_i'|\hatY_i'}(\bx|y) \big)\big(\bx^\top\bx-2y\btheta^\top \bx+\btheta^\top\btheta \big) \rmd \bx \rmd \btheta\\
			&=-\frac{1}{2\sigma^2}\int P_{\btheta_1|\xi_0,\bmu^{\bot}}(\btheta) \bigg(\frac{1}{2}P_{\bX|Y}(\bx|y)-\frac{1}{2}P_{\bX_i'|\hatY_i'}(\bx|y) \bigg)\big(y\btheta^\top \bx \big) \rmd \bx \rmd \btheta\\
			&=-\frac{1}{2\sigma^2}(\bmu_1^{\xi_0,\bmu^{\bot}})^\top\big(\bmu-\bmu_1^{\xi_0,\bmu^{\bot}} \big)\\
			&=\frac{J_{\sigma}^2(\alpha)+K_{\sigma}^2(\alpha)-J_{\sigma}(\alpha)}{2\sigma^2}.
		\end{align}
		Thus,
		\begin{align}
			\bbE_{\btheta_1|\xi_0,\bmu^{\bot}}[\HD(P_{\bZ}\| P_{\bX_i',\hatY_i'|\xi_0,\bmu^{\bot}}|p_{\btheta_1})]=\frac{J_{\sigma}^2(\alpha)+K_{\sigma}^2(\alpha)-J_{\sigma}(\alpha)}{\sigma^2}.
		\end{align}
	\end{itemize}
	Finally, the gen-error at $t=1$ can be characterized as follows:
	\begin{align}
		&\mathrm{gen}_1=\bbE_{\xi_0,\bmu^{\bot}}\bigg[\frac{J_{\sigma}^2(\alpha(\xi_0,\bmu^{\bot}))+K_{\sigma}^2(\alpha(\xi_0,\bmu^{\bot}))-J_{\sigma}(\alpha(\xi_0,\bmu^{\bot}))}{\sigma^2} \nn\\
		&\qquad \qquad~ +\frac{d\sigma^2+1-J_{\sigma}^2(\alpha(\xi_0,\bmu^{\bot}))-K_{\sigma}^2(\alpha(\xi_0,\bmu^{\bot}))}{m\sigma^2} \bigg]\\
		&=\bbE_{\xi_0,\bmu^{\bot}}\bigg[\frac{(m-1)(J_{\sigma}^2(\alpha(\xi_0,\bmu^{\bot}))+K_{\sigma}^2(\alpha(\xi_0,\bmu^{\bot})))-mJ_{\sigma}(\alpha(\xi_0,\bmu^{\bot}))+d\sigma^2+1}{m\sigma^2}\bigg].
	\end{align}
	\item \textbf{Pseudo-label using $\btheta_1$:} Let $\bar{\btheta}_1:=\btheta_1/\|\btheta_1\|_2$.  For any $i\in\calI_2$, the pseudo-labels are given by
	\begin{align}
		\hatY_{i}'=\sgn(\btheta_1^\top \bX'_i)=\sgn(\bar{\btheta}_1^\top \bX'_i).
	\end{align}
	It can be seen that the pseudo-labels $\{\hatY_{i}'\}_{i\in\calI_2}$ are conditionally i.i.d.\ given $\btheta_1$ and let us denote the conditional   distribution under fixed $\btheta_1$ as $P_{\hatY'|\btheta_1}\in\calP(\calY)$. The pseudo-labelled dataset is denoted as $\hatS_{\rmu,2}=\{(\bX'_i,\hatY_{i}')\}_{i\in\calI_2}$. 
	
	For any fixed $(\btheta_1,\xi_0,\bmu^{\bot})$, let $\btheta_1$ be decomposed as $\btheta_1=A_1(\xi_0,\bmu^{\bot})\bmu+B_1(\xi_0,\bmu^{\bot})\bup$, where $A_1^2(\xi_0,\bmu^{\bot})+B_1^2(\xi_0,\bmu^{\bot})=\|\btheta_1\|_2^2$. In addition, let $\alpha_1(\xi_0,\bmu^{\bot}):=A_1/\sqrt{A_1^2+B_1^2}$ and $\beta_1(\xi_0,\bmu^{\bot})=\sqrt{1-(\alpha_1(\xi_0,\bmu^{\bot}))^2}$.
	 In the following, we use $A_1$, $B_1$, $\alpha_1$, and $\beta_1$ for the above quantities if there is no risk of confusion.
	
	%For any fixed $\bar{\btheta}_1\in\Theta$, we can decompose it as $\bar{\btheta}_1=\alpha'_1 \bmu+\beta'_1 \bup$, where $\alpha'_1\in[-1,1]$ and $\beta'_1=\sqrt{1-(\alpha'_1)^2}$. 
	Recall the decomposition of $\bX_i'$ and $\bar{\btheta}_0^\top\bX_i'$ in \eqref{Eq:X_i decomp} and \eqref{Eq:theta_0X_i decomp}. Similarly, we have
	\begin{align}
		&\bar{\btheta}_1^\top\bX'_i=:Y_i'\alpha_1+\sigma h_i^1, \label{Eq:theta_1*X}
	\end{align}
	where $h_i^1 \sim \calN(0,1)$. 
	Note that $P_{\hatY_i'|\btheta_1,\xi_0,\bmu^{\bot}}=P_{\hatY_i'|\btheta_1}$ and then the conditional probability $P_{\hatY_i'|\btheta_1,\xi_0,\bmu^{\bot}}$ can be given by
	\begin{align}
		P_{\hatY_i'|\btheta_1,\xi_0,\bmu^{\bot}}(1)&=P_{\hatY_i'|\btheta_1}(1) 
		=\bbP_{\btheta_1}\big(\bar{\btheta}_1^\top\bX_i'>0\big)\\
		&=\frac{1}{2}\bbP_{\btheta_1}\big(\alpha_1+\sigma h_i^1>0 \big)+\frac{1}{2}\bbP_{\btheta_1}\big(\alpha_1+\sigma h_i^1\leq 0 \big)
		=\frac{1}{2}, \label{Eq:P_hatY 2 prob}
	\end{align}
	and $P_{\hatY_i'|\btheta_1,\xi_0,\bmu^{\bot}}(-1)=1/2$, where $\bbP_{\btheta_1}$ denotes the probability measure under parameter $\btheta_1$.

	\begin{figure}[t!]
		\centering
		%	\begin{minipage}[t]{0.33\linewidth}
			%		\centering
			%		\includegraphics[width=1\textwidth]{Fsig_t}
			%		\caption{$F_{\sigma}^{(t)}(x)$ versus $x$ for different $t$ when $\sigma=0.5$.}
			%		\label{Fig: Fsig}
			%	\end{minipage}
		%	\hspace{-4pt}
		\begin{minipage}[t]{0.4\linewidth}
			\centering
			\includegraphics[width=1\textwidth]{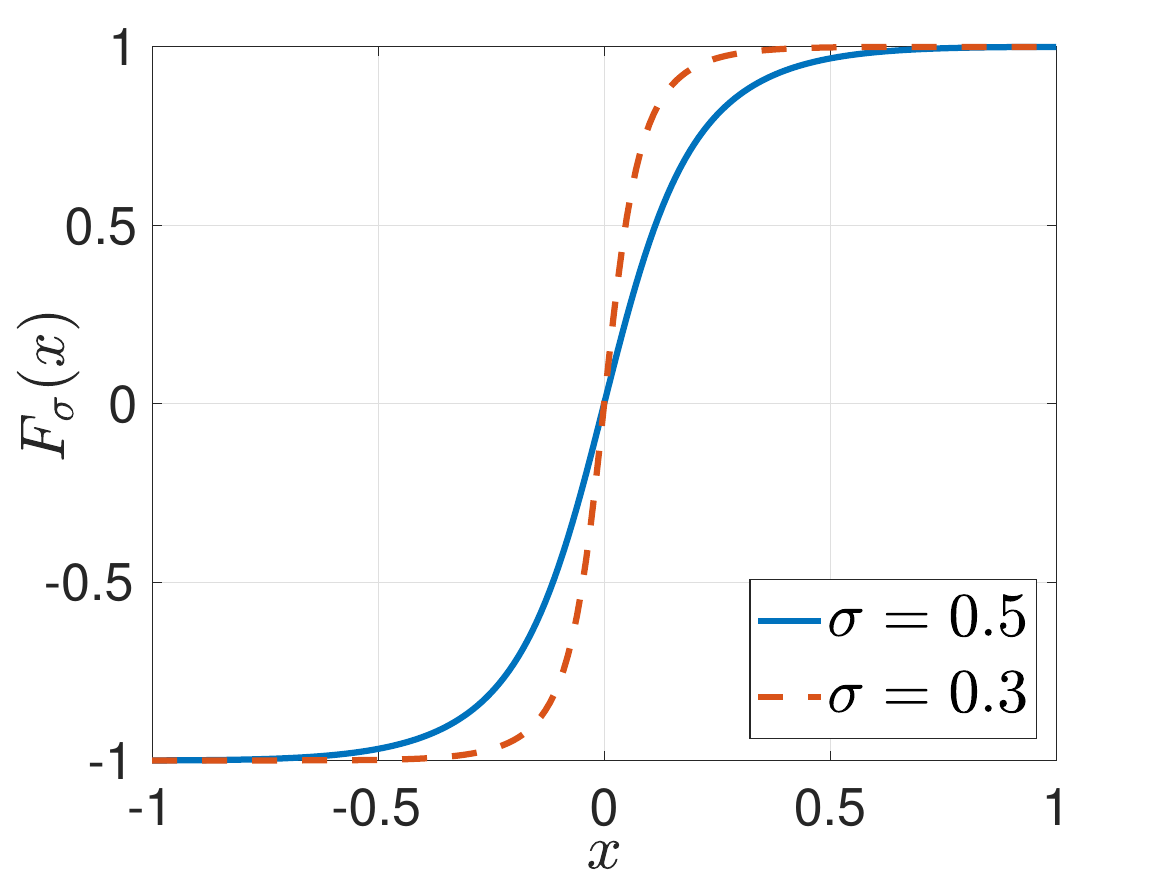}
			\caption{$F_{\sigma}(x)$ versus  $x$ for $\sigma\in\{0.3, 0.5\}$.}
			\label{Fig: Fsig_sig}
		\end{minipage}
		\hspace{5pt}
		\begin{minipage}[t]{0.4\linewidth}
			\centering
			\includegraphics[width=1\textwidth]{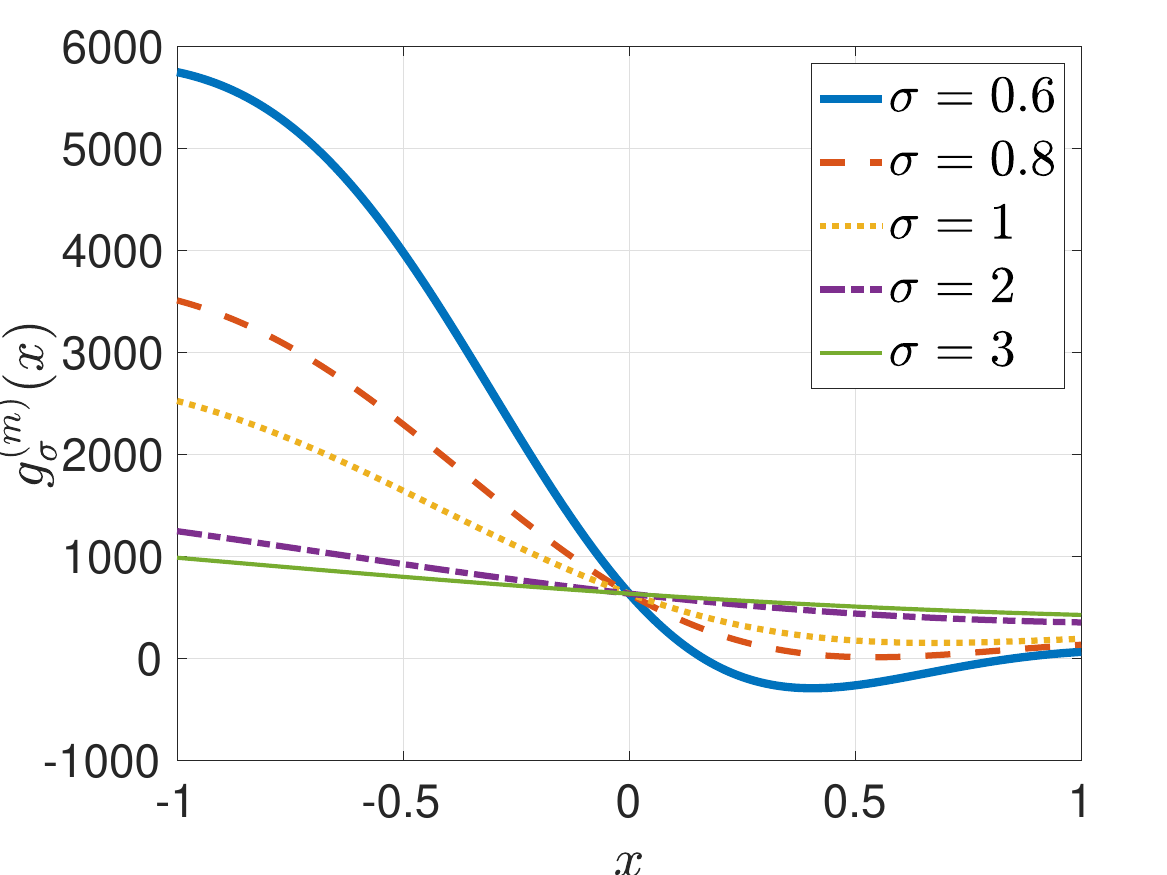}
			\caption{$g_{\sigma}^{(m)}(x)$ versus  $x$ for $\sigma\in\{0.6,0.8,1,2,3\}$.}
			\label{Fig: g_sig_m}
		\end{minipage}
	\end{figure}
	\item \textbf{Iteration $t=2$:} Recall \eqref{Eq:theta_t} and the new model parameter learned from the pseudo-labelled dataset $\hatS_{\rmu,2}$ is given by
	\begin{align}
		\btheta_2&=\frac{1}{m}\sum_{i\in\calI_2} \hatY_i'\bX'_i=\frac{1}{m}\sum_{i\in\calI_2}\sgn(\bar{\btheta}_1^\top \bX'_i)\bX'_i, \label{Eq:theta2}
	\end{align}
	where $\{ \sgn(\bar{\btheta}_1^\top \bX'_i)\bX'_i\}_{i\in\calI_2}$ are conditionally i.i.d.\ random variables given $\btheta_1,\xi_0,\bmu^{\bot}$.

	Similar to \eqref{Eq:mu1 decomp}, the expectation of $\btheta_2$ conditioned on $\btheta_1,\xi_0,\bmu^{\bot}$ is given by
	\begin{align}
		&\bmu_2^{\btheta_1,\xi_0,\bmu^{\bot}}:=\bbE[\btheta_2|\btheta_1,\xi_0,\bmu^{\bot}]=\bbE[\sgn(\bar{\btheta}_1^\top \bX'_j)\bX'_j|\btheta_1,\xi_0,\bmu^{\bot}] \nn\\
		&=\bigg(1-2\rmQ \bigg(\frac{\alpha_1}{\sigma}\bigg)+\frac{2\sigma \alpha_1}{\sqrt{2\pi}}\exp\bigg(-\frac{\alpha_1^2}{2\sigma^2}\bigg) \bigg)\bmu+\frac{2\sigma \beta_1}{\sqrt{2\pi}}\exp\bigg(-\frac{\alpha_1^2}{2\sigma^2}\bigg)\bup\\
		&=J_{\sigma}(\alpha_1(\xi_0,\bmu^{\bot}))\bmu+K_{\sigma}(\alpha_1(\xi_0,\bmu^{\bot}))\bup. \label{Eq:mu2 decomp}
	\end{align}
	As $m\to\infty$, by the strong law of large numbers, $\btheta_1|\xi_0,\bmu^{\bot} \to \bmu_1^{\xi_0,\bmu^{\bot}}$ almost surely. By the continuous mapping theorem, we also have $\alpha_1(\xi_0,\bmu^{\bot})\to F_{\sigma}(\alpha(\xi_0,\bmu^{\bot}))$ almost surely. Equivalently, for almost all sample paths,  there exists a vanishing sequence $\epsilon_m$ (i.e., $\epsilon_m\to 0$ as $m\to\infty$) such that $|\alpha_1(\xi_0,\bmu^{\bot})-F_{\sigma}(\alpha(\xi_0,\bmu^{\bot}))|=\epsilon_m$.
	
	The gen-error at $t=2$ is given by
	\begin{align}
		&\mathrm{gen}_2=\frac{1}{m}\sum_{i\in\calI_2}\bbE_{\btheta_1,\xi_0,\bmu^{\bot}}\Big[ \bbE_{\btheta_2|\btheta_1,\xi_0,\bmu^{\bot}}\big[ \HD(P_{\bZ}\| P_{\bX_i',\hatY_i'|\btheta_1,\xi_0,\bmu^{\bot}}|p_{\btheta_2}) \nn\\
		&\qquad \qquad + \HD(P_{\bX_i',\hatY_i'|\btheta_1,\xi_0,\bmu^{\bot}}\| P_{\bX_i',\hatY_i'|\btheta_2,\btheta_1,\xi_0,\bmu^{\bot}} | p_{\btheta_2}) \big] \Big].
	\end{align}
	
	By applying the same techniques in the iteration $t=1$, we obtain the exact characterization of gen-error at $t=2$ as follows. By the uniform continuity of $J_{\sigma}$ , for any vanishing sequence $\epsilon_m>0$, there exist $\epsilon_m', \epsilon_m''>0$ such that  $\sup_{x\in[0,1]}|J_{\sigma}(x+\epsilon_m)-J_{\sigma}(x)|=\epsilon_m'$ and $\sup_{x\in[0,1]}|J_{\sigma}(x-\epsilon_m)-J_{\sigma}(x)|=\epsilon_m''$, where $\epsilon_m',\epsilon_m''\downarrow 0$ as $\epsilon_m\downarrow 0$.
	The same result holds for $K_{\sigma}$.
	
	Finally we can obtain the gen-error as follows. For almost all sample paths, there exists a vanishing sequence $\epsilon_m$ (i.e., $\epsilon_m\to 0$ as $m\to\infty$), such that
	\begin{align}
		 \mathrm{gen}_2&=\bbE_{\xi_0,\bmu^{\bot}}\bigg[\frac{J_{\sigma}^2(F_{\sigma}(\alpha))+K_{\sigma}^2(F_{\sigma}(\alpha))-J_{\sigma}(F_{\sigma}(\alpha))}{\sigma^2} \nn\\
		&\qquad \qquad~ +\frac{d\sigma^2+1-J_{\sigma}^2(F_{\sigma}(\alpha))-K_{\sigma}^2(F_{\sigma}(\alpha))}{m\sigma^2} \bigg]+\epsilon_m\\
		&=\bbE_{\xi_0,\bmu^{\bot}}\bigg[\frac{(m-1)(J_{\sigma}^2(F_{\sigma}(\alpha))+K_{\sigma}^2(F_{\sigma}(\alpha)))-mJ_{\sigma}(F_{\sigma}(\alpha))}{m\sigma^2}\bigg]+\epsilon_m',
	\end{align}
	where $\epsilon_m'=\epsilon_m+\frac{d\sigma^2+1}{m\sigma^2}$ and $\alpha$ stands for $\alpha(\xi_0,\bmu^{\bot})$.
	
	\item \textbf{Iteration $t\in[2:\tau]$:} By iteratively implementing the calculation, we finally obtain the characterization of $\mathrm{gen}_t$ as follows. For almost all sample paths, there exists a vanishing sequence $\epsilon_{m,t}$ ($\epsilon_{m,t}\to 0$ as $m\to\infty$) , such that
	\begin{align}
		&\mathrm{gen}_t=\bbE_{\xi_0,\bmu^{\bot}}\bigg[\frac{(m-1)(J_{\sigma}^2(F_{\sigma}^{(t-1)}(\alpha))+K_{\sigma}^2(F_{\sigma}^{(t-1)}(\alpha)))-mJ_{\sigma}(F_{\sigma}^{(t-1)}(\alpha))}{m\sigma^2}\bigg]+\epsilon_{m,t}',
	\end{align}
	where $\epsilon_{m,t}'=\epsilon_{m,t}+\frac{d\sigma^2+1}{m\sigma^2}$ and $\alpha$ stands for $\alpha(\xi_0,\bmu^{\bot})$.
	
	The proof is thus completed.
\end{itemize}

\begin{remark}[Numerical plots of $F_{\sigma}^{(t)}(\cdot)$ and $g_{\sigma}^{(m)}(\cdot)$]
Recall $g_{\sigma}^{(m)}(x)=((m-1)(J_{\sigma}^2(x)+K_{\sigma}^2(x))-mJ_{\sigma}(x))/\sigma^2$ for any $x\in[-1,1]$, which is the quantity that determines the behaviour of \eqref{Eq:exact gent bGMM}.
To gain more insight, we numerically plot $F_{\sigma}^{(t)}(x)$ for $t=0,1,2$ in Figure \ref{Fig: Fsig} and $F_{\sigma}(x)$ under different $\sigma$ in Figure \ref{Fig: Fsig_sig}. We also plot $g_{\sigma}^{(m)}(x)$ under different $\sigma$ in Figure \ref{Fig: g_sig_m}.
%\begin{figure}[h!]
%	\centering
%%	\begin{minipage}[t]{0.33\linewidth}
%%		\centering
%%		\includegraphics[width=1\textwidth]{Fsig_t}
%%		\caption{$F_{\sigma}^{(t)}(x)$ versus $x$ for different $t$ when $\sigma=0.5$.}
%%		\label{Fig: Fsig}
%%	\end{minipage}
%%	\hspace{-4pt}
%	\begin{minipage}[t]{0.33\linewidth}
%		\centering
%		\includegraphics[width=1\textwidth]{Fsig_sig}
%		\caption{$F_{\sigma}(x)$ versus $x$ for $\sigma\in\{0.3, 0.5\}$.}
%		\label{Fig: Fsig_sig}
%	\end{minipage}
%	\hspace{-4pt}
%	\begin{minipage}[t]{0.33\linewidth}
%		\centering
%		\includegraphics[width=1\textwidth]{g_sig_m}
%		\caption{$g_{\sigma}^m(x)$ versus $x$ for $\sigma\in\{0.6,0.8,1,2,3\}$.}
%		\label{Fig: g_sig_m}
%	\end{minipage}
%\end{figure}
\end{remark}

\section{Applying Theorem \ref{Coro: sub Gau gen} to bGMMs}\label{pf of Thm:gen bound GMM}

In anticipation of leveraging Theorem~\ref{Coro: sub Gau gen} together with the sub-Gaussianity of the loss function for the bGMM to derive generalization bounds in terms of information-theoretic quantities (just as in  \citet{russo2016controlling,xu2017information, bu2020tightening}), we find it convenient to show that  
%To invoke concentration results for the loss function, we find it convenient to show that it is sub-Gaussian with high probability. To this end, it is easy to see that 
$\bX$ and $l(\btheta,(\bX,Y))$ are bounded w.h.p.. By defining the $\ell_\infty$ ball $\calB_r^y:=\{\bx \in\bbR^d: \|\bx-y \bmu\|_\infty \leq r \}$, we see that %$\Pr (\bX\in\calB_r^Y)= (1-2\Phi(\frac{-r}{\sigma} ))^d=:1-\delta_{r,d}$
\vspace{-0.6em}
\begin{align}
	\Pr (\bX\in\calB_r^Y)= \bigg(1-2\Phi\Big(-\frac{ r}{\sigma} \Big)\bigg)^d=:1-\delta_{r,d},
\end{align}
where $\Phi(\cdot)$ is the Gaussian cumulative distribution function.
By choosing $r$ appropriately, the failure probability $\delta_{r,d}$ can be made arbitrarily small.

To show that $\btheta$ is bounded with high probability, define the set $\Theta_{\bmu,c}:=\{\btheta\in\Theta: \|\btheta-\bmu\|_\infty\leq c\}$ for some $c>0$.
For any $\btheta\in\Theta_{\bmu,c}$, we have
\begin{align}
	\min_{(\bx,y)\in\calZ}  l(\btheta,(\bx,y))&= \log (2\sqrt{(2\pi)^d}\sigma^d ) =:c_1, \text{~~~and} \\
	   	\max_{\bx\in\calB_r^y, y\in\calY}  l(\btheta,(\bx,y))& \le  \log (2\sqrt{(2\pi)^d}\sigma^d)  +  \frac{d (c+r)^2}{2\sigma^2}=:c_2.
\end{align} 
For any $(\bX,Y)$ from the bGMM($\bmu,\sigma$) and any $\btheta\in\Theta_{\bmu,c}$, the probability that $l(\btheta,(\bX,Y))$ belongs to the interval $[c_1,c_2]$ ($c_1$, $c_2$ depend on $\delta_{r,d}$) can be lower bounded by
\vspace{-0.6em}
\begin{align}
	&\Pr\big(l(\btheta,(\bX,Y))\in[c_1,c_2] \big) 
	%	&=\Pr\big(l(\btheta,(\bX,Y))\leq c_2 \big)\\
	%	&\geq \Pr (\bX\in\calB_r^Y)\\
	\geq 1-\delta_{r,d}.
\end{align}

Thus, according to Hoeffding's lemma, with probability at least $1-\delta_{r,d}$, $l(\btheta,(\bX,Y))\sim \mathrm{subG}((c_2-c_1)/2)$ under $(\bX,Y)\sim  \calN(Y\bmu,\sigma^2 \bI_d)\otimes P_Y$ for all $\btheta\in\Theta_{\bmu,c}$, i.e., for all $\lambda\in\bbR$,
\begin{align}
	&\bbE_{\bX,Y}\big[\exp\big(\lambda\big(l(\btheta,(\bX,Y))-\bbE_{\bX,Y}[l(\btheta,(\bX,Y))]\big) \big)\big] \nn\\
	&\leq \exp\big(\lambda^2 (c_2-c_1)^2/8 \big). \label{Eq:loss sub-Gaussian}
\end{align}

Recall the definition of $\alpha$ in \eqref{Eq:rho_0} and the decomposition of $\bar{\btheta}_0$ in \eqref{Eq:theta0 decomp mu up}.
Define the {\em KL-divergence  between the pseudo-labelled data distribution and the
	true data distribution after the first iteration} $G_\sigma: [-1,1]\times\bbR\times\bbR^d\to [0,\infty)$, as follows:
\begin{align}
	G_{\sigma}(\alpha,\xi_0,\bmu^{\bot})  :=  D\big( p'_{\tilg,\tilde{\bg}^{\bot}} \big\| p_{\tilg} \otimes p_{\tilde{\bg}^{\bot}} \big), \label{Eq:G_sigma}
\end{align}
where
\vspace{-1em}
\begin{align}
	p'_{\tilg,\tilde{\bg}^{\bot}}=&\Phi\Big(-\frac{\alpha}{\sigma}\Big)p_{\tilg+\frac{2\alpha}{\sigma}|\tilg\leq -\frac{\alpha}{\sigma}}\otimes p_{\tilde{\bg}^{\bot}+\bar{\btheta}_0^{\bot}} +\Phi\Big(\frac{\alpha}{\sigma}\Big)p_{\tilg|\tilg\leq \frac{\alpha}{\sigma}} \otimes p_{\tilde{\bg}^{\bot}},
\end{align}
$\tilg\sim\calN(0,1)$, $\tilde{\bg}^{\bot}\sim\calN(0,\bI_d-\bar{\btheta}_0\bar{\btheta}_0^\top)$, $\tilde{\bg}^{\bot}$ is independent of $\tilg$ and perpendicular to $\bar{\btheta}_0$. Note that $p_{\tilg+\frac{2\alpha}{\sigma}|\tilg\leq -\frac{\alpha}{\sigma}}$ is the Gaussian probability density function with mean $\frac{2\alpha}{\sigma}$ and variance $1$ {\em truncated} to   the interval $(-\infty,-\frac{ \alpha}{\sigma})$, and similarly for $p_{\tilg|\tilg\leq \frac{\alpha}{\sigma}}$.
In general, when $G_{\sigma}(\alpha,\xi_0,\bmu^{\bot})$ is small, so is the generalization error.

By applying the result in Theorem \ref{Coro: sub Gau gen}, the following theorem provides an upper bound for the generalization error at each iteration $t$ for $m$ large enough. 
%Let $\delta_{m,c-\tilc_1,d}:=2d\exp(-\frac{m(\veps-\tilc_2)^2}{2\sigma^2} )$. 
\begin{theorem}\label{Thm:gen bound GMM}
	Fix any $\sigma\in\bbR_+$, $d\in\bbN$, $\epsilon\in\bbR_+$ and $\delta\in(0,1)$.
	With probability at least $1-\delta$, 
	the absolute generalization error at $t=0$ can be upper bounded as follows
	\begin{align}
		&    \big|\mathrm{gen}_0(P_\bZ, P_\bX, P_{\btheta_0|S_{\rml},S_{\rmu}} ) \big| \leq \sqrt{ \frac{(c_2-c_1)^2 d}{4}  \log\frac{n}{n-1} }. \label{Eq:gen_0 bound}
	\end{align}
	For each $t\in[1:\tau]$, for $m$ large enough, with probability at least $1-\delta$,  
	\begin{align}
		&\big|\mathrm{gen}_t(P_\bZ, P_\bX, \{P_{\btheta_k|S_{\rml},S_{\rmu}}\}_{k=0}^t, \{f_{\btheta_k}\}_{k=0}^{t-1} ) \big| \nn\\
		&   \leq  \frac{c_2-c_1}{\sqrt{2}}\bbE_{\xi_0,\bmu^{\bot}}\bigg[ \sqrt{G_{\sigma}\big(F_\sigma^{(t-1)}(\alpha(\xi_0,\bmu^{\bot})),\xi_0,\bmu^{\bot}\big) + \epsilon} \bigg]. \label{Eq:upper bd gent}
	\end{align}
	%where $\alpha(\xi_0,\bmu^{\bot})=\rho(\btheta_0,\bmu)$ in \eqref{Eq:rho_0}, $\xi_0\sim\calN(0,1)$, $\bmu^{\bot}\sim\calN(0,\bI_d-\bmu\bmu^\top)$ is a random vector perpendicular to $\bmu$ and independent of $\xi_0$.
\end{theorem}
The proof of Theorem \ref{Thm:gen bound GMM} is provided in Appendix \ref{pf: proof of thm: gen bound gmm}.
Several remarks about Theorem \ref{Thm:gen bound GMM} are as follows.

\begin{figure}[!t]
	%	\setlength{\belowcaptionskip}{-0.4cm}  
	%	\subcapraggedrighttrue
%	\subcapcenterlasttrue
	\centering
	\begin{minipage}{1\linewidth}
		\centering
		\vspace{-1em}
		\subfigure[$\sigma=0.6$]{
			\includegraphics[width=0.4\linewidth, trim=5 0 10 0,clip]{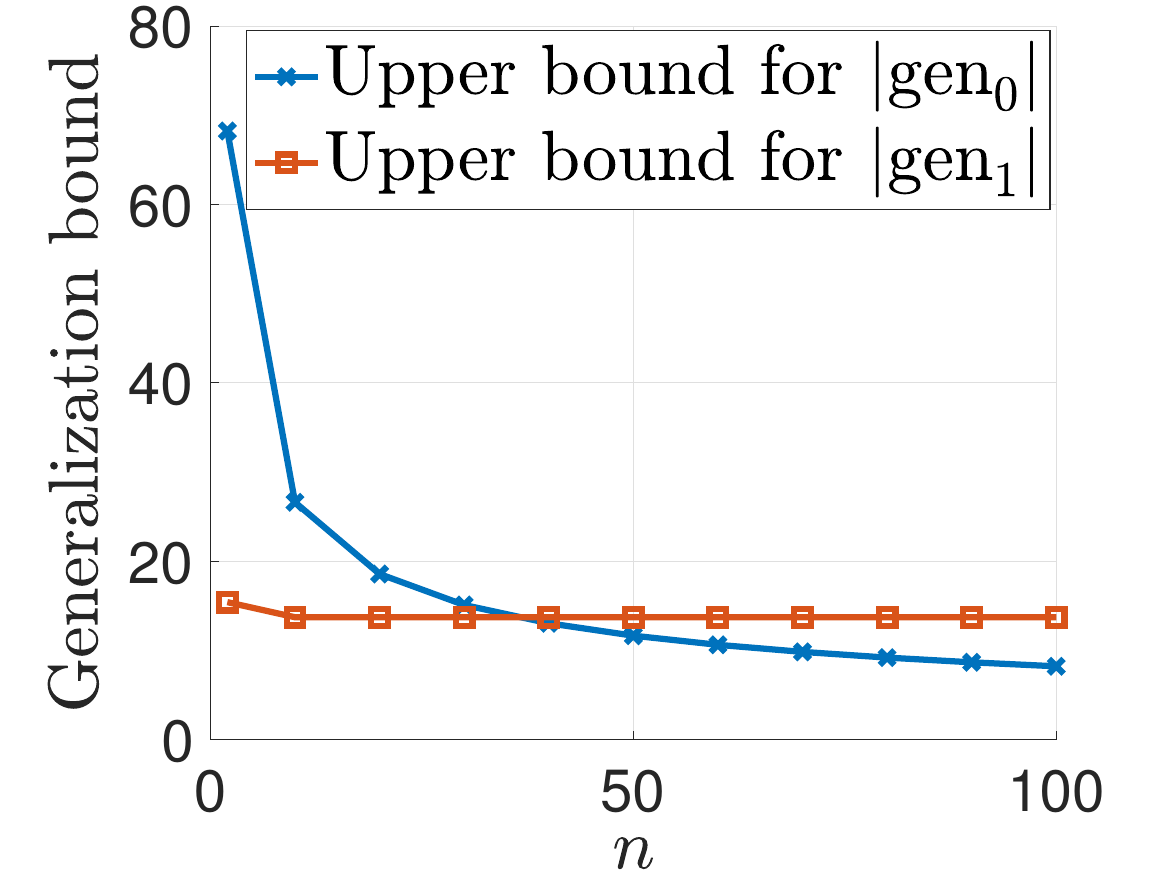}
			\label{Fig: gen_vs_sig 0.6}}
		\subfigure[$\sigma=3$]{
			\includegraphics[width=0.4\linewidth, trim=10 0 10 0,clip]{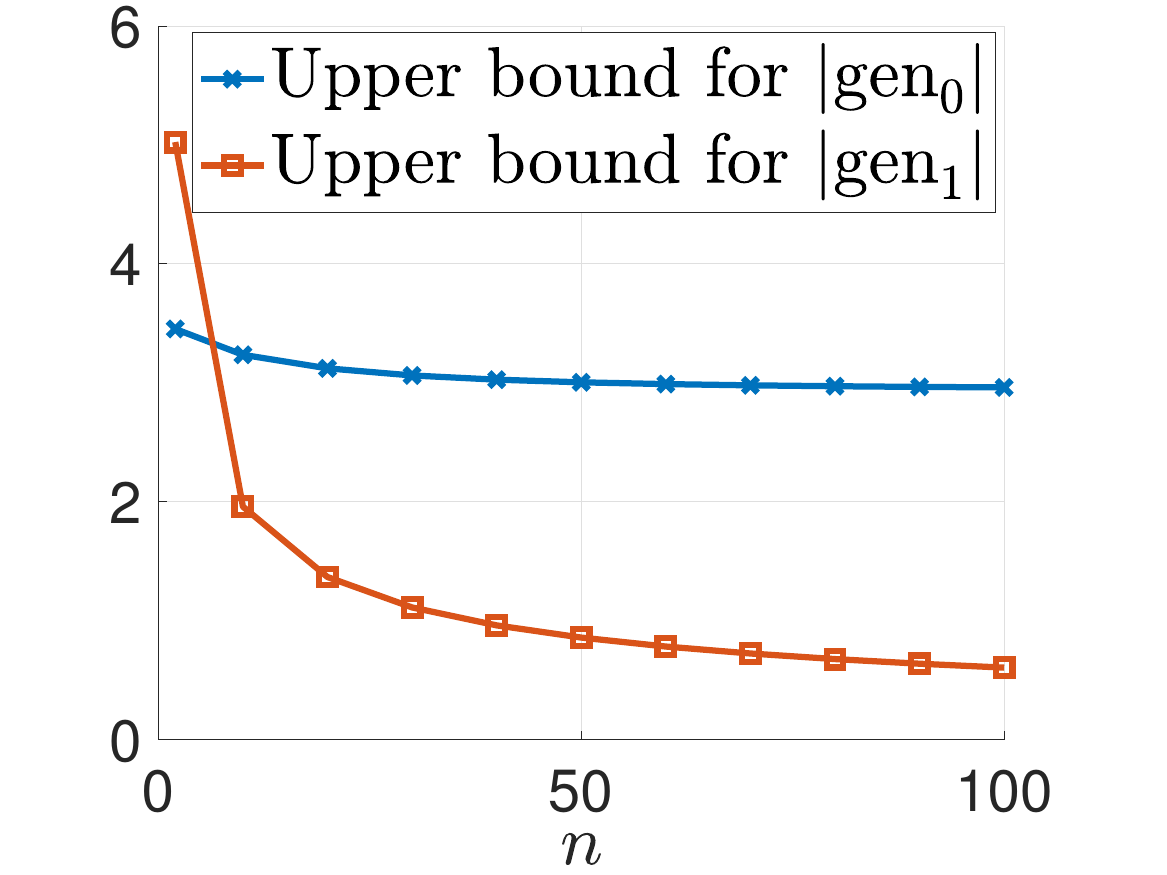}
			\label{Fig: gen_vs_sig 0.7}}
		\caption{Upper bounds for generalization error at $t=0$ and $t=1$ under different $\sigma$ when $d=2$ and $\bmu=(1,0)$. }
	\end{minipage}
%	\hspace{-10pt}
	\begin{minipage}{1\linewidth}
		\centering
		\subfigure[$\sigma=0.6$]{
			\includegraphics[width=0.4\linewidth, trim=5 0 10 0,clip]{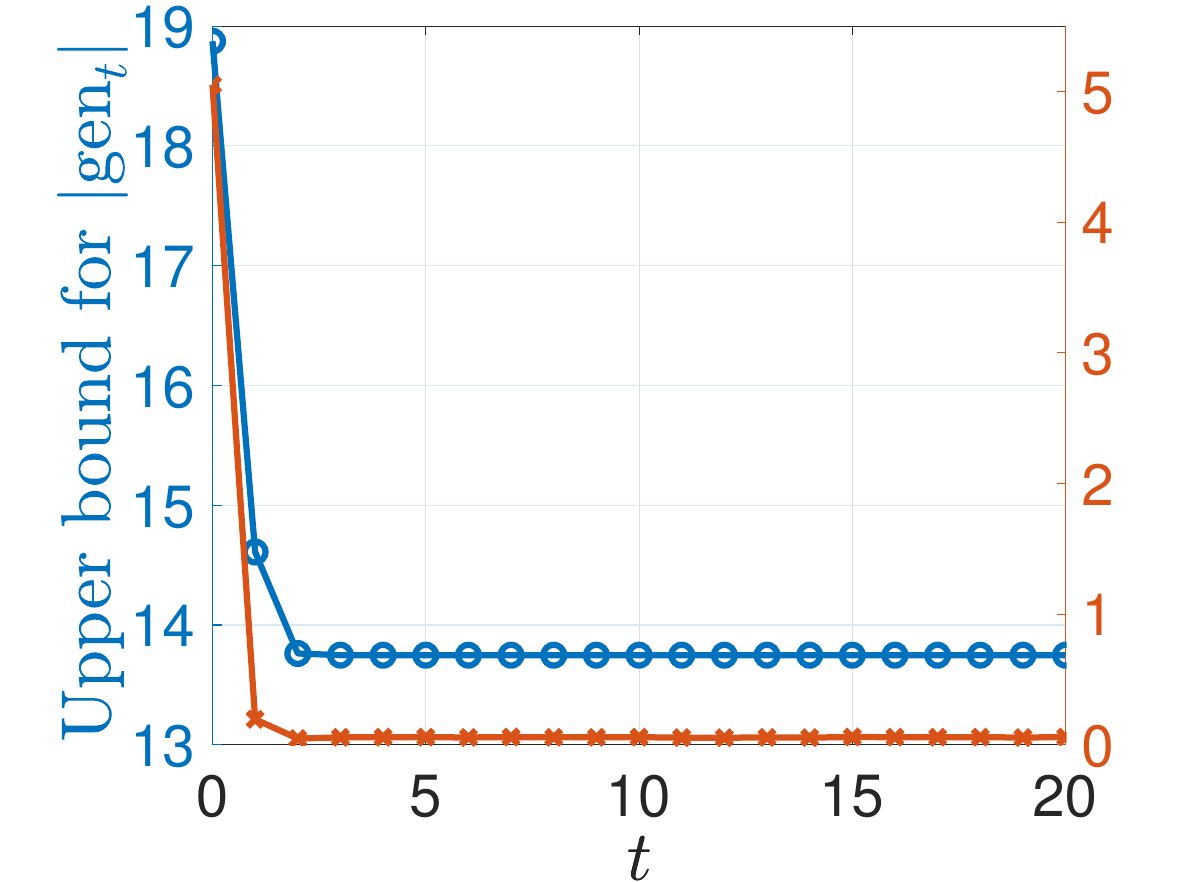}
			\label{Fig:gen bd and empirical}}
		\subfigure[$\sigma=3$]{
			\includegraphics[width=0.4\linewidth, trim=5 0 10 0,clip]{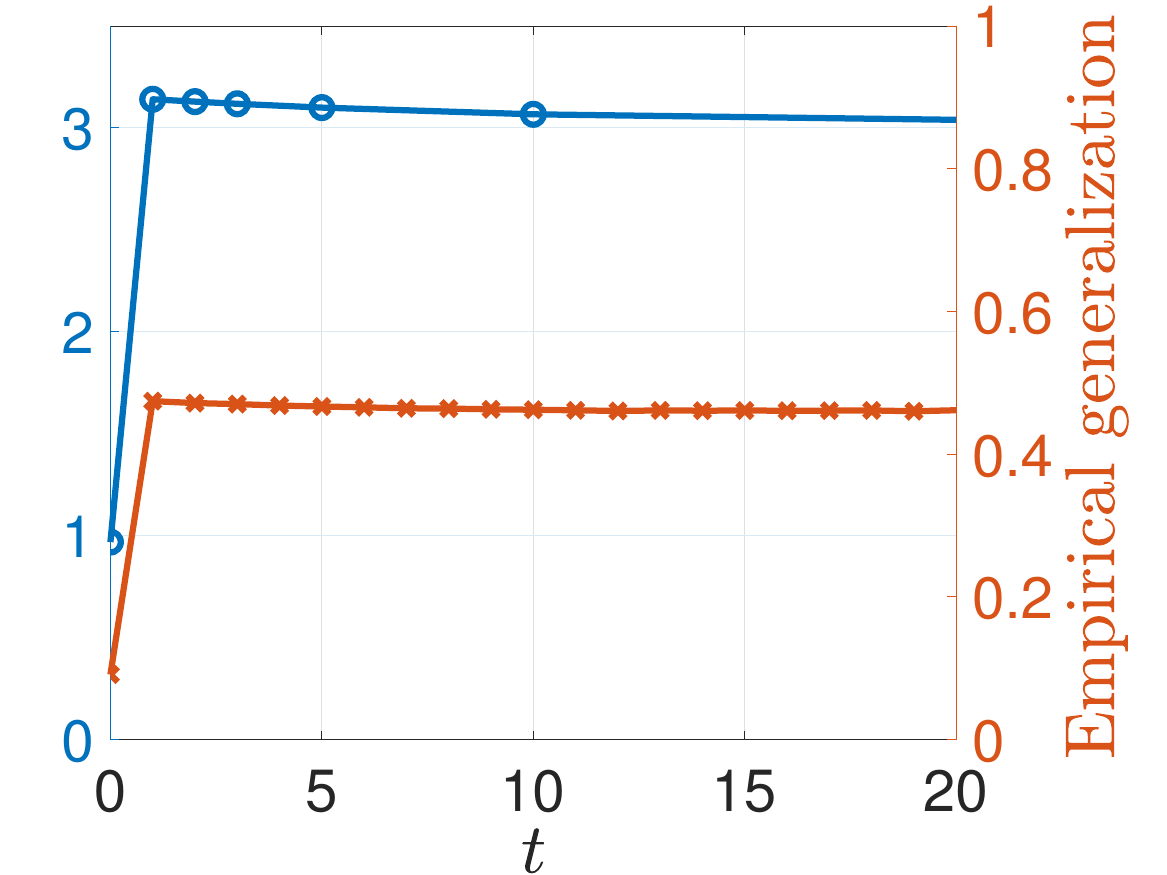}
			\label{Fig:gen bd and empirical_big_sig}}
		
		\caption{The comparison between the upper bound for $|\mathrm{gen}_t|$ and the empirical generalization error at each iteration $t$. The upper bounds are both for $d=2$.}
	\end{minipage}
%	\vspace{-0.5cm}
\end{figure}
\begin{figure}[!t]
	%			\vspace{-10pt}
	\centering
	\includegraphics[width=.43\linewidth]{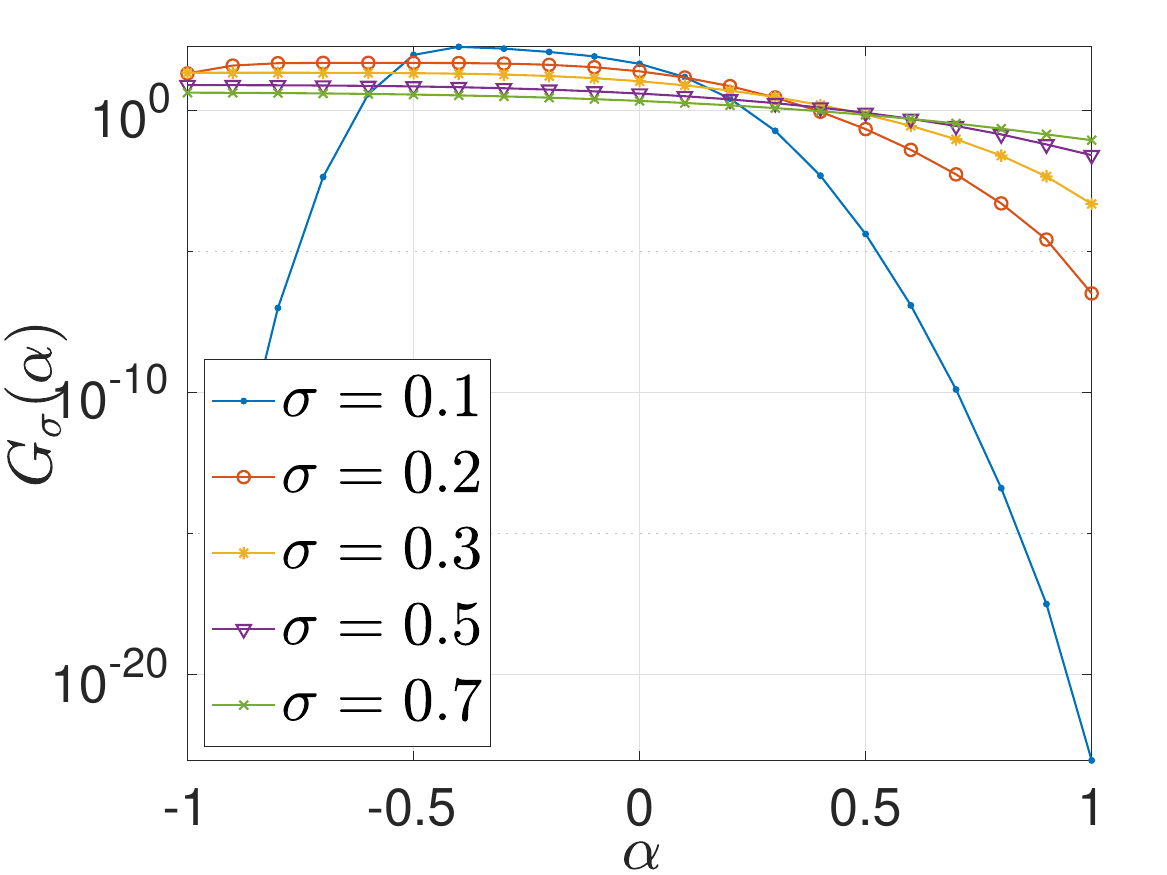}
	\caption{$G_{\sigma}(\alpha)$ versus  $\alpha$ for different $\sigma$.}
	\label{Fig:H_alp}
\end{figure}

First, we compare the upper bounds for $|\mathrm{gen}_0|$ and $|\mathrm{gen}_1|$, as shown in Figures~\ref{Fig: gen_vs_sig 0.6} and~\ref{Fig: gen_vs_sig 0.7}. For any fixed $\sigma$, when $n$ is sufficiently small, the upper bound for $|\mathrm{gen}_0|$ is greater than that for $|\mathrm{gen}_1|$. As $n$ increases, the upper bound for $|\mathrm{gen}_1|$ surpasses that of $|\mathrm{gen}_0|$, as shown in Figure \ref{Fig: gen_vs_sig 0.7}. This is consistent with the intuition that when the labelled data is limited, using the unlabelled data can help improve the generalization performance. However, as the number of labelled data increases, using the unlabelled data may degrade the generalization performance, if the distributions corresponding to classes $+1$ and $-1$ have a large overlap.  This is because the labelled data is already effective in learning the unknown parameter $\btheta_t$ well and additional pseudo-labelled data does not help to further boost the generalization performance. Furthermore, by comparing Figures \ref{Fig: gen_vs_sig 0.6} and \ref{Fig: gen_vs_sig 0.7}, we can see that for smaller $\sigma$, the improvement from $|\mathrm{gen}_0|$ to $|\mathrm{gen}_1|$ is more pronounced. The intuition is that when $\sigma$ decreases, the data samples have smaller variance and thus the pseudo-labelling is more accurate. In this case, unlabelled data can improve the generalization performance. Let us examine the effect of $n$, the number of labelled training samples. Recall the Taylor expansion of $\alpha$ in \eqref{Eq:alpha taylor}.
%\begin{align}
%	\alpha%=\frac{1}{\sqrt{1+\frac{\frac{\sigma^2}{n}\|\bmu^{\bot}\|_2^2}{(1+\frac{\sigma}{\sqrt{n}}\xi_0)^2}}}
%	=1-\frac{\sigma^2}{2n}\|\bmu^{\bot}\|_2^2+o\bigg(\frac{1}{n}\bigg). 
%\end{align}
It can be seen that as $n$ increases, $\alpha$ converges to $1$ in probability. Suppose the dimension $d=2$ and $\bmu=(1,0)$. Then $\bmu^{\bot}=[0,\mu_2^{\bot}]$ where $\mu_2^{\bot}\sim\calN(0,1)$. The upper bound for the absolute generalization error at $t=1$ can be rewritten as
\vspace{-0.6em}
\begin{align}
	  |\mathrm{gen}_1| 
	%		\leq \sqrt{\frac{(c_2-c_1)^2}{2} }~\bbE_{\mu_2^{\bot}}\bigg[\sqrt{ G_{\sigma}\bigg(1-\frac{\sigma^2(\mu_2^{\bot})^2}{2n}+o\bigg(\frac{1}{n}\bigg) \bigg)} ~\bigg]\\
	%	&\approx\sqrt{\frac{(c_2-c_1)^2}{2}}~\int_{-\infty}^{\infty}\frac{1}{\sqrt{2\pi}}e^{-\frac{x^2}{2}} \sqrt{ G_{\sigma}\bigg(1-\frac{\sigma^2 x^2}{2n} \bigg)} ~\rmd x\\
	\lessapprox\frac{c_2-c_1}{\sqrt{2}}   \int_{-\sqrt{2}}^{\sqrt{2}}\sqrt{\frac{n}{\pi\sigma^2} }\, e^{-\frac{ny^2}{\sigma^2}} \,   \sqrt{ G_{\sigma}(1-y^2)} ~\rmd y, 
\end{align}
which is a decreasing function of $n$, as shown in Figures~\ref{Fig: gen_vs_sig 0.6} and~\ref{Fig: gen_vs_sig 0.7}.

Second, in Figure \ref{Fig:gen bd and empirical}, we plot the theoretical upper bound in \eqref{Eq:upper bd gent} by ignoring~$\epsilon$. Unfortunately it is computationally difficult to numerically calculate the bound in \eqref{Eq:upper bd gent} for high dimensions $d$ (due to the need for high-dimensional numerical integration), but we can still gain insight from the result for $d=2$.
It is shown that the upper bound for $|\mathrm{gen}_{t}|$ decreases as $t$ increases and finally converges to a non-zero constant. The gap between the upper bounds for $|\mathrm{gen}_t|$ and for $|\mathrm{gen}_{t+1}|$ decreases as $t$ increases and shrinks to almost 0 for $t\geq 2$. The intuition is that as $m\to\infty$, there are sufficient data at each iteration and the algorithm can converge at very early stage. In the empirical simulation, we let $n=10$, $d=50$, $\bmu=(1,0,\ldots,0)$ and iteratively run the self-training procedure for $20$ iterations and $2000$ rounds. We find that the behaviour of the empirical generalization error (the red `-x' line) is similar to the theoretical upper bound (the blue `-o' line), which almost converges to its final value at $t=2$.
This result shows that the theoretical upper bound in \eqref{Eq:upper bd gent} serves as a useful rule-of-thumb for how the generalization error changes over iterations. In Figure \ref{Fig:gen bd and empirical_big_sig}, we plot the theoretical bound and result from the empirical simulation based on the toy example for $d=2$ but larger $n$ and $\sigma$. This figure shows that when we increase $n$ and $\sigma$, using unlabelled data may not be able to improve the generalization performance. The intuition is that for $n$ large enough, merely using the labelled data can yield sufficiently low generalization error and for subsequent iterations with the pseudo-labelled data, the reduction in the test loss is negligible but the training loss will decrease more significantly (thus causing the generalization error to increase). When $\sigma$ is larger, the data samples have larger variance and the classes have a  larger overlap, and thus,   the initial parameter $\btheta_0$ learned by the labelled data cannot produce pseudo-labels with sufficiently high accuracy. Thus, the pseudo-labelled data cannot help to improve the generalization error significantly.

\begin{remark}[Numerical plot of $G_{\sigma}(\cdot)$]
	To gain more insight, we   plot  $G_{\sigma}(\alpha,\xi_0,\bmu^{\bot})$ when $d=2$ and $\bmu=(1,0)$ in Figure~\ref{Fig:H_alp}. Under these settings, $G_{\sigma}(\alpha,\xi_0,\bmu^{\bot})$ depends only on $\alpha$ and hence, we can rewrite it as $G_\sigma(\alpha)$. As shown in Figure~\ref{Fig:H_alp} in Appendix \ref{pf of Thm:gen bound GMM}, for all $\sigma_1>\sigma_2$,  there exists an $\alpha_0\in[-1,1]$ such that for all $\alpha\geq \alpha_0= \alpha_0( \sigma_1,\sigma_2)$, $G_{\sigma_1}(\alpha) >G_{\sigma_2}(\alpha) $. From \eqref{Eq:rho_0}, we can see that $\alpha$ is close to 1 of high probability, which means that $\sigma\mapsto G_\sigma(\alpha)$ is monotonically increasing in $\sigma$ with high probability. As a result, $\bbE_{\alpha}[\sqrt{G_{\sigma}(\alpha)}]$ increases as $\sigma$ increases.  This is consistent with the intuition that when the training data has larger in-class variance, it is more difficult to generalize well.
\end{remark}

	\begin{remark}[Discussion on Theorem \ref{Thm:gen bound GMM}] \label{rmk:minus_one}
		As $n\to\infty$, $\btheta_0\to \bmu$ and $\alpha=\rho(\btheta_0,\bmu)\to 1$ almost surely, which means that the estimator converges to the optimal classifier for this bGMM. However, since there is no margin between two groups of data samples, the error probability $\Pr(\hatY_j'\neq Y_j')\to\rmQ (1/\sigma)>0$ (which is the Bayes error rate) and the disintegrated KL-divergence $D_{\xi_0,\bmu^{\bot}}(P_{\bX_j',\hatY_j'} \| P_{\bX,Y})$ between the estimated and underlying distributions cannot converge to $0$.  
		
		In the other extreme case, when $\alpha=\rho(\btheta_0,\bmu)=-1$ and $\bar{\btheta}_0=-\bmu$, the error probability $\Pr(\hatY_j'\neq Y_j')=1-\rmQ (1/\sigma)>\frac{1}{2}$  (for all $\sigma>0$) and $D_{\xi_0,\bmu^{\bot}}(P_{\bX_j',\hatY_j'} \| P_{\bX,Y})<\infty$, so in this other extreme (flipped) scenario, we have more mistakes than correct pseudo-labels. The reason why $D_{\alpha,\bmu^{\bot}}(P_{\bX_j',\hatY_j'} \| P_{\bX,Y})$ is finite is that when $P_{\bX,Y}(\bx,y)$ is  small, it means that $\bx$ is far from both $-\bmu$ and $\bmu$, and then $P_\bX(\bx)$ is also  small. Thus, $P_{\bX_j',\hatY_j'}(\bx,y)=P_{\hatY_j'|\bX_j'}(y|\bx)P_\bX(\bx)$ is also small.
	\end{remark}

\section{Proof of Theorem \ref{Thm:gen bound GMM}}\label{pf: proof of thm: gen bound gmm}
For simplicity, in the following, we abbreviate $\mathrm{gen}_t(P_\bZ, P_\bX, \{P_{\btheta_k|S_{\rml},S_{\rmu}}\}_{k=0}^t, \{f_{\btheta_k}\}_{k=0}^{t-1} )$ as $\mathrm{gen}_t$.
\begin{enumerate}[leftmargin= 15 pt]
	\item \textbf{Initial round $t=0$:} Since $Y_i\bX_i\overset{\text{i.i.d.}}{\sim} \calN(\bmu,\sigma^2 \bI_d)$, we have $\btheta_0\sim \calN(\bmu,\frac{\sigma^2}{n} \bI_d)$ and for some constant $c\in\bbR_+$,
	\begin{align}
		\Pr(\btheta_0 \in \Theta_{\bmu,c})= \Pr(\|\btheta_0-\bmu\|_\infty \leq c)=\bigg(1-2\Phi\Big(-\frac{\sqrt{n}c}{\sigma} \Big)\bigg)^d=:1-\delta_{\sqrt{n}c,d}. \label{Eq:theta_0 prob}
	\end{align}
	By choosing $c$ large enough, $\delta_{\sqrt{n}c,d}$ can be made arbitrarily small.
	Consider $\tilde{\btheta}_0$ and $(\tilde{\bX},\tilY)$ as independent copies of $\btheta_0 \sim Q_{\btheta_0}$ and $(\bX,Y)\sim P_{\bX,Y}=P_Y\otimes \calN(Y\bmu,\sigma^2 I_d)$, respectively, such that $P_{\tilde{\theta}_0,\tilde{\bX},\tilY}=Q_{\btheta_0}\otimes P_{\bX,Y}$. Then the probability that  $l(\btheta_0,(\bX,Y))\sim \mathrm{subG}((c_2-c_1)/2)$  under $(\bX,Y)\sim P_{\bX,Y}$ is given as follows
	\begin{align}
		&\Pr\bigg(\Lambda_{l(\tilde{\btheta}_0,(\tilde{\bX},\tilY))}(\lambda,\tilde{\btheta}_0)\leq \frac{\lambda^2(c_2-c_1)^2}{8} \bigg)\\
		&\geq \Pr\bigg(\Lambda_{l(\tilde{\btheta}_0,(\tilde{\bX},\tilY))}(\lambda,\tilde{\btheta}_0)\leq \frac{\lambda^2(c_2-c_1)^2}{8} \text{ and } \tilde{\btheta}_0\in\Theta_{\bmu,c} \bigg)\\
		&=\Pr(\tilde{\btheta}_0\in\Theta_{\bmu,c}) \Pr\bigg(\Lambda_{l(\tilde{\btheta}_0,(\tilde{\bX},\tilY))}(\lambda,\tilde{\btheta}_0)\leq \frac{\lambda^2(c_2-c_1)^2}{8} \Big| \tilde{\btheta}_0\in\Theta_{\bmu,c} \bigg)\\
		&=(1-\delta_{\sqrt{n}c,d})(1-\delta_{r,d}), \label{Eq:l(theta_0) sub-G}
	\end{align}
	where \eqref{Eq:l(theta_0) sub-G} follows from \eqref{Eq:loss sub-Gaussian} and \eqref{Eq:theta_0 prob}.
	
	Fix some $d\in\bbN$, $\epsilon>0$ and $\delta\in(0,1)$. There exists $n_0(d,\delta)\in\bbN$,  $r_0(d,\delta)\in\bbR_+$ such that for all $n>n_0, r>r_0$, $\delta_{\sqrt{n}c,d}<\frac{\delta}{3}$, $\delta_{r,d}<\frac{\delta}{3}$, and then with probability at least $1-\delta$, the absolute generalization error can be upper bounded as follows
	\begin{align}
		&|\mathrm{gen}_0| \leq \frac{1}{n}\sum_{i=1}^{n}\sqrt{\frac{(c_2-c_1)^2}{2}I(\btheta_0;\bX_i,Y_i) }.
	\end{align} 
	
	%	we first use the labelled dataset $S_{\rml}$ to learn the model parameter $\btheta_0$ that minimizes the following empirical risk
	%	\begin{align}
		%		L_{S_{\rml}}(\btheta)&=\frac{1}{n}\sum_{i=1}^n l(\btheta,(\bX_i,Y_i))\\
		%		&=\frac{1}{n}\sum_{i=1}^n\log\Big(2\sqrt{(2\pi)^d}\sigma^d \Big)+\frac{1}{2\sigma^2 n}\sum_{i=1}^n(\bX_i-Y_i\btheta)^\top(\bX_i-Y_i\btheta).
		%	\end{align}	
	%	Since $L_{S_{\rml}}(\btheta)$ is a convex function of $\btheta$, by differentiating $L_{S_{\rml}}(\btheta)$ over $\btheta$, we can get the minimizer
	%	\begin{align}
		%		\btheta_0=\frac{1}{n}\sum_{i=1}^n Y_i \bX_i.
		%	\end{align}
	Then mutual information can be calculated as follows 
	\begin{align}
		I(\btheta_0;\bX_i,Y_i) 
		&=h(\btheta_0)-h(\btheta_0|\bX_i,Y_i)\\
		&=h\bigg(\frac{1}{n}\sum_{j=1}^n Y_j \bX_j\bigg)-h\bigg(\frac{1}{n}\sum_{j=1}^n Y_j \bX_j \bigg| \bX_i,Y_i \bigg)\\
		%		&=h\bigg(\frac{1}{n}\sum_{j=1}^n Y_j \bX_j\bigg)-h\bigg(\frac{1}{n-1}\sum_{j\in[n], j\ne i} Y_j \bX_j \bigg)+\log\bigg(\frac{n-1}{n} \bigg)\\
		%		&\overset{n\to\infty}{\longrightarrow} 0  \quad \text{(\cite[Theorem 3.1]{piera2009convergence})}
		&=\frac{d}{2}\log\bigg(\frac{2\pi e \sigma^2}{n} \bigg)-h\bigg(\frac{1}{n}\sum_{j\in[n], j\ne i} Y_j \bX_j \bigg)\\
		&=\frac{d}{2}\log\bigg(\frac{2\pi e \sigma^2}{n} \bigg)-\frac{d}{2}\log\bigg(\frac{2\pi e (n-1)\sigma^2 }{n^2} \bigg)\\
		&=\frac{d}{2}\log\frac{n}{n-1}.
	\end{align} 
	Thus we obtain~\eqref{Eq:gen_0 bound}.

	\item \textbf{Pseudo-label using $\btheta_0$:} The same as those in Appendix \ref{pf of Thm:exact gen GMM}.

	\item \textbf{Iteration $t=1$:} Recall $\btheta_1$ in \eqref{Eq: theta1}, the new model parameter learned from the pseudo-labelled dataset $\hatS_{\rmu,1}$.
	
	\begin{enumerate}[leftmargin= 5 pt]
		\item Recall the condition expectation of $\bmu_1^{\xi_0,\bmu^{\bot}}$ in \eqref{Eq:mu1 decomp}. 
		The $l_\infty$ norm between $\bmu_1^{\xi_0,\bmu^{\bot}}$ and $\bmu$ can be upper bounded by
		\begin{align}
			\|\bmu_1^{\xi_0,\bmu^{\bot}}-\bmu \|_{\infty} 
			&\leq \sqrt{\bigg(-2\rmQ \bigg(\frac{\alpha}{\sigma}\bigg)+\frac{2\sigma \alpha}{\sqrt{2\pi}}\exp\bigg(-\frac{\alpha^2}{2\sigma^2}\bigg) \bigg)^2+\frac{2\sigma^2 \beta^2}{\pi}\exp\bigg(-\frac{2\alpha^2}{2\sigma^2}\bigg)}\\
			&< \sqrt{\bigg(2\Phi\bigg(\frac{1}{\sigma}\bigg)+\frac{2\sigma}{\sqrt{2\pi}}\bigg)^2+\frac{2\sigma^2}{\pi}}=:\tilc_1, \label{Eq:mu1 mu distance}
		\end{align}
		where $\tilc_1$ is a constant only dependent on $\sigma$.
		
		\item Next, we need to calculate the probability that $l(\btheta_1,(\bX,Y))\sim \mathrm{subG}((c_2-c_1/2))$ under $(\bX,Y)\sim P_{\bX,Y}$. 
		
		Let $\bV_i=\sgn(\bar{\btheta}_0^\top\bX_i')\bX_i'-\bmu_1^{\xi_0,\bmu^{\bot}}$. For any $k\in[d]$, let $V_{i,k}$, $\theta_{1,k}$, $\mu_{1,k}$ denote the $k$-th components of $\bV_i$, $\btheta_1$ and  $\bmu_{1}^{\xi_0,\bmu^{\bot}}$, respectively. 
		Recall the decompositions $\bX_i'=Y_i'\bmu+\sigma \tilg_i\bar{\btheta}_0+\sigma \tilde{\bg}_i^{\bot}$ in \eqref{Eq:X new decomp} and $\bar{\btheta}_0\bX_i'=Y_i'\alpha+ \sigma\tilg_i$ in~\eqref{Eq:theta*X new decomp}. Suppose the basis of $\bbR^d$ is denoted by $B=\{\bv_1,\ldots,\bv_d\}$ and let $\bv_1=\bar{\btheta}_0$. Then we have
		\begin{align}
			&\tilde{\bg}_i^{\bot}=\tilg_{i,2}^{\bot}\bv_2+\ldots+\tilg_{i,d}^{\bot}\bv_d,
		\end{align}
		where $\tilg_{i,k}^{\bot}\sim \calN(0,1)$ for any $k\in[2:d]$ and $\{\tilg_{i,k}\}_{k=2}^d$ are mutually independent. We also let $\bmu=(\mu_{0,1},\ldots,\mu_{0,d})$.
		
		Given any $(\xi_0,\bmu^{\bot})$,
		%		 we can deduce that 
		%		\begin{align}
			%			P_{\bV_i}=P_{\bX_i'-\bmu_1^{\xi_0,\bmu^{\bot}}|\btheta_0^\top\bX_i'>0}.
			%		\end{align} 
		the moment generating function (MGF) of $V_{i,1}$ is given as follows: for any $s_1>0$,
		\begin{align}
			&\bbE_{V_{i,1}}[e^{s_1 V_{i,1}}]\nn\\
			&=\rmQ\bigg(-\frac{\alpha}{\sigma}\bigg)\bbE_{ \tilg_i}\Big[e^{s_1(\mu_{0,1}-\mu_{1,1}+\sigma \tilg_i)}\Big|\tilg_i>-\frac{\alpha}{\sigma} \Big] +\rmQ\bigg(\frac{\alpha}{\sigma}\bigg)\bbE_{ \tilg_i}\Big[e^{s_1(-\mu_{0,1}-\mu_{1,1}+\sigma \tilg_i)}\Big|\tilg_i>\frac{\alpha}{\sigma} \Big] \\
			&= e^{s_1(\mu_{0,1}-\mu_{1,1})}e^{\frac{\sigma^2 s_1^2}{2}}\Phi\bigg(\frac{\alpha}{\sigma}+\sigma s_1 \bigg) +e^{s_1(-\mu_{0,1}-\mu_{1,1})}e^{\frac{\sigma^2 s_1^2}{2}}\Phi\bigg(-\frac{\alpha}{\sigma}+\sigma s_1\bigg).
			%			&\leq  e^{\frac{\sigma^2 s_1^2}{2}+\tilc_2 s_1},
		\end{align}
		The final equality follows from the fact that the MGF of a zero-mean univariate Gaussian truncated to $(a,b)$ is $e^{\sigma^2 s^2/2} \Big[ \frac{\Phi(b-\sigma s)- \Phi(a-\sigma s)}{\Phi(b)-\Phi(a)} \Big]$. 
		%		where $\tilc_2=1+\sqrt{(1+2\Phi(1/\sigma)+2\sigma/\sqrt{2\pi})^2+2\sigma^2/\pi}$ is a constant only dependent on $\sigma$.
		The second derivative of $\log \bbE_{V_{i,1}}[e^{s_1 V_{i,1}}]$ is given as
		\begin{align}
			&\tilR_1(s_1):=\frac{\rmd^2 \log\bbE_{V_{i,1}}[e^{s_1 V_{i,1}}]}{\rmd s_1^2}\\
			&\leq \sigma^2+\frac{\mathrm{const.}}{\big(\Phi(\frac{\alpha}{\sigma}+\sigma s_1 )e^{s_k\mu_{0,k}}+\Phi(\frac{-\alpha}{\sigma}+\sigma s_1 )e^{-s_k\mu_{0,k}}\big)^2}<\infty.
		\end{align}		
		For $k\in[2:d]$ and any $s_k>0$, the MGF of $V_{i,k}$ is given as
		\begin{align}
			&\bbE_{V_{i,k}}[e^{s_k V_{i,k}}]=\bbE_{\sigma\tilg_{i,k}^{\bot},Y_i'}\Big[e^{s_k(Y_i'\mu_{0,k}-\mu_{1,k}+\sigma \tilg_{i,k}^{\bot})}\Big]\\
			&=Q\bigg(-\frac{\alpha}{\sigma}\bigg)e^{s_k(\mu_{0,k}-\mu_{1,k})}e^{\frac{\sigma^2 s_k^2}{2}}+Q\bigg(\frac{\alpha}{\sigma}\bigg)e^{s_k(-\mu_{0,k}-\mu_{1,k})}e^{\frac{\sigma^2 s_k^2}{2}},
			%			&\leq e^{\frac{\sigma^2 s_k^2}{2}+\tilc_2 s_k}.
		\end{align}
		and the second derivative of $\log \bbE_{V_{i,k}}[e^{s_k V_{i,k}}]$ is given by
		\begin{align}
			\tilR_k(s_k):=\frac{\rmd^2 \log\bbE_{V_{i,k}}[e^{s_k V_{i,k}}]}{\rmd s_k^2}=\sigma^2+\frac{4\mu_{0,k}^2Q(-\frac{\alpha}{\sigma})Q(\frac{-\alpha}{\sigma})}{(Q(-\frac{\alpha}{\sigma})e^{s_k\mu_{0,k}}+Q(\frac{\alpha}{\sigma})e^{-s_k\mu_{0,k}})^2}.
		\end{align}
		
		Fix $k \in [1:d]$. According to   Taylor's theorem, we have
		\begin{align}
			\log \bbE_{V_{i,k}}[e^{s_k V_{i,k}}]=\frac{\tilR_k(\xi_{\rmL,k})}{2}s_k^2,
		\end{align} 
		for some  $\xi_{\rmL,k}\in (0,s_k) $ and  $\tilR_k(\xi_{\rmL,k})<\infty$.
		Then the Cram\'{e}r transform of $\log\bbE_{V_{i,k}}[e^{s_k V_{i,k}}]$ can be lower bounded as follows: for any $\veps>0$,
		\begin{align}
			&\sup_{s_k>0}\Big(s_k\veps - \log\bbE_{V_{i,k}}[e^{s_k V_{i,k}}] \Big)
			\geq \sup_{s_k>0}\Big(s_k\veps - \frac{\tilR_k(\xi_{\rmL,k}) s_k^2}{2} \Big)
			=\frac{\veps^2}{2\tilR_k(\xi_{\rmL,k})}. \label{Eq:cramer transform V_i,k}
		\end{align}

		%		Let $\Lambda_{\theta_{1,k}-\mu_{1,k}|\xi_0,\bmu^{\bot}}(s)=\log\bbE_{\theta_{1,k}}[\exp(s(\theta_{1,k}-\mu_{1,k}))|\xi_0,\bmu^{\bot}]$, i.e., the cumulant generating function of $\theta_{1,k}-\mu_{1,k}$ given  $\xi_0,\bmu^{\bot}$. According to \cite[Lemma 2.4]{boucheron2013concentration}, the Cram\'{e}r transform of  $\Lambda_{\theta_{1,k}-\mu_{1,k}|\xi_0,\bmu^{\bot}}(s)$, i.e.,  $\sup_{s>0}(s\veps-\Lambda_{\theta_{1,k}|\xi_0,\bmu^{\bot}}(s) )$, is always positive.
		
		Let $\tilR^*=\max_{\xi_0, \bmu^{\bot}}\min_{k\in[1:d]}\tilR_k(\xi_{\rmL,k})$, which is a finite constant only dependent on $\sigma$. Since $\{\sgn(\bar{\btheta}_0^\top \bX_i')\bX_i'\}_{i=1}^m$ are i.i.d.\ random variables conditioned on $(\xi_0, \bmu^{\bot})$, by applying Chernoff-Cram\'{e}r inequality, we have for all $\veps> 0$
		\begin{align}
			&\bbP_{\xi_0,\bmu^{\bot}}\Big(\|\btheta_1-\bmu_1^{\xi_0,\bmu^{\bot}}\|_{\infty} > \veps \Big) \nn\\
			&=\bbP_{\xi_0,\bmu^{\bot}}\Big(\max_{k\in[1:d]}|\theta_{1,k}-\mu_{1,k}|>\veps \Big)\\
			&\leq \sum_{k=1}^d\bbP_{\xi_0,\bmu^{\bot}}\Big(|\theta_{1,k}-\mu_{1,k}|>\veps \Big)\\
			&= \sum_{k=1}^d\bbP_{\xi_0,\bmu^{\bot}}\bigg(\bigg|\frac{1}{m}\sum_{i=1}^m V_{i,k} \bigg|>\veps \bigg)\\
			&\leq \sum_{k=1}^d 2\exp\Big(-m\sup_{s>0}\Big(s\veps-\log\bbE_{V_{i,k}}[e^{s V_{i,k}}] \Big) \Big)\\
			&\leq 2d\exp\bigg(-\frac{m \veps^2}{2\tilR^*} \bigg)\\
			&=:\delta_{m,\veps,d}, \label{Eq:theta1 concentrate ineq}
		\end{align}
		where $\delta_{m,\veps,d}\xrightarrow{\text{a.s.}}0$ as $m\to\infty$ and does not depend on $\xi_0,\bmu^{\bot}$. 
		
		Choose some $c\in(\tilc_1,\infty)$ ($\tilc_1$ defined in \eqref{Eq:mu1 mu distance}). We have
		\begin{align}
			\bbP_{\xi_0,\bmu^{\bot}}\big(\btheta_1\in\Theta_{\bmu,c})\geq \bbP_{\xi_0,\bmu^{\bot}}(\|\btheta_1-\bmu_1^{\xi_0,\bmu^{\bot}}\|_{\infty} \leq c-\tilc_1 \big)\geq 1-\delta_{m,c-\tilc_1,d}.
		\end{align}
		
		Consider $\tilde{\btheta}_1$ as an independent copy of $\btheta_1$ and   independent of $(\tilde{\bX},\tilY)$. Then the probability that  $l(\btheta_1,(\bX,Y))\sim \mathrm{subG}((c_2-c_1)/2)$  under $(\bX,Y)\sim P_{\bX,Y}$ is given as follows
		\begin{align}
			&\bbP_{\xi_0,\bmu^{\bot}}\bigg(\Lambda_{l(\tilde{\btheta}_1,(\tilde{\bX},\tilY))}(\lambda,\tilde{\btheta}_1)\leq \frac{\lambda^2(c_2-c_1)^2}{8} \bigg)\\
			%			&\geq \Pr\bigg(\Lambda_{l(\tilde{\btheta}_0,(\tilde{\bX},\tilY))}(\lambda,\tilde{\btheta}_0)\leq \frac{\lambda^2(c_2-c_1)^2}{8} \text{ and } \tilde{\btheta}_0\in\Theta_{\bmu,c} \bigg)\\
			&\geq \bbP_{\xi_0,\bmu^{\bot}}(\tilde{\btheta}_1\in\Theta_{\bmu,c}) \bbP_{\xi_0,\bmu^{\bot}}\bigg(\Lambda_{l(\tilde{\btheta}_1,(\tilde{\bX},\tilY))}(\lambda,\tilde{\btheta}_1)\leq \frac{\lambda^2(c_2-c_1)^2}{8} \Big| \tilde{\btheta}_1\in\Theta_{\bmu,c} \bigg)\\
			&=(1-\delta_{m,c,d})(1-\delta_{r,d}).
		\end{align}

		%		by the strong law of large numbers, we have
		%		\begin{align}
			%			\btheta_1|\alpha,\bmu^{\bot} \xrightarrow{\text{a.s.}} \bmu_1^{\alpha,\bmu^{\bot}}, \quad \text{as } m\to\infty. \label{Eq:theta_1 converge}
			%		\end{align}
	\end{enumerate}

	Thus, for some $c\in(\tilc_1,\infty)$, with probability at least $(1-\delta_{m,c-\tilc_1,d})(1-\delta_{r,d})$, the absolute generalization error can be upper bounded as follows:
	\begin{align}
		&|\mathrm{gen}_1|=|\bbE[L_{P_{\bZ}}(\btheta_1)-L_{\hatS_{\rmu,1}}(\btheta_1)]|\\
		&=\bigg|\frac{1}{m}\sum_{i=1}^m \bbE_{\xi_0,\bmu^{\bot}}\bigg[\bbE\left[l(\btheta_1,(\bX,Y))-l(\btheta_1,(\bX'_i,\hatY_i'))|\xi_0,\bmu^{\bot} \right]\bigg] \bigg|\\
		&\leq \frac{1}{m}\sum_{i=1}^m \bbE_{\xi_0,\bmu^{\bot}}\bigg[\sqrt{\frac{(c_2-c_1)^2}{2}\Big(I_{\xi_0,\bmu^{\bot}}(\btheta_1,(\bX'_i,\hatY_i'))+D_{\xi_0,\bmu^{\bot}}(P_{\bX'_i,\hatY_i'}\| P_{\bX,Y}) \Big)} ~\bigg], \label{Eq:last line of gen_1 bound}
	\end{align}
	where $P_{\btheta_1,(\bX,Y)|\xi_0,\bmu^{\bot}}=Q_{\btheta_1|\xi_0,\bmu^{\bot}}\otimes P_{\bX,Y}$ and $Q_{\btheta_1|\xi_0,\bmu^{\bot}}$ denotes the marginal distribution of $\btheta_1$ under parameters $(\xi_0,\bmu^{\bot})$.

	In the following, we derive   closed-form expressions of the mutual information and KL-divergence in \eqref{Eq:last line of gen_1 bound}. For any $j\in[1:m]$:
	\begin{itemize}[leftmargin= 10 pt]
		\item \textbf{Calculate $I_{\xi_0,\bmu^{\bot}}(\btheta_1; \bX'_j,\hatY_j')$:}
		For arbitrary random variables $X$ and $U$, we define the {\em disintegrated conditional  differential entropy} of $X$ given $U$ as 
		\begin{align}
			h_U(X):=h(P_{X|U}).
		\end{align}
		This is a $\sigma(U)$-measurable random variable.
		Conditioned on a certain pair   $(\xi_0,\bmu^{\bot})$, the mutual information between $\btheta_1$ and $(\bX'_i,\hatY_i')$ is
		\begin{align}
			&I_{\xi_0,\bmu^{\bot}}(\btheta_1; \bX'_i,\hatY_i')\nn\\
			&=h_{\xi_0,\bmu^{\bot}}\bigg(\frac{1}{m}\sum_{i=1}^m\sgn(\btheta_0^\top \bX'_i)\bX'_i \bigg)-h_{\xi_0,\bmu^{\bot}}\bigg(\frac{1}{m}\sum_{j=1}^m \hatY_j'\bX'_j \bigg|\bX'_i,\hatY_i' \bigg)\\
			&=h_{\xi_0,\bmu^{\bot}}\bigg(\frac{1}{m}\sum_{i=1}^m\sgn(\btheta_0^\top \bX'_i)\bX'_i \bigg)-h_{\xi_0,\bmu^{\bot}}\bigg(\frac{1}{m}\sum_{j\in[m],j\ne i}\sgn(\btheta_0^\top \bX'_j)\bX'_j \bigg)\\
			&=h_{\xi_0,\bmu^{\bot}}\bigg(\frac{1}{m}\sum_{i=1}^m\sgn(\btheta_0^\top \bX'_i)\bX'_i \bigg) \nn\\
			&\quad -h_{\xi_0,\bmu^{\bot}}\bigg(\frac{1}{m-1}\sum_{j\in[m],j\ne i}\sgn(\btheta_0^\top \bX'_j)\bX'_j \bigg)-d\log \frac{m-1}{m}. \label{Eq: I calculation}
			%		&\xrightarrow{m\to\infty} 0.
			%		&=\frac{d}{2}\log \frac{m}{m-1}, \quad \text{for $m$ large enough}.
		\end{align}
		As $m\to\infty$, $I_{\xi_0,\bmu^{\bot}}(\btheta_1; \bX'_i,\hatY_i')\to 0$ almost surely and  hence, in probability. Thus, for any $\epsilon>0$, and there exists $m_0(\epsilon,d,\delta)\in\bbN$ such that for all $m>m_0$,
		\begin{align}
			\bbP_{\xi_0,\bmu^{\bot}}(I_{\xi_0,\bmu^{\bot}}(\btheta_1; \bX'_i,\hatY_i')>\epsilon)\leq \delta. \label{Eq:I_1 Op1}
		\end{align}
		
		%	That is,
		%	for any $\epsilon>0$ and any $\delta\in(0,1)$,
		%	there exists $m_1(\epsilon,\delta)\in\bbN_+$ such that for all $m>m_1(\epsilon,\delta)$,
		%	\begin{align}
			%		\bbP_{\xi_0,\bmu^{\bot}}\Big(I_{\xi_0,\bmu^{\bot}}(\btheta_1; \bX'_j,\hatY_j')>\epsilon \Big)\leq 1-\delta. \label{Eq:I1 convergence}
			%	\end{align}

		\item \textbf{Calculate $D_{\xi_0,\bmu^{\bot}}(P_{\bX_j',\hatY_j'} \| P_{\bX,Y})$:}
		First of all, since $P_{\hatY_j'}=P_{Y}$ (cf.~\eqref{Eq:hatY prob}) regardless of the values of $(\xi_0,\bmu^{\bot})$, the {\em disintegrated conditional KL-divergence} can be rewritten as
		\begin{align}
			&D_{\xi_0,\bmu^{\bot}}(P_{\bX_j',\hatY_j'} \| P_{\bX,Y}) \nn\\
			&=P_{ \hatY_j'}(-1)D_{\xi_0,\bmu^{\bot}}(P_{\bX_j'| \hatY_j'=-1} \| P_{\bX | Y=-1})+P_{ \hatY_j'}(1)D_{\xi_0,\bmu^{\bot}}(P_{\bX_j'| \hatY_j'=1} \| P_{\bX | Y=1}). \label{Eq:simplify KL div}
		\end{align}
		
		Recall the decomposition of a Gaussian vector $\tilde{\bg}_j\sim\calN(0,\bI_d)$ in \eqref{Eq:new decomp Gaussian vec}. Note that $\rank(\cov(\tilde{\bg}_j^{\bot}))=\rank(\bI_d-\bar{\btheta}\bar{\btheta}^\top)=d-1$.
		
		For any pair of labelled data sample $(\bX,Y)$, from \eqref{Eq:X new decomp}, we similarly decompose $\bX$ as $\bX=Y\bmu+\sigma(\tilg\bar{\btheta}_0+\tilde{\bg}^{\bot})$, where $\tilg\sim\calN(0,1)$ and $\tilde{\bg}^{\bot}\sim\calN(0,\bI_d-\bar{\btheta}_0\bar{\btheta}_0^\top)$. Let $p_{\tilg}$ and $p_{\tilde{\bg}^{\bot}}$ denote the probability density functions of $\tilg$ and $\tilde{\bg}^{\bot}$, respectively. 
		For any $\bx=\bmu+\sigma(u\bar{\btheta}_0+\bu^{\perp})\in\bbR^d$, the joint probability distribution at $(\bX,Y)=(\bx,1)$ is given by 
		\begin{align}
			P_{\bX,Y}(\bx,1)&=P_Y(1)p_{\bmu}(\bx|1)\nn\\
			&=\frac{P_Y(1)}{\sqrt{(2\pi)^d}\sigma^d}\exp\bigg(-\frac{1}{2\sigma^2}(\bx-y\bmu)^\top(\bx-y\bmu) \bigg)\\
			&=\frac{P_Y(1)}{\sqrt{(2\pi)^d}\sigma^d}\exp\bigg(-\frac{1}{2\sigma^2}(\sigma u\bar{\btheta_0}+\sigma\bu^{\bot})^\top(\sigma u\bar{\btheta_0}+\sigma\bu^{\bot}) \bigg)\\
			&=\frac{P_Y(y)}{\sqrt{(2\pi)^d}\sigma^d}\exp\bigg(-\frac{u^2 }{2}\bigg)\exp\bigg(-\frac{(\bu^{\bot})^\top \bu^{\bot} }{2}\bigg)\\
			&=P_Y(1)p_{\tilg}(u)p_{\tilde{\bg}^{\bot}}(\bu^{\bot}).
		\end{align}
		%\red{\sout{where $p_{\tilde{\bg}^{\bot}}$ is not absolutely continuous with respect to the Lebesgue measure on $\bbR^d$ but with respect to $\bbR^{d-1}$.} }
		%	can not be expressed with respect to $d$-dimensional Lebesgue measure but can be expressed with respect to a $(d-1)$-dimensional Lebesgue measure using a different base measure and disintegration theorem.
		
		Similarly, for any $\bx=-\bmu+\sigma(u\bar{\btheta}_0+\bu^{\perp})\in\bbR^d$, the joint probability density evaluated at $(X,Y)=(\bx,-1)$ is given by
		\begin{align}
			P_{\bX,Y}(\bx,-1)=P_Y(-1)p_{\bmu}(\bx|-1)
			=P_Y(-1)p_{\tilg}(u)p_{\tilde{\bg}^{\bot}}(\bu^{\bot}).
		\end{align}
		
		Second, we have $P_{\bX_j'|\hatY_j'}=\sum_{y\in\{-1,+ 1\}}P_{\bX_j'|\hatY_j',Y_j'=y}P_{Y_j'=y|\hatY_j'}$.
		%	 $P_{\bX_j',Y_j^{'(1)}}=\sum_{Y_j'}P_{Y_j'}P_{\bX_j',Y_j^{'(1)}|Y_j'}=\sum_{Y_j'}P_{Y_j'}P_{Y_j^{'(1)}|Y_j'}P_{\bX_j'|Y_j^{'(1)},Y_j'}$.	
		%	 $P_{\bX_j',Y_j^{'(1)}}=P_{Y_j^{'(1)}}P_{\bX_j'|Y_j^{'(1)}}=P_{Y_j^{'(1)}}\sum_{Y_j'}P_{Y_j'|Y_j^{'(1)}}P_{\bX_j'|Y_j^{'(1)},Y_j'}$. 
		The conditional probability distribution $P_{Y_j'|\hatY_j'}$ can be calculated as follows
		\begin{align}
			P_{Y_j'|\hatY_j'}=\frac{ P_{\hatY_j'|Y_j'}P_{Y_j'}}{P_{\hatY_j'}}=P_{\hatY_j'|Y_j'},
		\end{align}
		where the last equality follows since $P_{Y_j'}(-1)=P_{Y_j'}(1)=P_{\hatY_j'}(-1)=P_{\hatY_j'}(1)=1/2$.	
		Since $\hatY_j'=\sgn(Y_j' \alpha+\sigma \tilg_j)$ (cf. \eqref{Eq:theta*X new decomp}), we have
		\begin{align}
			P_{\hatY_j'|Y_j'}(-1|-1)=\Pr(Y_j' \alpha+\sigma \tilg_j<0|Y_j'=-1)=\rmQ \bigg(-\frac{\alpha}{\sigma} \bigg),
		\end{align}
		and similarly, 
		\begin{align}
			P_{\hatY_j'|Y_j'}(1|-1)=\rmQ \bigg(\frac{\alpha}{\sigma} \bigg),\;\;\;\; P_{\hatY_j'|Y_j'}(-1|1)=\rmQ \bigg(\frac{\alpha}{\sigma} \bigg), \;\;\;\; P_{\hatY_j'|Y_j'}(1|1)=\rmQ \bigg(-\frac{\alpha}{\sigma} \bigg).
		\end{align}
		Thus, we conclude that
		\begin{align}
			P_{Y_j'|\hatY_j'} (y_j' | \haty_j') =\left\{\begin{array}{cc}
				\rmQ (-\frac{\alpha}{\sigma} ) &y_j'=\haty_j'\\
				\rmQ (\frac{\alpha}{\sigma} ) &y_j'\neq \haty_j'.
			\end{array} \right.
		\end{align}
		
		To calculate the conditional probability distribution $P_{\bX_j'|\hatY_j',Y_j'}$, recall the decomposition of $\bX_j'$ and $\bar{\btheta}_0^\top\bX_j'$ in \eqref{Eq:X new decomp} and \eqref{Eq:theta*X new decomp}.
		Since the event $\{\hatY_j'=-1,Y_j'=-1\}$ is equivalent to $\{\tilg_j<\alpha/\sigma\}$ and $\tilg_j\sim\calN(0,1)$, 
		%	We can easily derive the conditional density $p_{\tilg_j|\tilg_j<\alpha/\sigma}(z)=\frac{\mathbbm{1}\{z\leq \alpha/\sigma\}f_{\tilg_j}(z)}{\Phi(\alpha/\sigma)}$ for any $z\in\bbR$. 
		the conditional density of $\tilg_j$ given $\hatY_j'=-1,Y_j'=-1$ is given by
		\begin{align}
			p_{\tilg_j|\hatY_j',Y_j'}(u|-1,-1)=p_{\tilg_j|\tilg_j\leq\alpha/\sigma}(u)=\frac{\mathbbm{1}\{u\leq \alpha/\sigma\}p_{\tilg_j}(u)}{\Phi(\alpha/\sigma)}, \quad \forall u\in\bbR.
		\end{align}
		Similarly, for any $u\in\bbR$
		\begin{align}
			p_{\tilg_j|\hatY_j',Y_j'}(u|-1,1)&=p_{\tilg_j|\tilg_j\leq-\alpha/\sigma}(u)=\frac{\mathbbm{1}\{u\leq -\alpha/\sigma\}f_{\tilg_j}(u)}{\Phi(-\alpha/\sigma)},\\
			p_{\tilg_j|\hatY_j',Y_j'}(u|1,-1)&=p_{\tilg_j|\tilg_j>\alpha/\sigma}(u)=\frac{\mathbbm{1}\{z> \alpha/\sigma\}f_{\tilg_j}(u)}{\rmQ (\alpha/\sigma)},\\
			p_{\tilg_j|\hatY_j',Y_j'}(u|1,1)&=p_{\tilg_j|\tilg_j>-\alpha/\sigma}(u)=\frac{\mathbbm{1}\{u> -\alpha/\sigma\}p_{\tilg_j}(u)}{\rmQ (-\alpha/\sigma)}.
		\end{align}
		
		For any $\bx=\bmu+\sigma(u\bar{\btheta}_0+\bu^{\perp})\in\bbR^d$, given $\hatY_j'=1,Y_j'=1$, the conditional probability distribution at $\bX_j'=\bx$ is given by
		\begin{align}
			P_{\bX_j'|\hatY_j',Y_j'}(\bx|1,1)&=P_{\bmu+\sigma\tilg_j \bar{\btheta}_0+\sigma\tilde{\bg}_j^{\bot}|\hatY_j',Y_j'}(\bmu+\sigma(u\bar{\btheta}_0+\bu^{\perp})|1,1)\\
			&=P_{\sigma\tilg_j \bar{\btheta}_0+\sigma\tilde{\bg}_j^{\bot}|\hatY_j',Y_j'}(\sigma(u\bar{\btheta}_0+\bu^{\perp})|1,1)\\
			&=p_{\tilg_j|\hatY_j',Y_j'}(u|1,1)p_{\tilde{\bg}_j^{\bot}}(\bu^{\bot}) \label{Eq:from P_x to f_g},
		\end{align}
		where \eqref{Eq:from P_x to f_g} follows since $\tilg_j$ and $\tilde{\bg}_j^{\bot}$ are mutually independent and $\bar{\btheta}_0\perp \tilde{\bg}_j^{\bot}$.
		
		Since we can decompose $2\bmu/\sigma$ as 
		\begin{align}
			\frac{2\bmu}{\sigma}=\frac{2\alpha \bar{\btheta}_0+2\beta^2\bmu-2\alpha\beta\bup}{\sigma}=\frac{2\alpha}{\sigma} \bar{\btheta}_0+\bar{\btheta}_0^{\bot},
		\end{align}
		given $\hatY_j'=1,Y_j'=-1$, the conditional probability distribution at $\bX_j'=\bx$ is given by
		\begin{align}
			P_{\bX_j'|\hatY_j',Y_j'}(\bx|1,-1)&=P_{-\bmu+\sigma\tilg_j \bar{\btheta}_0+\sigma\tilde{\bg}_j^{\bot}|\hatY_j',Y_j'}(\bmu+\sigma(u\bar{\btheta}_0+\bu^{\perp})|1,-1)\\
			&=P_{\sigma\tilg_j \bar{\btheta}_0+\sigma\tilde{\bg}_j^{\bot}|\hatY_j',Y_j'}\Big(\sigma\Big(\frac{2\bmu}{\sigma}+u\bar{\btheta}_0+\bu^{\perp} \Big) \Big| 1,-1\Big)\\
			&=p_{\tilg_j|\hatY_j',Y_j'}\Big(u+\frac{2\alpha}{\sigma} \Big|1,-1\Big)p_{\tilde{\bg}_j^{\bot}}(\bu^{\bot}+\bar{\btheta}_0^{\bot}). \label{Eq:from P_x to f_g 2}
		\end{align}
		Similarly, 	for any $\bx=-\bmu+\sigma(u\bar{\btheta}_0+\bu^{\perp})\in\bbR^d$, given $\hatY_j'=-1,Y_j'=1$, the conditional   distribution at $\bX_j'=\bx$ is given by
		\begin{align}
			P_{\bX_j'|\hatY_j',Y_j'}(\bx|-1,1)&=P_{\bmu+\sigma\tilg_j \bar{\btheta}_0+\sigma\tilde{\bg}_j^{\bot}|\hatY_j',Y_j'}(-\bmu+\sigma(u\bar{\btheta}_0+\bu^{\perp})|-1,1)\\
			&=p_{\tilg_j|\hatY_j',Y_j'}\Big(u-\frac{2\alpha}{\sigma} \Big|-1,1\Big)p_{\tilde{\bg}_j^{\bot}}(\bu^{\bot}-\bar{\btheta}_0^{\bot});
		\end{align}
		and given $\hatY_j'=-1,Y_j'=-1$,
		\begin{align}
			P_{\bX_j'|\hatY_j',Y_j'}(\bx|-1,-1)&=P_{-\bmu+\sigma\tilg_j \bar{\btheta}_0+\sigma\tilde{\bg}_j^{\bot}|\hatY_j',Y_j'}(-\bmu+\sigma(u\bar{\btheta}_0+\bu^{\perp})|-1,-1)\\
			&=p_{\tilg_j|\hatY_j',Y_j'}(u|-1,-1)p_{\tilde{\bg}_j^{\bot}}(\bu^{\bot}).
		\end{align}
		
		%	In fact, given $\hatY_j',Y_j'$, the conditional probability distribution of $\bX_j'$ is given by
		%	\begin{align}
			%		P_{\bX_j'|\hatY_j',Y_j'}= P_{Y_j',\tilg_j,\tilde{\bg}_j^{\bot}|\hatY_j',Y_j'}=P_{Y_j'|Y_j^{'(1)},Y_j'} P_{\tilg_j|\hatY_j',Y_j'} P_{\tilde{\bg}_j^{\bot}}=P_{\tilg_j|\hatY_j',Y_j'} P_{\tilde{\bg}_j^{\bot}}=f_{\tilg_j|\hatY_j',Y_j'} f_{\tilde{\bg}_j^{\bot}},
			%	\end{align}
		%	and thus
		Furthermore, for any $\bx=-\bmu+\sigma(u\bar{\btheta}_0+\bu^{\perp})\in\bbR^d$, we have
		\begin{align}
			P_{\bX_j'|\hatY_j'=-1}(\bx)&=\sum_{y\in\{-1,+ 1\}}P_{\bX_j'|\hatY_j'=-1,Y_j'=y}(\bx)P_{Y_j'|\hatY_j'=-1}(y)\\
			&= P_{Y_j'|\hatY_j'=-1}(1)p_{\tilg_j|\hatY_j',Y_j'}\Big(u-\frac{2\alpha}{\sigma} \Big|-1,1\Big)p_{\tilde{\bg}_j^{\bot}}(\bu^{\bot}-\bar{\btheta}_0^{\bot}) \nn\\
			&\quad +P_{Y_j'|\hatY_j'=-1}(-1)p_{\tilg_j|\hatY_j',Y_j'}(u|-1,-1)p_{\tilde{\bg}_j^{\bot}}(\bu^{\bot})\\
			&=\mathbbm{1}\Big\{u\leq \frac{\alpha}{\sigma} \Big\}p_{\tilg_j}\Big(u-\frac{2\alpha}{\sigma}\Big)p_{\tilde{\bg}_j^{\bot}}(\bu^{\bot}-\bar{\btheta}_0^{\bot})+\mathbbm{1}\Big\{u\leq \frac{\alpha}{\sigma} \Big\}p_{\tilg_j}(u)p_{\tilde{\bg}_j^{\bot}}(\bu^{\bot});
		\end{align}
		for any $\bx=\bmu+\sigma(u\bar{\btheta}_0+\bu^{\perp})\in\bbR^d$,  we have
		\begin{align}
			&P_{\bX_j'|\hatY_j'=1}(\bx)=\sum_{y\in\{-1,+1\}} P_{\bX_j'|\hatY_j'=1,Y_j'=y}(\bx)P_{Y_j'|\hatY_j'=1}(y)\\
			&=\mathbbm{1}\Big\{u> -\frac{\alpha}{\sigma} \Big\}p_{\tilg_j}\Big(u+\frac{2\alpha}{\sigma}\Big)p_{\tilde{\bg}_j^{\bot}}(\bu^{\bot}+\bar{\btheta}_0^{\bot})+\mathbbm{1}\Big\{u> -\frac{\alpha}{\sigma} \Big\}p_{\tilg_j}(u)p_{\tilde{\bg}_j^{\bot}}(\bu^{\bot}).
		\end{align}
		
		%	For any $z\in\bbR$,
		%	\begin{align}
			%		&\sum_{y\in\{\pm 1\}}P_{Y_j'|\hatY_j'=-1}(y) f_{\tilg_j|\hatY_j'=-1,Y_j'=y}(z)\nn\\
			%		&= \mathbbm{1}\{z\leq \alpha/\sigma\}f_{\tilg_j}(z)+\mathbbm{1}\{z\leq -\alpha/\sigma\}f_{\tilg_j}(z)
			%	\end{align}
		%	and
		%	\begin{align}
			%		&\sum_{y\in\{\pm 1\}}P_{Y_j'|\hatY_j'=1}(y) f_{\tilg_j|Y_j^{'(1)}=1,Y_j'=y}(z)\nn\\
			%		&=\mathbbm{1}\{z> \alpha/\sigma\}f_{\tilg_j}(z)+\mathbbm{1}\{z> -\alpha/\sigma\}f_{\tilg_j}(z).
			%	\end{align}
		
		Define the set $\calU_0^{\bot}(\xi_0,\bmu^{\bot}):=\{\bu^{\bot}\in\bbR^d: \bu^{\bot}\perp \btheta_0\}$. We also use $\calU_0^{\bot}$ to represent  $\calU_0^{\bot}(\xi_0,\bmu^{\bot})$, if there is no risk of confusion. Recall \eqref{Eq:G_sigma} and note that $\int_{\calU_0^{\bot}}p_{\tilde{\bg}^{\bot}}(\bu^{\bot})\rmd \bu^{\bot}=1$. Finally, the KL-divergence is given by
		\begin{align}
			&D_{\xi_0,\bmu^{\bot}}(P_{\bX_j'| \hatY_j'=-1} \| P_{\bX | Y=-1}) \nn\\
			&=\int_{\calU_0^{\bot}}\int^{\frac{\alpha}{\sigma}}_{-\infty}\bigg(p_{\tilg_j}\Big(u-\frac{2\alpha}{\sigma}\Big)p_{\tilde{\bg}_j^{\bot}}(\bu^{\bot}-\bar{\btheta}_0^{\bot})+p_{\tilg_j}(u)p_{\tilde{\bg}_j^{\bot}}(\bu^{\bot}) \bigg) \nn\\
			&\qquad  \times \log\bigg(1+ \frac{p_{\tilg_j}\Big(u-\frac{2\alpha}{\sigma}\Big)p_{\tilde{\bg}_j^{\bot}}(\bu^{\bot}-\bar{\btheta}_0^{\bot}) }{p_{\tilg_j}(u)p_{\tilde{\bg}_j^{\bot}}(\bu^{\bot})}\bigg)\, \rmd u\, \rmd \bu^{\bot}
			=G_{\sigma}(\alpha,\xi_0,\bmu^{\bot}) 
		\end{align}
		and 
		\begin{align}
			&D_{\xi_0,\bmu^{\bot}}(P_{\bX_j'| \hatY_j'=1} \| P_{\bX | Y=1}) \nn\\
			&=\int_{\calU_0^{\bot}}\int_{-\frac{\alpha}{\sigma}}^{+\infty}\bigg(p_{\tilg_j}\Big(u+\frac{2\alpha}{\sigma}\Big)p_{\tilde{\bg}_j^{\bot}}(\bu^{\bot}+\bar{\btheta}_0^{\bot})+p_{\tilg_j}(u)p_{\tilde{\bg}_j^{\bot}}(\bu^{\bot}) \bigg) \nn\\
			&\qquad \times  \log\bigg(1+ \frac{p_{\tilg_j}\Big(u+\frac{2\alpha}{\sigma}\Big)p_{\tilde{\bg}_j^{\bot}}(\bu^{\bot}+\bar{\btheta}_0^{\bot}) }{p_{\tilg_j}(u)p_{\tilde{\bg}_j^{\bot}}(\bu^{\bot})}\bigg)\, \rmd u\, \rmd \bu^{\bot}
			%		&=\int_{\calZ_0^{\bot}}\int^{\frac{\alpha}{\sigma}}_{-\infty}\bigg(f_{\tilg_j}\Big(-z+\frac{2\alpha}{\sigma}\Big) f_{\tilde{\bg}_j^{\bot}}(-\bz^{\bot}+\bar{\btheta}^{\bot})+f_{\tilg_j}(-z) f_{\tilde{\bg}_j^{\bot}}(-\bz^{\bot}) \bigg) \log\bigg(1+ \frac{f_{\tilg_j}\Big(-z+\frac{2\alpha}{\sigma}\Big)f_{\tilde{\bg}_j^{\bot}}(-\bz^{\bot}+\bar{\btheta}^{\bot}) }{f_{\tilg_j}(-z)f_{\tilde{\bg}_j^{\bot}}(-\bz^{\bot})}\bigg) \, \rmd z \,  \rmd \bz^{\bot}
			 =G_{\sigma}(\alpha,\xi_0,\bmu^{\bot}) , \label{Eq:f shift}
		\end{align}
		where \eqref{Eq:f shift} follows from since $p_{\tilg_j}$ and $p_{\tilde{\bg}_j^{\bot}}$ are zero-mean Gaussian distributions.
		Then from~\eqref{Eq:simplify KL div}, we have
		\begin{align}\label{Eq:disintegrated KL}
			&D_{\xi_0,\bmu^{\bot}}(P_{\bX_j',\hatY_j'} \| P_{\bX,Y})=G_{\sigma}(\alpha,\xi_0,\bmu^{\bot}).
		\end{align}

	\end{itemize}
	
	Thus, by combining the aforementioned results, we get the closed-form expression of the upper bound for $|\mathrm{gen}_1|$. Indeed, if we fix some $d\in\bbN$, $\epsilon>0$ and $\delta\in(0,1)$,  there exists $n_0(d,\delta)\in\bbN$, $m_0(\epsilon,d,\delta)\in\bbN$, $c_0(d,\delta)\in(\tilc_1,\infty)$,  $r_0(d,\delta)\in\bbR_+$ such that for all $n>n_0, m>m_0, c>c_0, r>r_0$, $\delta_{m,c-\tilc_1,d}<\frac{\delta}{3}$, $\delta_{r,d}<\frac{\delta}{3}$, 
	and with probability at least $1-\delta$, 
	\begin{align}
		&|\mathrm{gen}_1| 
		%		&\leq \frac{1}{m}\sum_{i=1}^{m}\bbE_{\alpha,\bmu^{\bot}}\bigg[\sqrt{\frac{(p_2-p_1)^2}{2}\bigg(\frac{d}{2}\log\frac{m}{m-1}+2\log 2\cdot \Phi\bigg(\frac{-\alpha}{\sigma}\bigg) \bigg)} ~\bigg]\\
		\leq \sqrt{\frac{(c_2-c_1)^2}{2}} \bbE_{\xi_0,\bmu^{\bot}}\bigg[\sqrt{ G_{\sigma}(\alpha(\xi_0,\bmu^{\bot}),\xi_0,\bmu^{\bot})+\epsilon} ~\bigg]. 
	\end{align}

	\item \textbf{Pseudo-label using $\btheta_1$:} The same as those in Appendix \ref{pf of Thm:exact gen GMM}.

	\item \textbf{Iteration $t=2$:} Recall $\btheta_2$ in \eqref{Eq:theta2}, the new model parameter learned from the pseudo-labelled dataset $\hatS_{\rmu,2}$.
	
	Given any $(\btheta_1,\xi_0,\bmu^{\bot})$, for any $j\in\calI_2$, let $\bmu_2^{\btheta_1,\xi_0,\bmu^{\bot}}:=\bbE[\sgn(\bar{\btheta}_1^\top \bX'_j)\bX'_j|\btheta_1,\xi_0,\bmu^{\bot}]$ and $\bbP_{\btheta_1,\xi_0,\bmu^{\bot}}$ denotes the probability measure under the parameters $\btheta_1,\xi_0,\bmu^{\bot}$.
	Following the similar steps that derive \eqref{Eq:theta1 concentrate ineq}, for any $\veps>0$, we have
	\begin{align}
		\bbP_{\btheta_1,\xi_0,\bmu^{\bot}}\big(\| \btheta_2-\bmu_2^{\btheta_1,\xi_0,\bmu^{\bot}}\|_{\infty}>\veps \big)\leq \delta_{m,\veps,d}. \label{Eq:theta2 concentrate ineq}
	\end{align}
	%	where $\tdelta_{m,\veps,d}\xrightarrow{\text{a.s.}} 0$ as $m\to\infty$. Let $\tdelta_{m,\veps,d}^*=\sup_{\btheta_1,\xi_0,\bmu^{\bot}}\tdelta_{m,\veps,d}$.
	
	From \eqref{Eq:mu1 mu distance}, no matter what $\btheta_1$ is, we always have $\|\bmu_2^{\btheta_1,\xi_0,\bmu^{\bot}}-\bmu^{\xi_0,\bmu^{\bot}}\|\leq \tilc_1$. Then, for some $c\in(\tilc_1,\infty)$,
	\begin{align}
		\bbP_{\btheta_1,\xi_0,\bmu^{\bot}}(\btheta_2 \in \Theta_{\bmu,c})\geq 1-\delta_{m,c-\tilc_1,d}.
	\end{align}
	With probability at least $(1-\delta_{m,c-\tilc_1,d})(1-\delta_{r,d})$, the absolute  generalization error can be upper bounded as follows:
	\begin{align}
		&|\mathrm{gen}_2|=|\bbE[L_{P_\bZ}(\btheta_2)-L_{\hatS_{\rmu,2}}(\btheta_2)]|\\
		&=\bigg|\frac{1}{m}\sum_{i\in\calI_2} \bbE_{\btheta_1,\xi_0,\bmu^{\bot}}\bigg[\bbE\left[l(\btheta_2,(\bX,Y))-l(\btheta_2,(\bX'_i,\hatY_i'))|\btheta_1,\xi_0,\bmu^{\bot} \right]\bigg] \bigg|\\
		&\leq \sqrt{\frac{(c_2-c_1)^2}{2}}  \frac{1}{m}\sum_{i\in\calI_2} \bbE_{\btheta_1,\xi_0,\bmu^{\bot}}\bigg[\sqrt{I_{\btheta_1,\xi_0,\bmu^{\bot}}(\btheta_2;(\bX'_i,\hatY_i'))+D_{\btheta_1,\xi_0,\bmu^{\bot}}(P_{\bX'_i,\hatY_i'}\| P_{\bX,Y})} \bigg],
	\end{align}
	where $P_{\btheta_2,\bX,Y|\btheta_1,\xi_0,\bmu^{\bot}}=P_{\btheta_2|\btheta_1,\xi_0,\bmu^{\bot}}\otimes P_{\bX,Y}$.
	
	Similar to \eqref{Eq:I_1 Op1}, for any $\epsilon>0$ and $\delta\in(0,1)$, there exists $m_1(\epsilon,d,\delta)$ such that for all $m>m_1$,
	\begin{align}
		\bbP_{\btheta_1,\xi_0,\bmu^{\bot}}(I_{\btheta_1,\xi_0,\bmu^{\bot}}(\btheta_2;(\bX'_i,\hatY'_i))>\epsilon)\leq \delta.
	\end{align}
	
	%	Similar to \eqref{Eq:I1 convergence}, for any $\epsilon>0$ any $\delta\in(0,1)$, there exists $m_2(\epsilon,\delta)\in\bbN_+$ such that for all $m>m_2(\epsilon,\delta)$,
	%	\begin{align}
		%		&\bbP_{\btheta_1,\xi_0,\bmu^{\bot}}\Big(I_{\btheta_1,\xi_0,\bmu^{\bot}}(\btheta_2;(\bX'_i,\hatY'_i))>\epsilon\Big)\nn\\
		%		&=\bbP_{\btheta_1,\xi_0,\bmu^{\bot}}\bigg(h_{\btheta_1,\xi_0,\bmu^{\bot}}\bigg(\frac{1}{m}\sum_{j=1}^m\sgn(\btheta_1^\top \bX'_j)\bX'_j \bigg)-h_{\btheta_1,\xi_0,\bmu^{\bot}}\bigg(\frac{1}{m}\sum_{j=1,j\ne i}^m\sgn(\btheta_1^\top \bX'_j){}\bX'_j \bigg)  >\frac{\epsilon}{2} \bigg)\\
		%		&\leq 1-\delta.
		%	\end{align}	
	
	Recall \eqref{Eq:P_hatY 2 prob} that $P_{\hatY_i'|\btheta_1,\xi_0,\bmu^{\bot}}\sim \textrm{unif}(\{-1,+1\})$. 
	For any fixed $(\btheta_1,\xi_0,\bmu^{\bot})$, recall  $\bar{\btheta}_1$ can be decomposed as $\bar{\btheta}_1=\alpha_1(\xi_0,\bmu^{\bot}) \bmu+\beta_1(\xi_0,\bmu^{\bot}) \bup$.
	
	By following the similar steps in the first iteration, the disintegrated conditional KL-divergence between pseudo-labelled distribution and true distribution is given by
	\begin{align}
		&D_{\btheta_1,\xi_0,\bmu^{\bot}}\big(P_{\bX'_i,\hatY'_i}\| P_{\bX,Y}\big) \nn\\
		&=\frac{1}{2}D_{\btheta_1,\xi_0,\bmu^{\bot}}\big(P_{\bX_i'| \hatY_i'=-1} \| P_{\bX | Y=-1}\big)+\frac{1}{2}D_{\btheta_1,\xi_0,\bmu^{\bot}}\big(P_{\bX_i'| \hatY_i'=1} \| P_{\bX | Y=1}\big)\\
		&=G_{\sigma}\big(\alpha_1(\xi_0,\bmu^{\bot}),\xi_0,\bmu^{\bot}\big).
	\end{align}

	Given any pair of $(\xi_0,\bmu^{\bot})$, recall the decomposition of $\bmu_1^{\xi_0,\bmu^{\bot}}$ in \eqref{Eq:mu1 decomp}. Then the correlation between $\bmu_1^{\xi_0,\bmu^{\bot}}$ and $\bmu$ is given by
	\begin{align}
		%		&\alpha_1(\xi_0,\bmu^{\bot}):=
		\rho(\bmu_1^{\xi_0,\bmu^{\bot}},\bmu)&=\frac{1-2\rmQ \big(\frac{\alpha}{\sigma}\big)+\frac{2\sigma\alpha}{\sqrt{2\pi}}\exp(-\frac{\alpha^2}{2\sigma^2})}{\sqrt{\big(1-2\rmQ \big(\frac{\alpha}{\sigma}\big)+\frac{2\sigma\alpha}{\sqrt{2\pi}}\exp(-\frac{\alpha^2}{2\sigma^2}) \big)^2+\frac{2\sigma^2(1-\alpha^2)}{\pi}\exp(-\frac{\alpha^2}{\sigma^2})}}\\
		&=F_{\sigma}(\alpha(\xi_0,\bmu^{\bot})). \label{Eq:corre mu1 mu}
	\end{align}
	By the strong law of large numbers, we have $\alpha_1(\xi_0,\bmu^{\bot})\xrightarrow{\text{a.s.}}F_{\sigma}(\alpha(\xi_0,\bmu^{\bot}))$ as $m\to\infty$. Then for any $\epsilon>0$ and $\delta\in(0,1)$, there exists $m_2(\epsilon,d,\delta)$ such that for all $m>m_2$,
	\begin{align}
		\bbP_{\btheta_1,\xi_0,\bmu^{\bot}}\Big( \Big|G_{\sigma}\big(\alpha_1(\xi_0,\bmu^{\bot}),\xi_0,\bmu^{\bot}\big)-G_{\sigma}\Big(F_{\sigma}(\alpha(\xi_0,\bmu^{\bot})),\xi_0,\bmu^{\bot}\Big) \Big|>\epsilon \Big)\leq \delta.
	\end{align}
	
	%	For any $\epsilon>0$ and any $\delta\in(0,1)$, there exists $m_3(\epsilon,\delta)\in\bbN_+$ such that for all $m>m_3(\epsilon,\delta)$,
	%	\begin{align}
		%		\bbP_{\btheta_1,\xi_0,\bmu^{\bot}}\bigg(\bigg|G_{\sigma}\big(\alpha'_1(\xi_0,\bmu^{\bot}),\xi_0,\bmu^{\bot}\big)-G_{\sigma}\Big(F_{\sigma}(\alpha(\xi_0,\bmu^{\bot})),\xi_0,\bmu^{\bot}\Big) \bigg| >\frac{\epsilon}{2}  \bigg)\leq 1-\delta.
		%	\end{align}
	Therefore, fix some $d\in\bbN$, $\epsilon>0$ and $\delta\in(0,1)$. There exists $n_0(d,\delta)\in\bbN$, $m_3(\epsilon,d,\delta)\in\bbN$, $c_0(d,\delta)\in(\tilc_1,\infty)$,  $r_0(d,\delta)\in\bbR_+$ such that for all $n>n_0, m>m_3, c>c_0, r>r_0$, $\delta_{m,c-\tilc_1,d}<\frac{\delta}{3}$, $\delta_{r,d}<\frac{\delta}{3}$, 
	and then with probability at least $1-\delta$, the absolute generalization error at $t=2$ can be upper bounded as follows:
	\begin{align}
		&|\mathrm{gen}_2|\leq  \frac{c_2-c_1}{\sqrt{2}} \bbE_{\xi_0,\bmu^{\bot}}\left[\sqrt{ G_{\sigma}\Big(F_{\sigma}(\alpha(\xi_0,\bmu^{\bot})),\xi_0,\bmu^{\bot}\Big)+\epsilon} \right].
	\end{align}

	\item \textbf{Any iteration $t\in[3:\tau]$:} By similarly repeating the calculation in iteration $t=2$, we obtain the upper bound for $|\mathrm{gen}_t|$ in \eqref{Eq:upper bd gent}.
	
\end{enumerate}

\section{Reusing $S_{\rml}$ in Each Iteration}\label{Append:reuse S_l}
If the labelled data $S_{\rml}$ are reused in each iteration and $w=\frac{n}{n+m}$ (cf.~\eqref{Eq: def of empirical risk}), for each $t\in[1:\tau]$, the learned model parameter is given by
\begin{align}
	 \btheta_t'&=\frac{n}{n+m}\btheta_0+\frac{1}{n+m}\sum_{i\in\calI_t} \hatY_i'\bX'_i\\
	&=\frac{n}{n+m}\btheta_0+\frac{1}{n+m}\sum_{i\in\calI_t}\sgn(\bar{\btheta}_{t-1}'^\top \bX'_i)\bX'_i. \label{Eq:theta_t with labelled}
\end{align}

Similarly to $F_{\sigma}$, let us define the {\em enhanced correlation evolution function} $\tilF_{\sigma,\xi_0,\bmu^{\bot}} : [-1,1]\to [-1,1]$ as follows:
\begin{align}
	\tilF_{\sigma,\xi_0,\bmu^{\bot}}(x) =\bigg(1+\frac{\big(w\frac{\sigma\|\bmu^{\bot}\|_2}{n}+(1-w)(\frac{2\sigma \sqrt{1-x^2}}{\sqrt{2\pi}}\exp(-\frac{x^2}{2\sigma^2}) \big)^2}{\big(w(1+\frac{\sigma}{\sqrt{n}}\xi_0)+(1-w)(1-2\rmQ \big(\frac{x}{\sigma}\big)+\frac{2\sigma x}{\sqrt{2\pi}}\exp(-\frac{x^2}{2\sigma^2}))\big)^2} \bigg)^{-\frac{1}{2}}. \label{Eq:new F_sig}
\end{align}
From Theorem \ref{Thm:gen bound GMM}, we can obtain similar characterization for $\mathrm{gen}_t$.
\begin{corollary}\label{Coro:reuse labeldata gen bound}
Fix any $\sigma\in\bbR_+$, $d\in\bbN$ and  $\alpha=\alpha(\xi_0,\bmu^{\bot})$. For almost all sample paths, %, there exists a vanishing sequence $\epsilon_m$ ($\epsilon_m\to 0$ as $m\to\infty$) , such that for any $t\in[1:\tau]$,
\begin{align}
	  &\mathrm{gen}_t-o(1)\nn\\
	&=\bbE_{\xi_0,\bmu^{\bot}}\bigg[\frac{m(m\!-\! 1)(J_{\sigma}^2(\tilF_{\sigma,\xi_0,\bmu^{\bot}}(\alpha)) \! +\!  K_{\sigma}^2(\tilF_{\sigma,\xi_0,\bmu^{\bot}}(\alpha))) \! -\!  m(m-n)J_{\sigma}(\tilF_{\sigma,\xi_0,\bmu^{\bot}}(\alpha))\! -\! nm}{(n+m)^2\sigma^2} \bigg]   .\label{Eq:exact gent reuse labelled data}
\end{align}
%where $o(1)$ tends to zero as $m\to\infty$.
%where $\epsilon_m'=\epsilon_m+\frac{(n+2m)(1+d\sigma^2)}{(n+m)^2\sigma^2}\to 0$ as $m\to\infty$ and $\alpha$ stands for $\alpha(\xi_0,\bmu^{\bot})$.
\end{corollary}
The proof of Corollary \ref{Coro:reuse labeldata gen bound} is provided in Appendix \ref{pf of Coro:reuse labeldata gen bound}.

Recall the definition of the function $\tilF_{\sigma,\xi_0,\bmu^{\bot}}$ in \eqref{Eq:new F_sig}. Let the $t$-th iterate of $\tilF_{\sigma,\xi_0,\bmu^{\bot}}$ be denoted as $\tilF_{\sigma,\xi_0,\bmu^{\bot}}^{(t)}$ with initial condition $\tilF^{(0)}_{\sigma,\xi_0,\bmu^{\bot}}(x)=x$. As shown in Figure \ref{Fig:newFsig}, we can see that for any fixed $(\sigma, \xi_0, \bmu^{\bot})$, $\tilF_{\sigma,\xi_0,\bmu^{\bot}}^{(t)}$ has a similar behaviour as $F_{\sigma}^{(t)}$ as $t$ increases, which implies that the gen-error in \eqref{Eq:exact gent reuse labelled data} in Corollary~\ref{Coro:reuse labeldata gen bound} also decreases as $t$ increases. As a result, $\tilF^{(t)}_{\sigma,\xi_0,\bmu^{\bot}}$ represents the improvement of the model parameter $\btheta_t$ over the iterations.

\begin{figure}[!t]
	\centering
	\begin{minipage}[t]{0.43\linewidth}
		\centering
		\includegraphics[width=1\linewidth]{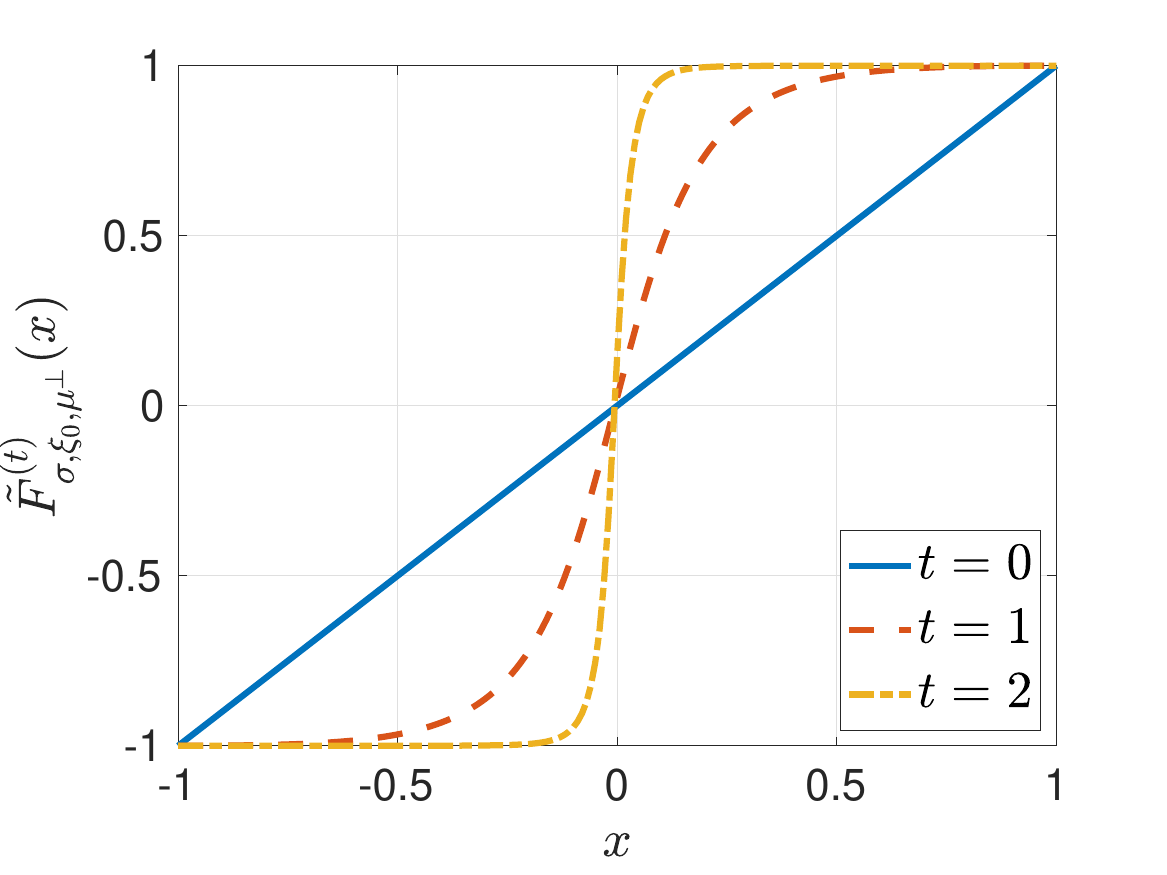}
		\caption{$\tilF_{\sigma,\xi_0,\bmu^{\bot}}^{(t)}(x)$ versus  $x$ for $t\in\{0,1,2\}$ when $\sigma=0.5$, $\xi_0=0$, $\|\bmu^{\bot}\|_2=1$, $n=10$,  and $m=1000$.}
		\label{Fig:newFsig}
	\end{minipage}
	\hspace{5pt}
	\begin{minipage}[t]{0.43\linewidth}
		\centering
		\includegraphics[width=1\linewidth]{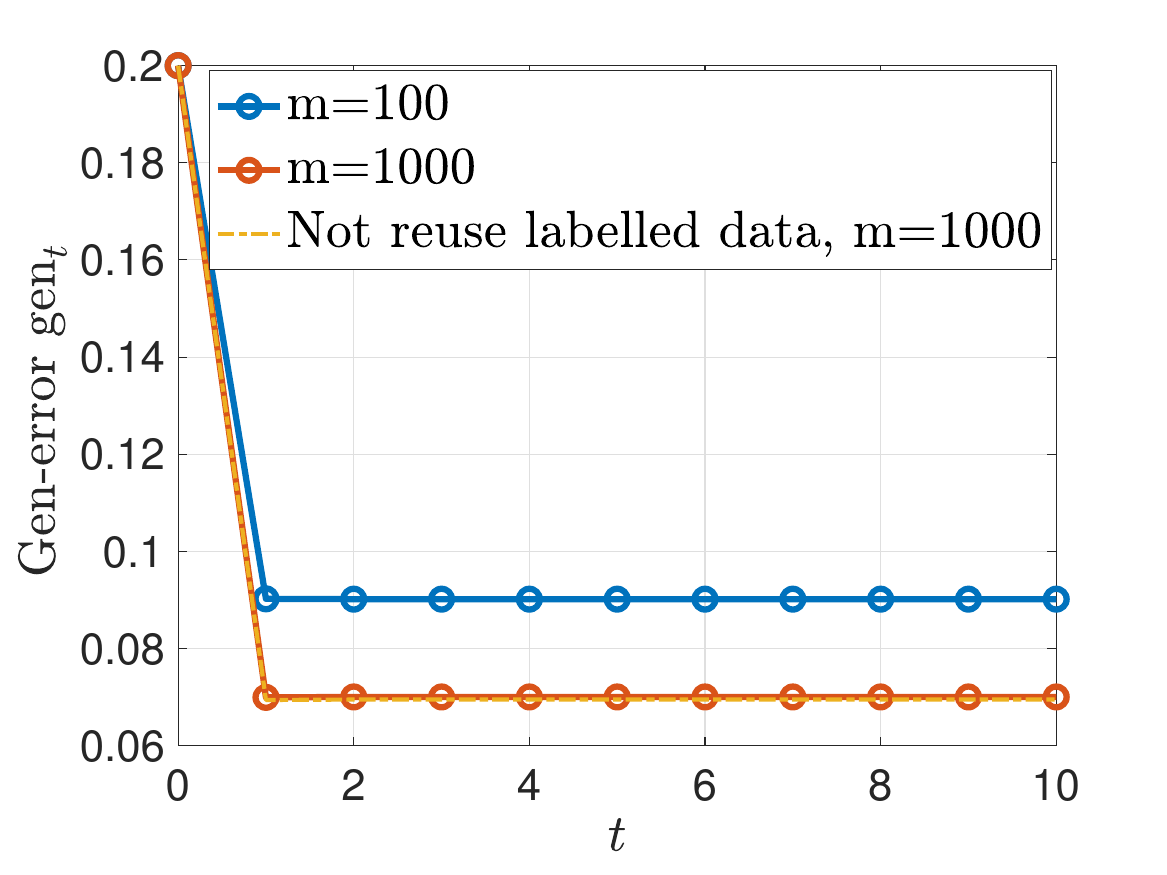}
		\caption{$\mathrm{gen}_t$ versus  $t$ for $m=100$ and $m=1000$, when $n=10$, $\sigma=0.6$, $d=2$, and $\bmu=(1,0)$.}
		\label{Fig:gen_bd_with_labelled}
	\end{minipage}
\end{figure}

%Combining Theorem \ref{Coro: sub Gau gen}, we can then extend Theorem \ref{Thm:gen bound GMM} to the following corollary.
%\begin{corollary}\label{Coro:reuse labeldata gen bound}
%	Fix some $d\in\bbN$, $\epsilon>0$ and $\delta\in(0,1)$. For $m$ large enough,
%	with probability at least $1-\delta$, the absolute generalization error at any $t\in[1:\tau]$ can be upper bounded as follows
%	\begin{align}
%		&\big|\mathrm{gen}_t(P_Z, P_X, \{P_{\theta_k|S_{\rml},S_{\rmu}}\}_{k=0}^t, \{f_{\theta_k}\}_{k=0}^{t-1} ) \big|\nn\\
%		&\leq w\sqrt{\frac{(c_2-c_1)^2 d}{4}\log\frac{n}{n-1} } \nn\\
%		&\quad +(1-w)\sqrt{\frac{(c_2-c_1)^2}{2}}~\bbE_{\xi_0,\bmu^{\bot}}\bigg[\sqrt{G_{\sigma}\big(\tilF_{\sigma,\xi_0,\bmu^{\bot}}^{(t-1)}(\alpha(\xi_0,\bmu^{\bot})),\xi_0,\bmu^{\bot}\big) + \epsilon} ~\bigg], \label{Eq:upper bd gent reuse labelled data}
%	\end{align}
%	where $\alpha(\xi_0,\bmu^{\bot})=\rho(\btheta_0,\bmu)$ in \eqref{Eq:rho_0}, $\xi_0\sim\calN(0,1)$, $\bmu^{\bot}\sim\calN(0,\bI_d-\bmu\bmu^\top)$ is a random vector perpendicular to $\bmu$ and independent of $\xi_0$.
%\end{corollary}
%The proof of Corollary \ref{Coro:reuse labeldata gen bound} is provided in Appendix \ref{pf of Coro:reuse labeldata gen bound}.
As shown in Figure \ref{Fig:gen_bd_with_labelled}, under the same setup as Figure \ref{Fig:gen bd and empirical}, when the labelled data $S_{\rml}$ are reused in each iteration, 

the upper bound for $|\mathrm{gen}_t|$ is also a decreasing function of $t$. When $m=1000$, $\mathrm{gen}_t$ is almost the same as the gen-error when the labelled data are not reused in the subsequent iterations, which means that for large enough $m/n$, reusing the labelled data does not necessarily help to improve the generalization performance. Moreover, when $m=100$, $\mathrm{gen}_t$ is higher than that for $m=1000$, which coincides with the intuition that increasing the number of unlabelled data helps to reduce the generalization error.

\section{Proof of Corollary \ref{Coro:reuse labeldata gen bound}}\label{pf of Coro:reuse labeldata gen bound}
%When $w=\frac{n}{n+m}$, if the labelled data $S_{\rml}$ are reused in each iteration, for each $t\in[1:\tau]$, the learned model parameter is given by
%	\begin{align}
%	&\btheta_t=\frac{n}{n+m}\btheta_{t-1}+\frac{1}{n+m}\sum_{i=(t-1)m+1}^{tm} \hatY_i'\bX'_i\\
%	&=\frac{n}{n+m}\btheta_{t-1}+\frac{1}{n+m}\sum_{i=(t-1)m+1}^{tm}\sgn(\bar{\btheta}_{t-1}^\top \bX'_i)\bX'_i. \label{Eq:theta_t with labelled}
%\end{align}
Following similar steps as in Appendix \ref{pf of Thm:gen bound GMM}, we first derive the upper bound for $|\mathrm{gen}_1|$.

At $t=1$, from \eqref{Eq:theta_0} and \eqref{Eq:mu1 decomp}, the expectation $\bmu_1'^{\xi_0,\bmu^{\bot}}:=\bbE[\btheta_1'|\xi_0,\bmu^{\bot}]$ is rewritten as
\begin{align}
	\bmu_1'^{\xi_0,\bmu^{\bot}}&=\frac{n}{n+m}\btheta_0+\frac{1}{n+m}\sum_{i=1}^m\bbE[\sgn(\bar{\btheta}_0^\top \bX'_j)\bX'_j|\xi_0,\bmu^{\bot}] \nn\\
%	&=\bbE_{Y_j'}[~\bbE[\sgn(\bar{\btheta}_0^\top \bX'_j)\bX'_j ~|~ \xi_0,\bmu^{\bot},Y_j']~]\\
	&=\frac{n}{n+m}\bigg(\bigg(1+\frac{\sigma}{\sqrt{n}}\xi_0\bigg)\bmu+\frac{\sigma}{\sqrt{n}}\bmu^{\bot} \bigg) \nn\\*
	&\quad +\frac{m}{n+m}\bigg(\bigg(1-2\rmQ \bigg(\frac{\alpha}{\sigma}\bigg)+\frac{2\sigma \alpha}{\sqrt{2\pi}}\exp\bigg(-\frac{\alpha^2}{2\sigma^2}\bigg) \bigg)\bmu+\frac{2\sigma \beta}{\sqrt{2\pi}}\exp\bigg(-\frac{\alpha^2}{2\sigma^2}\bigg)\bup \bigg)\\
	&=\bigg(1+\frac{\sqrt{n}\sigma\xi_0}{n+m}+\frac{m}{n+m}\bigg(-2\rmQ \bigg(\frac{\alpha}{\sigma}\bigg)+\frac{2\sigma \alpha}{\sqrt{2\pi}}\exp\bigg(-\frac{\alpha^2}{2\sigma^2}\bigg) \bigg) \bigg)\bmu \nn\\
	&\quad +\bigg(\frac{\sqrt{n}\sigma \|\bmu^{\bot}\|_2}{n+m}+\frac{m}{n+m}\frac{2\sigma \beta}{\sqrt{2\pi}}\exp\bigg(-\frac{\alpha^2}{2\sigma^2}\bigg) \bigg)\bup.
\end{align}
Then the correlation between $\bmu_1'^{\xi_0,\bmu^{\bot}}$ and $\bmu$ is given by
\begin{align}
	\rho(\bmu_1'^{\xi_0,\bmu^{\bot}},\bmu)=\tilF_{\sigma,\xi_0,\bmu^{\bot}}(\alpha).
\end{align}

The gen-error $\mathrm{gen}_1$ is given by
\begin{align}
	\mathrm{gen}_1&=\frac{1}{n+m}\sum_{i=1}^n\bbE\Big[l(\btheta_1',(\bX,Y))-l(\btheta_1,(\bX_i,Y_i))\Big] \nn\\
	&\quad +\frac{1}{n+m}\sum_{i=1}^m \bbE_{\xi_0,\bmu^{\bot}}\bigg[\bbE\left[l(\btheta_1',(\bX,Y))-l(\btheta_1',(\bX'_i,\hatY_i'))|\xi_0,\bmu^{\bot} \right]\bigg]\\
	&=\frac{1}{n+m}\sum_{i=1}^n\bbE_{\btheta_1'} \Big[\HD(P_{\bZ}\|P_{\bZ_i|\btheta_1'}|p_{\btheta_1'}) \Big] \nn\\
	&\quad +\frac{1}{n+m}\sum_{i=1}^{m}\sum_{i=1}^n\bbE_{\xi_0,\bmu^{\bot}}\bbE_{\btheta_1'|\xi_0,\bmu^{\bot}} \bigg[\HD(P_{\bZ}\| P_{\bX_i',\hatY_i'|\xi_0,\bmu^{\bot}}|p_{\btheta_1'}) \nn\\
	&\quad +\HD(P_{\bX_i',\hatY_i'|\xi_0,\bmu^{\bot}}\| P_{\bX_i',\hatY_i'|\xi_0,\bmu^{\bot},\btheta_1'}| p_{\btheta_1}) \bigg].
\end{align}

\begin{itemize}
	\item 	\textbf{Calculate $\bbE_{\btheta_1'} \Big[\HD(P_{\bZ}\|P_{\bZ_i|\btheta_1'}|p_{\btheta_1'}) \Big] $:}
	\begin{align}
		&\bbE_{\btheta_1'} \Big[\HD(P_{\bZ}\|P_{\bZ_i|\btheta_1'}|p_{\btheta_1'}) \Big] \nn\\
		&=\int Q_{\btheta_1'}(\btheta)(P_{\bZ}(\bz)-P_{\bZ_i|\btheta_1'}(\bz|\btheta))\log\frac{1}{p_{\btheta}(\bz)} \rmd \bz \rmd \btheta\\
%		&=\frac{1}{2\sigma^2} \int Q_{\btheta_0}(\btheta)(P_{\bZ}(\bx,y)-P_{\bZ_i|\btheta_0}(\bx,y|\btheta)) \big(\bx^\top\bx-2y\btheta^\top\bx+\btheta^\top\btheta \big) \rmd \bx \rmd y \rmd \btheta\\
%		&=-\frac{ 1}{2\sigma^2}\int P_{\bZ}(\bx,y)(Q_{\btheta_0}(\btheta)-P_{\btheta_0|\bZ_i}(\btheta|\bx,y))2y\btheta^\top\bx ~\rmd \bx \rmd y\, \rmd \btheta\\
		&=-\frac{1}{2\sigma^2}\int \Big( P_{\bZ}(\bx,1)(Q_{\btheta_1'}(\btheta)-P_{\btheta_1'|\bZ_i}(\btheta|\bx,1)) \nn\\
		&\qquad - P_{\bZ}(\bx,-1)(Q_{\btheta_1'}(\btheta)-P_{\btheta_1'|\bZ_i}(\btheta|\bx,-1)) \Big)2\btheta^\top\bx ~\rmd \bx \rmd \btheta\\
		&=-\frac{1}{\sigma^2}\int \Big(\frac{\bmu-\bx}{n+m} P_{\bZ}(\bx,1)- \frac{\bmu+\bx}{n+m}P_{\bZ}(\bx,-1) \Big)^\top\bx ~\rmd \bx\\
		&=-\frac{1}{\sigma^2}\bigg(\frac{\bmu^\top\bmu-\bmu^\top\bmu-d\sigma^2}{2(n+m)}-\frac{-\bmu^\top\bmu+\bmu^\top\bmu+d\sigma^2}{2(n+m)} \bigg)\\
		&=\frac{d}{n+m}.
	\end{align}
	
	\item \textbf{Calculate $\bbE_{\btheta_1'|\xi_0,\bmu^{\bot}}\big[  \HD(P_{\bX_i',\hatY_i'|\xi_0,\bmu^{\bot}}\| P_{\bX_i',\hatY_i'|\xi_0,\bmu^{\bot},\btheta_1'}| p_{\btheta_1})\big]$:}
		\begin{align}
			&\bbE_{\btheta_1'|\xi_0,\bmu^{\bot}}\big[  \HD(P_{\bX_i',\hatY_i'|\xi_0,\bmu^{\bot}}\|P_{\bX_i',\hatY_i'|\xi_0,\bmu^{\bot},\btheta_1'}| p_{\btheta_1})\big]\nn\\
%			&=\bbE_{\btheta_1|\xi_0,\bmu^{\bot}}\Big[h(P_{\bX_i',\hatY_i'|\xi_0,\bmu^{\bot}},p_{\btheta_1})-h(P_{\bX_i',\hatY_i'|\xi_0,\bmu^{\bot},\btheta_1},p_{\btheta_1}) \Big]\\
%			&=\frac{1}{2\sigma^2}\int Q_{\btheta_1|\xi_0,\bmu^{\bot}}(\btheta|\xi_0,\bmu^{\bot}) \big(P_{\bX_i',\hatY_i'|\xi_0,\bmu^{\bot}}(\bx,y|\xi_0,\bmu^{\bot})-P_{\bX_i',\hatY_i'|\xi_0,\bmu^{\bot},\btheta_1}(\bx,y|\xi_0,\bmu^{\bot},\btheta) \big) \nn\\
%			&\qquad  \qquad \times(\bx^\top\bx-2y\btheta^\top\bx+\btheta^\top \btheta) \rmd \bx \rmd y \rmd \btheta \\
			&=-\frac{1}{\sigma^2}\int P_{\bX_i',\hatY_i'|\xi_0,\bmu^{\bot}}(\bx,y|\xi_0,\bmu^{\bot}) \big(Q_{\btheta_1'|\xi_0,\bmu^{\bot}}(\btheta|\xi_0,\bmu^{\bot}) \nn\\
			&\qquad \qquad  -P_{\btheta_1'|\bX_i',\hatY_i',\xi_0,\bmu^{\bot}}(\btheta| \bx,y,\xi_0,\bmu^{\bot}) \big)(y\btheta^\top\bx) \rmd \bx \rmd y \rmd \btheta\\
			&=-\frac{1}{\sigma^2}\int P_{\bX_i',\hatY_i'|\xi_0,\bmu^{\bot}}(\bx,y|\xi_0,\bmu^{\bot}) \bigg(\frac{\bmu_1^{\xi_0,\bmu^{\bot}}-y\bx}{n+m} \bigg)^\top  (y\bx) \rmd \bx \rmd y \\
			&=\frac{1}{(n+m)\sigma^2}\Big(\bbE[\bX_i'^\top\bX_i'|\xi_0,\bmu^{\bot}]-(\bmu_1^{\xi_0,\bmu^{\bot}})^\top\bmu_1^{\xi_0,\bmu^{\bot}}\Big)\\
			&=\frac{d\sigma^2+\bmu^\top\bmu-(\bmu_1^{\xi_0,\bmu^{\bot}})^\top\bmu_1^{\xi_0,\bmu^{\bot}}}{(n+m)\sigma^2}\\
			&=\frac{d\sigma^2+1-J_{\sigma}^2(\alpha)-K_{\sigma}^2(\alpha)}{(n+m)\sigma^2}.
		\end{align}
	
	\item \textbf{Calculate $\bbE_{\btheta_1'|\xi_0,\bmu^{\bot}}[\HD(P_{\bZ}\| P_{\bX_i',\hatY_i'|\xi_0,\bmu^{\bot}}|p_{\btheta_1'})]$:}
	Since given any fixed $\btheta_1$, $p_{\btheta_1'}(\bx|\cdot)$ is a Gaussian distribution, for any $y\in\{\pm 1\}$, we have
	\begin{align}
		&\frac{1}{2}\int P_{\btheta_1'|\xi_0,\bmu^{\bot}}(\btheta) \big(P_{\bX|Y}(\bx|\by)-P_{\bX_i'|\hatY_i'}(\bx|y) \big)\log \frac{1}{ p_{\btheta}(\bx|y)} \rmd \bx  \rmd \btheta \nn\\
		&=\frac{1}{4\sigma^2}\int P_{\btheta_1'|\xi_0,\bmu^{\bot}}(\btheta) \big(P_{\bX|Y}(\bx|y)-P_{\bX_i'|\hatY_i'}(\bx|y) \big)\big(\bx^\top\bx-2y\btheta^\top \bx+\btheta^\top\btheta \big) \rmd \bx \rmd \btheta\\
		&=-\frac{1}{2\sigma^2}\int P_{\btheta_1'|\xi_0,\bmu^{\bot}}(\btheta) \bigg(\frac{1}{2}P_{\bX|Y}(\bx|y)-\frac{1}{2}P_{\bX_i'|\hatY_i'}(\bx|y) \bigg)\big(y\btheta^\top \bx \big) \rmd \bx \rmd \btheta\\
		&=-\frac{ 1}{2\sigma^2}(\bmu_1'^{\xi_0,\bmu^{\bot}})^\top\big(\bmu-\bmu_1^{\xi_0,\bmu^{\bot}} \big)\\
		&=\frac{m(J_{\sigma}^2(\alpha)+K_{\sigma}^2(\alpha))-(m-n)J_{\sigma}(\alpha)-n}{2\sigma^2(n+m)},
	\end{align}
	and then 
	\begin{align}
		\bbE_{\btheta_1'|\xi_0,\bmu^{\bot}}[\HD(P_{\bZ}\| P_{\bX_i',\hatY_i'|\xi_0,\bmu^{\bot}}|p_{\btheta_1'})]=\frac{m(J_{\sigma}^2(\alpha)+K_{\sigma}^2(\alpha))-(m-n)J_{\sigma}(\alpha)-n}{\sigma^2(n+m)}.
	\end{align}
\end{itemize}

%Let $\bV_i=\sgn(\bar{\btheta}_0^\top \bX'_j)\bX'_j-\bbE[\sgn(\bar{\btheta}_0^\top \bX'_j)\bX'_j|\xi_0,\bmu^{\bot}]$ and then $\btheta_1-\bmu_1^{\xi_0,\bmu^{\bot}}=\frac{1}{n+m}\sum_{i=1}^{m}\bV_i$.  For any $k\in[1:d]$, let $V_{i,k}$, $\theta_{1,k}$, $\mu_{1,k}$ denote the $k$-th components of $\bV_i$, $\btheta_1$ and  $\bmu_{1}^{\xi_0,\bmu^{\bot}}$, respectively. From \eqref{Eq:cramer transform V_i,k}, by applying Chernoff-Cram\'{e}r inequality, we have for all $\veps> 0$
%\begin{align}
%	&\bbP_{\xi_0,\bmu^{\bot}}\Big(\|\btheta_1-\bmu_1^{\xi_0,\bmu^{\bot}}\|_{\infty} > \veps \Big) \nn\\
%%	&=\bbP_{\xi_0,\bmu^{\bot}}\Big(\max_{k\in[1:d]}|\theta_{1,k}-\mu_{1,k}|>\veps \Big)\\
%%	&\leq \sum_{k=1}^d\bbP_{\xi_0,\bmu^{\bot}}\Big(|\theta_{1,k}-\mu_{1,k}|>\veps \Big)\\
%	&\leq \sum_{k=1}^d\bbP_{\xi_0,\bmu^{\bot}}\bigg(\bigg|\frac{1}{n+m}\sum_{i=1}^m V_{i,k} \bigg|>\veps \bigg)\\
%	&\leq \sum_{k=1}^d 2\exp\Big(-(n+m)\sup_{s>0}\Big(s\veps-\frac{m}{n+m}\log\bbE_{V_{i,k}}[e^{s V_{i,k}}] \Big) \Big)\\
%	&\leq \sum_{k=1}^d 2\exp\Big(-(n+m)\sup_{s>0}\Big(s\veps-\log\bbE_{V_{i,k}}[e^{s V_{i,k}}] \Big) \Big)\\
%	&=2d\exp\bigg(-\frac{(n+m)(\veps-\tilc_2)^2}{2\sigma^2} \bigg)\\
%	&=:\delta_{n+m,\veps,d}, \label{Eq:new theta1 concentrate ineq}
%\end{align}
%where $\delta_{n+m,\veps,d}\xrightarrow{\text{a.s.}}0$ as $m\to\infty$ and does not depend on $\xi_0,\bmu^{\bot}$. 

Therefore, the gen-error at $t=1$ can be exactly characterized as follows
\begin{align}
	\mathrm{gen}_1&=\bbE_{\xi_0,\bmu^{\bot}}\bigg[\frac{nd}{(n+m)^2\sigma^2} \nn\\
	&\qquad  +\frac{m}{n+m}\frac{d\sigma^2+1-J_{\sigma}^2(\alpha)-K_{\sigma}^2(\alpha)+m(J_{\sigma}^2(\alpha)+K_{\sigma}^2(\alpha))-(m-n)J_{\sigma}(\alpha)-n}{(n+m)\sigma^2}\bigg]\\
	&=\bbE_{\xi_0,\bmu^{\bot}}\bigg[\frac{m(m-1)(J_{\sigma}^2(\alpha)+K_{\sigma}^2(\alpha))-m(m-n)J_{\sigma}(\alpha)-nm+m(1+d\sigma^2)+nd}{(n+m)^2\sigma^2} \bigg].
\end{align}

For $t\geq 2$, similar to the derivation in Appendix \ref{pf of Thm:exact gen GMM}, by iteratively implementing the calculation, we only need to replace $\alpha$ with the correlation evolution function $\tilF_{\sigma,\xi_0,\bmu^{\bot}}(\cdot)$ (cf.~\eqref{Eq:new F_sig}) and then the gen-error for any $t\geq 1$ is characterized as follows.
For almost all sample paths, there exists a vanishing sequence $\epsilon_m$ ($\epsilon_m\to 0$ as $m\to\infty$) , such that
\begin{align}
	\mathrm{gen}_t
	=\bbE_{\xi_0,\bmu^{\bot}}\bigg[&\frac{m(m-1)(J_{\sigma}^2(\tilF_{\sigma,\xi_0,\bmu^{\bot}}(\alpha))+K_{\sigma}^2(\tilF_{\sigma,\xi_0,\bmu^{\bot}}(\alpha)))}{(n+m)^2\sigma^2} \nn\\
	& +\frac{-m(m-n)J_{\sigma}(\tilF_{\sigma,\xi_0,\bmu^{\bot}}(\alpha))-nm}{(n+m)^2\sigma^2}\bigg]+\epsilon_m',
\end{align}
where $\epsilon_m'=\epsilon_m+\frac{m(1+d\sigma^2)+nd}{(n+m)^2\sigma^2}\to 0$ as $m\to\infty$ and $\alpha$ stands for $\alpha(\xi_0,\bmu^{\bot})$.

The proof of Corollary~\ref{Coro:reuse labeldata gen bound} is thus completed.

\section{Proof of Theorem \ref{Thm:gen bound GMM reg}}\label{pf of Thm:gen bound GMM reg}
Theorem \ref{Thm:gen bound GMM reg} can be proved similarly from the proof of Theorem \ref{Thm:gen bound GMM}. For simplicity, in the following proofs, we abbrviate $\mathrm{gen}_t(P_\bZ, P_\bX, \{P_{\btheta_k|S_{\rml},S_{\rmu}}\}_{k=0}^t, \{f_{\btheta_k}\}_{k=0}^{t-1} )$ as $\mathrm{gen}_t$. With $\ell_2$-regularization, the algorithm operates in the following steps. Let $\lambda\in\bbR_+$ be the regularization parameter.

\begin{itemize}[leftmargin= 15 pt]
	\item  \textbf{Step 1: Initial round $t=0$ with $S_{\rml}$:} By minimizing the regularized empirical risk of labelled dataset $S_{\rml}$
	\begin{align}
		L_{S_{\rml}}^{\mathrm{reg}}(\btheta)&=\frac{1}{n}\sum_{i=1}^n l(\btheta,(\bX_i,Y_i))+\frac{\lambda}{2}\|\btheta\|_2^2 \overset{\rmc}{=} \frac{1}{2\sigma^2 n}\sum_{i=1}^n(\bX_i-Y_i\btheta)^\top(\bX_i-Y_i\btheta)+\frac{\lambda}{2}\|\btheta\|_2^2,
	\end{align}	
	where $\stackrel{\rmc}{=}$ means that both sides differ by a constant independent of $\btheta$,  we obtain the minimizer
	\begin{align}
		\btheta_0^{\mathrm{reg}}&= \argmin_{\btheta\in\Theta}L_{S_{\rml}}(\btheta)=\frac{1}{n}\sum_{i=1}^n \frac{Y_i \bX_i}{1+\sigma^2\lambda} \nn\\
		&= \frac{\btheta_0}{1+\sigma^2\lambda} \sim \calN\bigg(\frac{\bmu}{1+\sigma^2\lambda}, \frac{\sigma^2}{n(1+\sigma^2\lambda)^2}\bI_d \bigg). \label{Eq: def theta_0 reg}
	\end{align}
	
	\item \textbf{Step 2: Pseudo-label data in $S_{\rmu}$:} At each iteration $t\in[1:\tau]$, for any $i\in\calI_t$, we use $\btheta_{t-1}^{\mathrm{reg}}$ to assign a pseudo-label for $\bX_i'$, that is, $\hatY_i'=f_{\btheta_{t-1}^{\mathrm{reg}}}(\bX_i')=\sgn(\bX_i'^\top \btheta_{t-1}^{\mathrm{reg}})$.
	
	\item \textbf{Step 3: Refine the model:} We then use the pseudo-labelled dataset $\hatS_{\rmu,t}$ to train the new model. By minimizing the empirical risk of $\hatS_{\rmu,t}$
	\begin{align}
		L_{\hatS_{\rmu,t}}(\btheta)& = \frac{1}{m}\sum_{i\in\calI_t} l(\btheta,(\bX'_i,\hatY_i')) +\frac{\lambda}{2m}\|\btheta\|_2^2 \nn\\
		& \overset{\rmc}{=}  \frac{1}{2\sigma^2 m} \sum_{i\in\calI_t}(\bX'_i-\hatY_i'\btheta)^\top(\bX'_i-\hatY_i'\btheta)+\frac{\lambda}{2}\|\btheta\|_2^2,
	\end{align}
	we obtain the new model parameter
	\begin{align}
		\btheta_t^{\mathrm{reg}}&=\frac{1}{m}\sum_{i\in\calI_t} \frac{\hatY_i'\bX'_i}{1+\sigma^2\lambda}=\frac{1}{m}\sum_{i\in\calI_t}\frac{\sgn(\bX_i'^\top \btheta_{t-1}^{\mathrm{reg}} )\bX'_i}{1+\sigma^2\lambda}. \label{Eq:theta_t reg}
	\end{align}
	If $t<\tau$, go back to Step 2.
	
\end{itemize}

\begin{enumerate}

	\item \textbf{Characterization of $|\mathrm{gen}_1|$}:
	
	From \eqref{Eq: def theta_0 reg}, we still have $\rho(\btheta_0^{\mathrm{reg}}, \bmu)=\alpha(\xi_0,\bmu^{\bot})$
		and
	\begin{align}
				\bar{\btheta}_0^{\mathrm{reg}}:=\frac{\btheta_0^{\mathrm{reg}}}{\|\btheta_0^{\mathrm{reg}}\|_2^2}=\alpha\bmu+\beta\bup=\bar{\btheta}_0. \label{Eq:normalized theta_0^R = theta_0}
	\end{align}
	From \eqref{Eq:normalized theta_0^R = theta_0}, we can rewrite \eqref{Eq:theta_t reg} as follows
	\begin{align}
		\btheta_1^{\mathrm{reg}}&
		=\frac{1}{m}\sum_{i=1}^{m}\frac{\sgn(\bX_i'^\top \bar{\btheta}_1^{\mathrm{reg}} )\bX'_i}{1+\sigma^2\lambda}=\frac{1}{m}\sum_{i=1}^{m}\frac{\sgn(\bX_i'^\top \bar{\btheta}_0 )\bX'_i}{1+\sigma^2\lambda}=\frac{\btheta_1}{1+\sigma^2\lambda}.
	\end{align}
	Thus, the expectation of $\btheta_1^{\mathrm{reg}}$ conditioned on $(\xi_0,\bmu^{\bot})$ is given by
	\begin{align}
		&\bmu_1^{\mathrm{reg}|\xi_0,\bmu^{\bot}}:=\frac{1}{1+\sigma^2\lambda}\bbE[\sgn(\bX_j'^\top \btheta_0^{\mathrm{reg}} )\bX'_j|\xi_0,\bmu^{\bot}] \\
		%		&=\frac{1}{1+\sigma^2\lambda} \bigg(\frac{1}{2}\bbE[\sgn(\bX_j'^\top \bar{\btheta}_0^{\mathrm{reg}} )\bX'_j ~|~\xi_0,\bmu^{\bot}, Y_j'=-1] \nn\\
		%		&\quad +\frac{1}{2}\bbE[\sgn(\bX_j'^\top \bar{\btheta}_0^{\mathrm{reg}} )\bX'_j ~|~\xi_0,\bmu^{\bot}, Y_j'=1] \bigg)\\
		%		&=\frac{1}{1+\sigma^2\lambda} \bigg(\frac{1}{2}\bbE[\sgn(\bX_j'^\top \bar{\btheta}_0 )\bX'_j ~|~\xi_0,\bmu^{\bot}, Y_j'=-1] \nn\\
		%		&\quad +\frac{1}{2}\bbE[\sgn(\bX_j'^\top \bar{\btheta}_0 )\bX'_j ~|~\xi_0,\bmu^{\bot}, Y_j'=1] \bigg) \label{Eq:use normalized theta_0^R = theta_0}\\
		&=\frac{1}{1+\sigma^2\lambda} \bmu_1^{\xi_0,\bmu^{\bot}}\\
		&=\frac{1}{1+\sigma^2\lambda} \bigg(\bigg(1-2\rmQ \bigg(\frac{\alpha}{\sigma}\bigg)+\frac{2\sigma \alpha}{\sqrt{2\pi}}\exp\bigg(-\frac{\alpha^2}{2\sigma^2}\bigg) \bigg)\bmu+\frac{2\sigma \beta}{\sqrt{2\pi}}\exp\bigg(-\frac{\alpha^2}{2\sigma^2}\bigg)\bup \bigg)\\
		&=\frac{J_{\sigma}(\alpha)\bmu+K_{\sigma}(\alpha)\bup}{1+\sigma^2\lambda}.  \label{Eq:mu1 decomp reg}
	\end{align}
	Recall $\mathrm{gen}_1$ given in \eqref{Eq:exact gen1}. In the case with regularization, the gen-error $\mathrm{gen}_1^{\reg}$ has the same definition as $\mathrm{gen}_1$. To derive $\mathrm{gen}_1^{\reg}$, we need to calculate the following two terms.
	\begin{itemize}
		\item 	\textbf{Calculate $\bbE_{\btheta_1^{\reg}|\xi_0,\bmu^{\bot}}\big[  \HD(P_{\bX_i',\hatY_i'|\xi_0,\bmu^{\bot}} \| P_{\bX_i',\hatY_i'|\xi_0,\bmu^{\bot},\btheta_1^{\reg}} | p_{\btheta_1^{\reg}})\big]$:}
		\begin{align}
			&\bbE_{\btheta_1^{\reg}|\xi_0,\bmu^{\bot}}\big[  \HD(P_{\bX_i',\hatY_i'|\xi_0,\bmu^{\bot}} \| P_{\bX_i',\hatY_i'|\xi_0,\bmu^{\bot},\btheta_1^{\reg}}| p_{\btheta_1^{\reg}})\big]\nn\\
			%				&=\bbE_{\btheta_1|\xi_0,\bmu^{\bot}}\Big[h(P_{\bX_i',\hatY_i'|\xi_0,\bmu^{\bot}},p_{\btheta_1})-h(P_{\bX_i',\hatY_i'|\xi_0,\bmu^{\bot},\btheta_1},p_{\btheta_1}) \Big]\\
			&=\frac{1}{2\sigma^2}\int Q_{\btheta_1^{\reg}|\xi_0,\bmu^{\bot}}(\btheta|\xi_0,\bmu^{\bot}) \big(P_{\bX_i',\hatY_i'|\xi_0,\bmu^{\bot}}(\bx,y|\xi_0,\bmu^{\bot}) \nn\\
			&\qquad  \qquad -P_{\bX_i',\hatY_i'|\xi_0,\bmu^{\bot},\btheta_1^{\reg}}(\bx,y|\xi_0,\bmu^{\bot},\btheta) \big)(\bx^\top\bx-2y\btheta^\top\bx+\btheta^\top \btheta) \rmd \bx \rmd y \rmd \btheta \\
			&=-\frac{1}{\sigma^2}\int P_{\bX_i',\hatY_i'|\xi_0,\bmu^{\bot}}(\bx,y|\xi_0,\bmu^{\bot}) \big(Q_{\btheta_1^{\reg}|\xi_0,\bmu^{\bot}}(\btheta|\xi_0,\bmu^{\bot}) \nn\\
			&\qquad \qquad  -P_{\btheta_1^{\reg}|\bX_i',\hatY_i',\xi_0,\bmu^{\bot}}(\btheta| \bx,y,\xi_0,\bmu^{\bot}) \big)(y\btheta^\top\bx) \rmd \bx \rmd y \rmd \btheta\\
			&=-\frac{1}{\sigma^2}\int P_{\bX_i',\hatY_i'|\xi_0,\bmu^{\bot}}(\bx,y|\xi_0,\bmu^{\bot}) \bigg(\frac{\bmu_1^{\xi_0,\bmu^{\bot}}-y\bx}{m(1+\sigma^2\lambda)} \bigg)^\top  (y\bx) \rmd \bx \rmd y \\
			&=\frac{1}{m\sigma^2(1+\sigma^2\lambda)}\Big(\bbE[\bX_i'^\top\bX_i'|\xi_0,\bmu^{\bot}]-(\bmu_1^{\xi_0,\bmu^{\bot}})^\top\bmu_1^{\xi_0,\bmu^{\bot}}\Big)\\
			&=\frac{d\sigma^2+\bmu^\top\bmu-(\bmu_1^{\xi_0,\bmu^{\bot}})^\top\bmu_1^{\xi_0,\bmu^{\bot}}}{m\sigma^2(1+\sigma^2\lambda)}\\
			&=\frac{d\sigma^2+1-J_{\sigma}^2(\alpha)-K_{\sigma}^2(\alpha)}{m\sigma^2(1+\sigma^2\lambda)}.
		\end{align}
		
		\item \textbf{Calculate $\bbE_{\btheta_1^{\reg}|\xi_0,\bmu^{\bot}}\big[ h(P_{\bZ},p_{\btheta_1^{\reg}})- h(P_{\bX_i',\hatY_i'|\xi_0,\bmu^{\bot}},p_{\btheta_1^{\reg}}) \big]$:}
		Given any $(\xi_0,\bmu^{\bot})$, in the following, we drop the condition on $\xi_0,\bmu^{\bot}$ for notational simplicity.
		Since given $\btheta_1$, $p_{\btheta_1}(\bx|\cdot)$ is a Gaussian distribution, for any $y\in\{\pm 1\}$, we have
		\begin{align}
			&\frac{1}{2}\int P_{\btheta_1^{\reg}|\xi_0,\bmu^{\bot}}(\btheta) \big(P_{\bX|Y}(\bx|\by)-P_{\bX_i'|\hatY_i'}(\bx|y) \big)\log \frac{1}{ p_{\btheta}(\bx|y)} \rmd \bx  \rmd \btheta \nn\\
			%				&=\frac{1}{4\sigma^2}\int P_{\btheta_1|\xi_0,\bmu^{\bot}}(\btheta) \big(P_{\bX|Y}(\bx|y)-P_{\bX_i'|\hatY_i'}(\bx|y) \big)\big(\bx^\top\bx-2y\btheta^\top \bx+\btheta^\top\btheta \big) \rmd \bx \rmd \btheta\\
			&=-\frac{1}{2\sigma^2}\int P_{\btheta_1^{\reg}|\xi_0,\bmu^{\bot}}(\btheta) \bigg(\frac{1}{2}P_{\bX|Y}(\bx|y)-\frac{1}{2}P_{\bX_i'|\hatY_i'}(\bx|y) \bigg)\big(y\btheta^\top \bx \big) \rmd \bx \rmd \btheta\\
			&=-\frac{1}{2\sigma^2}(\bmu_1^{\reg|\xi_0,\bmu^{\bot}})^\top\big(\bmu-\bmu_1^{\xi_0,\bmu^{\bot}} \big)\\
			%				&=\frac{\frac{J_{\sigma}^2(\alpha)+K_{\sigma}^2(\alpha)}{(1+\sigma^2\lambda)^2}-\frac{J_{\sigma}(\alpha)}{1+\sigma^2\lambda}}{2\sigma^2}\\
			&=\frac{J_{\sigma}^2(\alpha)+K_{\sigma}^2(\alpha)-J_{\sigma}(\alpha)}{2\sigma^2(1+\sigma^2\lambda)}.
		\end{align}
		Thus, we have
		\begin{align}
			\bbE_{\btheta_1|\xi_0,\bmu^{\bot}}[\HD(P_{\bZ}\| P_{\bX_i',\hatY_i'|\xi_0,\bmu^{\bot}}|p_{\btheta_1})]=\frac{J_{\sigma}^2(\alpha)+K_{\sigma}^2(\alpha)-J_{\sigma}(\alpha)}{\sigma^2(1+\sigma^2\lambda)}.
		\end{align}
	\end{itemize}
	Finally, the gen-error at $t=1$ can be characterized as follows:
	\begin{align}
		&\mathrm{gen}_1^{\reg}=\bbE_{\xi_0,\bmu^{\bot}}\bigg[ \frac{J_{\sigma}^2(\alpha)+K_{\sigma}^2(\alpha)-J_{\sigma}(\alpha)}{\sigma^2(1+\sigma^2\lambda)}+\frac{d\sigma^2+1-J_{\sigma}^2(\alpha)-K_{\sigma}^2(\alpha)}{m\sigma^2(1+\sigma^2\lambda)} \bigg]=\frac{\mathrm{gen}_1}{1+\sigma^2\lambda},
	\end{align}
	where $\alpha$ stands for $\alpha(\xi_0,\bmu^{\bot})$.

	\item \textbf{Iteration $t\in[2:\tau]$:} Since $\btheta_t^{\reg}=\frac{\btheta_t}{1+\sigma^2\lambda}$, by iteratively applying the same techniques in iteration $t=1$ and in Appendix \ref{pf of Thm:exact gen GMM}, the gen-error at any $t\in[2:\tau]$ can be characterized as follows
	\begin{align}
		\mathrm{gen}_t^{\reg}=\frac{\mathrm{gen}_t}{1+\sigma^2\lambda}.
	\end{align}

\end{enumerate}
Theorem \ref{Thm:gen bound GMM reg} is thus proved.

\bibliography{ref.bib}

\end{document}